\documentclass{article}

\usepackage{macros}
\usepackage{math_commands}

\newtoggle{arxiv}
\toggletrue{arxiv}

\usepackage{microtype}
\usepackage{graphicx}
\usepackage{subfigure}
\usepackage{multirow}
\usepackage{pifont}
\usepackage{xspace}
\usepackage[super]{nth}
\usepackage{booktabs} 
\usepackage[framemethod=TikZ]{mdframed}
\usepackage{hyperref}

\usepackage{multirow}
\usepackage{amssymb}
\usepackage{amsmath}
\usepackage{amsthm}
\usepackage[utf8]{inputenc} 
\usepackage[T1]{fontenc}    
\usepackage{hyperref}       
\usepackage{url}            
\usepackage{booktabs}       
\usepackage{amsfonts}       
\usepackage{nicefrac}       
\usepackage{microtype}      

\newcommand{\prompt}{\textsc{JRT-Prompt}\xspace}
\newcommand{\arch}{\textsc{JRT-RNN}\xspace}

\usepackage{caption}
\usepackage{doi}
\definecolor{darkblue}{rgb}{0, 0, 0.5}
\definecolor{darkgreen}{rgb}{0, 0.5, 0}
\definecolor{darkorange}{rgb}{0, 0.5, 0.5}

\hypersetup{colorlinks=true, citecolor=darkblue, linkcolor=darkblue, urlcolor=darkblue}
\usepackage{tabularx}
\usepackage{wrapfig}
\usepackage{enumitem}
\usepackage{graphicx}
\usepackage{float}
\usepackage{color, soul}
\usepackage{tabularx}
\usepackage{xspace}
\usepackage{makecell}
\usepackage{caption}
\usepackage{subcaption}
\usepackage{array}
\usepackage{float}
\usepackage{listings}
\usepackage{bbm}
\usepackage{algorithm}
\usepackage{algorithmicx}
\usepackage[noend]{algpseudocode}

\definecolor{codegreen}{rgb}{0,0.6,0}
\definecolor{codegray}{rgb}{0.5,0.5,0.5}
\definecolor{codepurple}{rgb}{0.58,0,0.82}
\definecolor{backcolour}{rgb}{1,1,1}
\lstdefinestyle{mystyle}{
    backgroundcolor=\color{backcolour},   
    commentstyle=\color{codegreen},
    keywordstyle=\color{magenta},
    numberstyle=\tiny\color{codegray},
    stringstyle=\color{codepurple},
    basicstyle=\ttfamily\footnotesize,
    breakatwhitespace=false,         
    breaklines=true,                 
    captionpos=b,                    
    keepspaces=true,                 
    numbers=left,                    
    numbersep=5pt,                  
    showspaces=false,                
    showstringspaces=false,
    showtabs=false,                  
    tabsize=2,
}

\usepackage{tikz}
\usepackage{listings}
\usepackage[T1]{fontenc}
\usepackage{mdframed} 
\definecolor{dkgreen}{rgb}{0,0.6,0}
\definecolor{gray}{rgb}{0.5,0.5,0.5}
\definecolor{light-gray}{gray}{0.95} 
\definecolor{mauve}{rgb}{0.58,0,0.82}
\definecolor{backcolour}{rgb}{0.95,0.95,0.92}
\newmdenv[linecolor=light-gray,%
backgroundcolor=light-gray,
innerleftmargin=2.8pt,
innerbottommargin=-0.8pt,
leftmargin=0.0pt,
rightmargin=0.0pt,
skipbelow=-2.0pt,
frametitle={}]{codeframe}

\lstdefinestyle{prompts}{
    commentstyle=\color{dkgreen},
    keywordstyle=\color{magenta},
    moredelim=**[is][\color{mauve}]{@}{@},
    basicstyle=\fontsize{7}{9}\selectfont\ttfamily,
    breakatwhitespace=false,         
    breaklines=true,                 
    captionpos=b,                    
    keepspaces=true,                 
    numbers=none,                    
    numbersep=5pt,                  
    showspaces=false,
    showstringspaces=false,
    showtabs=false,                  
    tabsize=2
}
\iftoggle{arxiv}{
    \usepackage[numbers,sort]{natbib}
    \setlength{\textwidth}{6.5in}
    \setlength{\textheight}{9in}
    \setlength{\oddsidemargin}{0in}
    \setlength{\evensidemargin}{0in}
    \setlength{\topmargin}{-0.5in}
    \newlength{\defbaselineskip}
    \setlength{\defbaselineskip}{\baselineskip}
    \setlength{\marginparwidth}{0.8in}
}

\makeatletter
\def\@copyrightspace{\relax}
\makeatother

\makeatletter
\def\@myauthornotes{}
\def\myauthornote#1{%
  \if@ACM@anonymous\else
    \g@addto@macro\addresses{}%
    \g@addto@macro\@myauthornotes{%
      \stepcounter{footnote}\footnotetext{#1}}%
  \fi}
\makeatother

\iftoggle{arxiv}{
    \title{
        Just read twice: closing the recall gap for \\ recurrent language models
    }
    \usepackage{authblk}
    \author[$\dagger$]{Simran Arora}
    \author[$\triangle$]{Aman Timalsina}
    \author[$\dagger$]{Aaryan Singhal}
    \author[$\dagger$]{Benjamin Spector}
    \author[$\dagger$]{Sabri Eyuboglu}
    \author[$\dagger$]{Xinyi Zhao} 
    \author[$\dagger$]{Ashish Rao}
    \author[$\triangle$]{Atri Rudra}
    \author[$\dagger$]{Christopher Ré}
    \vspace{4pt}
    \affil[$\dagger$]{\{\texttt{simarora,aaryan04,bfs,eyuboglu,xyzhao99,aprao,chrismre\}@stanford.edu}}
    \affil[$\triangle$]{\{\texttt{amantima,atri\}@buffalo.edu}}
}

\begin{document}
\maketitle
\vspace{-2mm}
\begin{abstract}
Recurrent large language models that compete with Transformers in language modeling perplexity are emerging at a rapid rate (e.g., Mamba, RWKV). Excitingly, these architectures use a constant amount of memory during inference. However, due to the limited memory, recurrent LMs cannot recall and use all the information in long contexts leading to brittle in-context learning (ICL) quality.  
A key challenge for efficient LMs is selecting what information to store versus discard. 
In this work, we observe the \textit{order} in which information is shown to the LM impacts the selection difficulty.
To formalize this, we show that the hardness of information recall reduces to the hardness of a problem called set disjointness (SD), a quintessential problem in communication complexity that requires a streaming algorithm (e.g., recurrent model) to decide whether inputted sets are disjoint. We  empirically and theoretically show that the recurrent memory required to solve SD changes with set order, i.e., whether the smaller set appears first in-context.
Our analysis suggests, to mitigate the reliance on data order, we can put information in the \textit{right order} in-context or process prompts \textit{non-causally}. 
Towards that end, we first propose: (1) {\prompt}, where context gets repeated multiple times in the prompt, effectively showing the model \textit{all} data orders. This gives $11.0 \pm 1.3$ points of improvement, averaged across $16$ recurrent LMs and the $6$ ICL tasks, with $11.9\times$ higher throughput than FlashAttention-2 for generation prefill (length $32\mathrm{k}$, batch size $16$, NVidia H100). We then propose (2) {\arch}, which uses non-causal \textit{prefix-linear-attention} to process prompts and provides $99\%$ of Transformer quality at $360\mathrm{M}$ params., $30\mathrm{B}$ tokens and $96\%$ at $1.3\mathrm{B}$ params., $50\mathrm{B}$ tokens on average across the tasks, with $19.2\times$ higher throughput for prefill than FA2.
\end{abstract}

\vspace{-4mm}
\section{Introduction}

\vspace{-2mm}
Recent work has made rapid progress in developing fixed-memory recurrent architectures (e.g., Mamba \citep{gu2023mamba} and RWKV \citep{peng2023rwkv}) that are competitive with attention in language modeling perplexity. 
During inference, these models are more memory efficient and asymptotically faster than the de-facto Transformer attention \cite{bahdanau2016neural, vaswani2018attention}.
However, there is no free lunch --- due to their limited memory capacity, recurrent LMs cannot recall all the information provided in long-contexts, making in-context learning (ICL) quality brittle
\citep{cho2014properties,
schlag2021linear, arora2024simple}. 
Despite matching in perplexity, we find a $2.8\mathrm{Bn}$ parameter Mamba LM trained on $300\mathrm{Bn}$ tokens of the Pile underperforms a $1.3\mathrm{Bn}$ param. ($2.2\times$ smaller) Transformer LM trained on $50\mathrm{Bn}$ tokens  
($6\times$ fewer tokens) by $5$ points, averaged across a suite of recall-intensive ICL tasks (Table \ref{table:main-quality}).

\begin{figure*}[t]
    \centering
    \includegraphics[width=\linewidth]{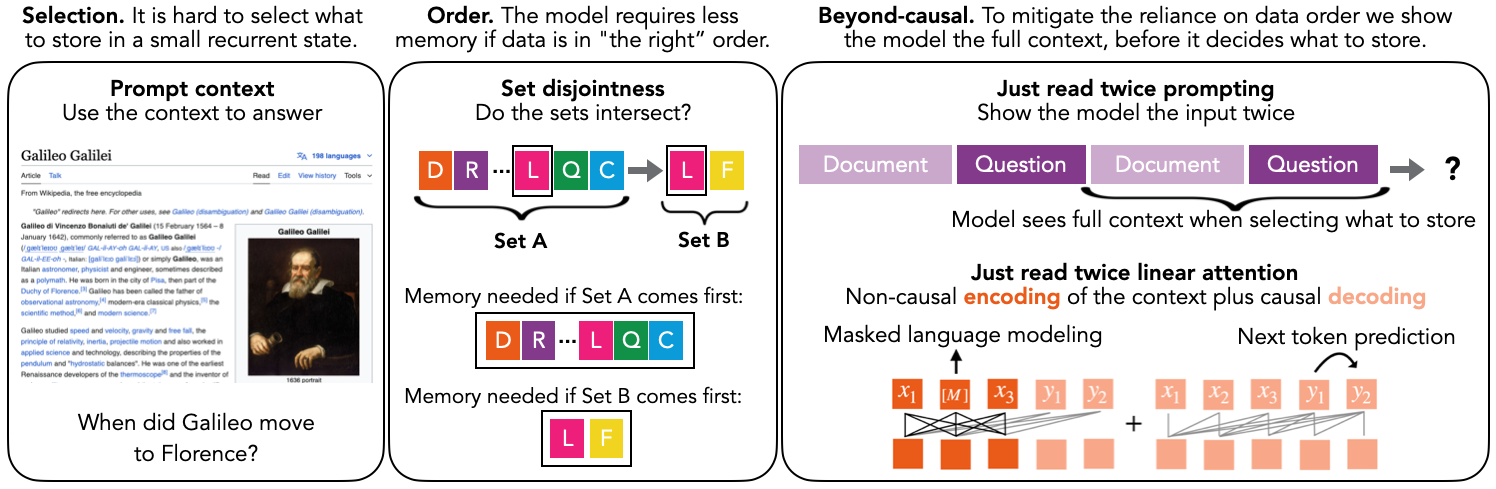}
    \caption{\textbf{Selecting (Left)} Recurrent models have limited memory and deciding what to store from long-contexts (e.g., Galileo's Wikipedia) 
    is challenging. \textbf{Data order (Middle)} changes the selection difficulty: seeing the question  \textit{before} the document simplifies the model's selection task. We formalize this by invoking \textit{set disjointness}, the canonical communication complexity problem of deciding whether two sets $A$ and $B$ are disjoint. A causal model needs enough memory to store set $A$ to be able to compare to set $B$'s elements so, ideally, the smaller set appears first. \textbf{Beyond causal (Right)} We show recurrent models the input twice in-context ({\prompt}) or use encoder-decoder recurrent models to  process the prompt ({\arch}), to mitigate the reliance on data order. }
    \vspace{-3mm}
    \label{fig:main_fig_jrt}
\end{figure*}

Prior work \citep{arora2024simple} formalizes the tradeoff between an architecture's recall ability and memory consumption during inference
by considering a simplified ICL setting shown below. Here, we have
the ``context'' of key-value token pair mappings on the left and ``questions''s on the right for which the model should output $\text{4, 6, 1, 2, 3}$:
$$
\text{A 4 B 3 C 6}\underbrace{\text{F 1}}_{\mathclap{\textbf{Key-Value}}}  \text{E 2} \rightarrow \text{A ? C ?} \underbrace{\text{F ?}}_{\mathclap{\textbf{Query}}} \text{E ? B ?}
$$  
Unfortunately, recurrent models need $\Omega(N)$ space to solve the recall task \citep{arora2024simple}.
\textit{This begs the question of whether we can rely on recurrent models that use constant $O(1)$ space 
for in-context learning.}

Luckily, models often do not need to remember \textit{all} information provided in-context to excel at a task.
The key challenge is \textit{predicting} which subset of information (e.g., facts from documents, variable names from code) is useful to store in memory to support next token predictions. A long line of work focuses on improving the \textit{selection mechanisms} or architectural inductive biases that recurrent language models use to select relevant information (e.g., LSTM \citep{hochreiter1997long}, decay rates \citep{gu2023mamba, yang2023gated}, delta rules \citep{schlag2021linear, munkhdalai2019metalearned}). Other works increase the recurrent state size in hardware efficient ways, traversing a quality-efficiency tradeoff space \cite{arora2024simple}.

Complementing these approaches, 
we focus on the simple observation that the \textit{order} in which data streams into the recurrent LM during inference drastically impacts the difficulty of predicting what to store in the limited memory. 
Suppose we ask questions $\mathcal{Q}$ (e.g., ``When did Galileo move to Florence?''), over documents $\mathcal{D}$ (e.g., the detailed Wikipedia for Galileo Galilei). The model needs to remember just one fact from $\mathcal{D}$ if the prompt is ordered $[\mathcal{Q},  \mathcal{D}]$, but needs to remember \textit{all} facts when it is $[\mathcal{D}, \mathcal{Q}]$ (\Cref{fig:main_fig_jrt} (Left)).

Our work first theoretically formalizes how data order impacts the 
memory requirement (\Cref{sec:understanding}), then proposes two ways to mitigate the reliance on data order: the Just-read-twice (JRT) prompting strategy (\Cref{sec:sec4-jrt-prompt}) and the JRT recurrent architecture (\Cref{sec:sec4-jrt-arch}). 

\textbf{Understanding the role of data order.} Our first insight is that the hardness of the recall problem reduces to the hardness of \textit{set disjointness} (SD), the quintessential, decades-old problem in communication complexity theory \citep{chattopadhyay2010story} (\Cref{thm: mqar-phot}). SD requires a streaming algorithm (e.g., a recurrent model) to decide whether inputted sets provided in-context are disjoint:
$$
\underbrace{\text{7 11 1 17 16 4 6 9}}_{\mathclap{\textbf{Set A}}} \text{ * } \underbrace{\text{8 1 5 6}}_{\mathclap{\textbf{Set B}}} \rightarrow  \text{ False, $\{\text{1 6}\}$ } 
$$
With theory and experiments, we show that the size of the first set, $|A|$, governs the memory needed to solve SD.  Causal models need to store all elements in $A$ to be able to compare to the elements of $B$. This suggests that using \textbf{``the right data order''} in-context, e.g. placing the set with $\min(|A|, |B|)$ first, would help memory-limited models. Further, models that see the context \textbf{non-causally} can solve SD in space $\min(|A|, |B|)$, regardless of data order  (\Cref{thm:rec-gen-sd}, \Cref{fig:synthetic}).  We next make use of these insights. 

\paragraph{Using ``the right'' order.} We propose {\prompt} (\Cref{sec:sec4-jrt-prompt}), an extremely simple strategy where information is repeated multiple times in context before the model generates answers (\Cref{fig:main_fig_jrt} (Right)). In the second$+$ pass, the LM conditions on the full context when deciding what to store, effectively avoiding the issue of getting the data order ``right''.
{\prompt} gives $11.0 \pm 1.3$ point improvement averaged across $16$ off-the-shelf recurrent LMs and the $6$ ICL tasks, while providing $11.9\times$ higher throughput than FlashAttention-2 (length $32\mathrm{k}$, batch size $16$) \citep{dao2023flashattention2} (Table \ref{table:main-quality}). {\prompt} increases the context length, but remains asymptotically more compute and memory efficient than attention.

\paragraph{Beyond causal models.} We next propose {\arch}, inspired by the simple design of Prefix-LM encoder-decoder architectures \citep{2020t5, dong2019unified}. Most in-context learning inputs contain two parts, the inputted prompts (context, instructions) and the text generated by the model as output. In Prefix-LMs, the LM processes the prompt region non-causally and causally decodes the output, using only a standard next token prediction loss in the causal region and in loss on the non-causal region. Unfortunately, prior approaches to training Prefix-LM models have seen limited success and use inefficient Transformer backbones \citep{wang2022language}. We apply simple changes to improve quality and efficiency including modifying the training loss 
and using a \textit{linear attention} formulation we term Prefix Linear Attention (PLA). 
We find {\arch} provides a $13.7$ and $6.9$ point average quality improvement at $360$m and $1.3$b parameters, and $19.2\times$ higher throughput than FA2, using our IO-aware implementation (\Cref{table:jrt-rnn-quality}).

Our contributions are: (1) a synthetic and theoretical study of data order and the memory requirement for recurrent models, (2) {\prompt}, and (3) {\arch}.
Researchers have developed many techniques for in-context leanring with Transformers \citep{wei2022chain, creswell2022selection}, and
we need a similar exploration into how to use alternative LLM architectures effectively. Code: \url{https://github.com/HazyResearch/prefix-linear-attention}.

\section{Background}
We focus on developing methods for in-context learning with recurrent LLMs.  We provide key background here and an extended related works discussion in \Cref{app:related-work}.

\paragraph{Recall and in-context learning.} 
Many prior works have identified a skill called \textit{associative recall} as highly correlated with in-context learning quality across architecture classes via extensive theoretical and empirical analysis \citep{graves2014neural, ba2016using, schlag2021linear, elhage2021mathematical, olsson2022context, dao2022hungry, arora2023zoology, gu2023mamba, akyurek2024incontext}. Recall entails using information provided in context (beyond the model's memorized knowledge) to generate next token predictions. For instance, models are used via in-context learning to produce the next steps in a proof given a provided list of Lemmas \citep{lewkowycz2022solving, trinh2024solving}, generate the next chunk of code given a repository \citep{roziere2023code, yang2024swe}, and answer questions or provide summaries given documents \citep{arora2023evaporate}. In a simplified view of the recall task,  a model needs to remember \textit{keys} and \textit{values} seen in context to provide the answers for different \textit{queries}. In this example, the model should output $\text{4, 6, 1, 2, 3}$:
$$
\text{A 4 B 3 C 6}\underbrace{\text{F 1}}_{\mathclap{\textbf{Key-Value}}}  \text{E 2} \rightarrow \text{A ? C ?} \underbrace{\text{F ?}}_{\mathclap{\textbf{Query}}} \text{E ? B ?}
$$  

\vspace{-3mm}
\paragraph{Memory-recall tradeoff for  causal language models.}  Today's LLMs process input text \textit{causally} in a fixed left-to-right order \cite{brown2020language}. Prior work theoretically and empirically demonstrates a fundamental tradeoff between a causal LM's memory consumption during inference and its ability to remember information provided in context (recall) \citep{cho2014properties, schlag2021linear, arora2024simple}.
Attention \citep{vaswani2018attention}, the de-facto LM architecture \citep{brown2020language, chowdhery2022palm, touvron2023llama}, provably solves recall perfectly in $\mathcal{O}(1)$ model depth and width as a function of sequence length. 
However, attention incurs
$\mathcal{O}(N^2)$ complexity during training and $O(N)$ complexity and memory consumption during inference, for sequence length $N$. 
Thus, many works explore alternative recurrent architectures that are more efficient --- sub-quadratic compute and memory in sequence length during training and $\mathcal{O}(1)$ during each token generation step during inference --- while competing with attention in quality \citep[inter alia.]{ma2022mega, dao2022hungry, gu2023mamba, arora2024simple, yang2023gated}. 

However, using a limited amount memory during inference, efficient models provably cannot retain all information seen in-context, sacrificing recall and in-context learning quality \citep{arora2024simple}. 
Models that can better select what information to store can extend the Pareto frontier of the tradeoff space. 
A long line of work explores how to improve this \textit{selection mechanism} via architectural inductive biases \citep[inter alia.]{hochreiter1997long, schlag2021linear, qin2023hgrnn, gu2023mamba}. Another approach is to navigate the quality-efficiency tradeoff space by varying the recurrent \textit{state size} in hardware-efficient ways \citep{massaroli2023laughing, arora2024simple, dao2024transformers}.
Complementing these approaches, the insight motivating our work is that the \textit{order} in which information appears in-context drastically influences the difficulty of the selection step \cite{sutskever2014sequence}. \textit{Non-causal} models, which can see all the input text at once, can help avoid this issue.

\vspace{-2mm}
\section{Understanding the role of data order on recurrent models}
\label{sec:understanding}

\begin{figure*}[t]
    \centering
    \includegraphics[width=0.95\linewidth]{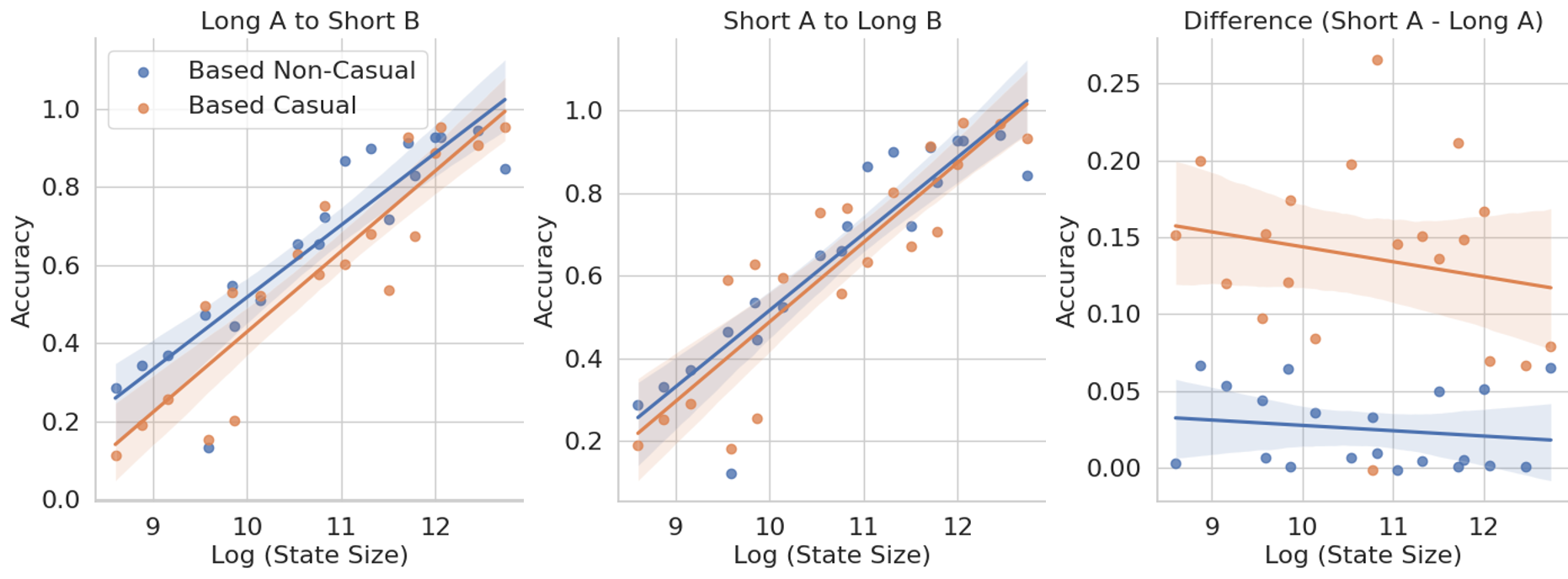}
    \caption{\textbf{Data order vs. quality.} The $x$-axis shows the recurrent state size in (bytes) during inference. The $y$-axis shows the accuracy on the set disjointness task, where the model needs to output the intersecting elements between two sets of tokens $A$ and $B$ (of lengths $|A|$ and $|B|$) provided in-context. \textbf{(Left)} $|A|$ is longer than $|B|$; \textbf{(Middle)} $|B|$ is longer than $|A|$; \textbf{(Right)} Difference in accuracy between the two orderings. 
    We evaluate \textcolor{blue}{non-causal} and \textcolor{orange}{causal} versions of the Based recurrent architecture from \cite{arora2024simple}. 
    For each, we vary the hyperparameters (e.g., model dimension, feature dimension) that affect the state size. We plot the maximum score for each point across a sweep of three learning rates $\{\mathrm{1e-4, 5e-4, 8e-4}\}$ and two random seeds. The plot shows that the causal recurrent models are more sensitive to the data order than non-causal models. }
    \vspace{-3mm}
    \label{fig:synthetic}
\end{figure*}

In this section, we show that the quality of recurrent large language models varies as a function of the order in which data arises in context making them brittle for in-context learning applications. 

\subsection{Analysis of data order and communication complexity}
\label{sec:set_disjointness}

\paragraph{Set disjointness problem.} To formalize the impact of data order, we invoke the set disjointness (SD) problem: given two sets of elements, determine if the intersection is empty or not. SD is the quintessential problem for studying the communication complexity of different streaming algorithms (such as recurrent models) over the past several decades \citep{hemaspaandra2010sigact}. The hardness for a wide collection of problems reduces to the hardness of SD~\citep{chattopadhyay2010story}. A formal definition of this task is provided in Appendix \ref{sec:rec-models-set-disjoint}.

\vspace{-2mm}
\paragraph{Synthetic formulation.} We construct a synthetic task where the model is given input sequences that contain two sets $A$ and $B$, seperated by a special token that designates the end of set $A$ and start of set $B$. Set elements are tokens $\in [0..|V|]$ for vocabulary size $|V|$ and the model needs to output the tokens in the intersection of $A$ and $B$. 
For example, the correct output below would be $\text{6}$:\footnote{Note that we train the model to output the set intersection, of size $1$, not binary disjointness result (\Cref{alg:synthetic}). 
We find explicitly outputting the intersection helps the model avoid the behavior of outputting $0$ or $1$ with $50\%$ accuracy during training.}
$$
\underbrace{\text{7 11 17 16 4 6 9}}_{\mathclap{\textbf{Set A}}} \text{ * } \underbrace{\text{8 1 5 6}}_{\mathclap{\textbf{Set B}}} \rightarrow  \text{ ? } 
$$
In Figure \ref{fig:synthetic}, we vary the state size of the Based recurrent architecture \citep{arora2024simple}, which has been demonstrated to outperform prior subquadratic models on recall, on the SD task.
We train on sequences where $|A|$ and $|B|$ are between $1$ and $1024$, and $|V| = 2048$. 
In addition to measuring overall accuracy, we consider the sliced accuracy on sequences where $|A| < |B|$ and sequences where $|B| < |A|$. 

We find the causal models achieve better quality when the size of set $A$ is smaller than set $B$. \Cref{fig:synthetic} (Right) shows the difference in quality between when $A$ is shorter vs. longer than $B$, reflecting that the gaps tend to be larger at smaller state sizes ($x$-axis). We additionally evaluate a non-causal variant of the Based architecture and find (1) it outperforms the causal models across state sizes when $A$ is longer than $B$ (\Cref{fig:synthetic} (Left)), and (2) displays less variation in quality as a function of data (set) order \Cref{fig:synthetic} (Right).
We release code to reproduce this plot.

\vspace{-2mm}
\paragraph{Theoretical study: recall and set disjointness.}
In Appendix \ref{app:theory}, we perform a systematic theoretical study of the connection between set disjointness and recall as well as the complexity of solving set disjointness in the JRT setting.

First, we show that set disjointness and the ``general associative recall'' (GAR) problem, which we define in Appendix \ref{app:theory} [Definition \ref{def: nmqar}]), are essentially equivalent (see Propositions \ref{prop: nmqar-sd} and \ref{prop: sd-gar}). Roughly speaking, the keys and queries in GAR correspond to sets $A$ and $B$ in set disjointness.

We argue that recurrent models need space $\Omega(\min(|A|, |B|))$ for solving set disjointness, and hence, GAR (see Proposition \ref{prop: jrp-sd} in Appendix \ref{app: jrp-gar}).
\begin{proposition}
    Given a {\em $\jrp$ prompt}\footnote{A $\jrp$ prompt is simply repeating the input $p$ times (see Definition \ref{def:jrp}).} $\vu^{\jrp} \in \{0,1\}^{p\inputLength \times \modelDim}$ for input $\vu \in \{0,1\}^{\inputLength \times \modelDim}$ to the $\nmqar$ problem, any recurrent model $\calM_{\nmqar}$ (\cref{def: reg-model}) solving $\nmqar$ requires its state size to be at least $\Omega\paren{\frac{\min\{|A|, |B|\}}{p}}$-bits.
\end{proposition}
That is, the lower bound holds even if we allow multiple, but constant, many passes, as opposed to $\Omega(\max(|A|, |B|))$ lower bound for recurrent models without repeats~\citep{arora2024simple} \textbf{Theorem F.3}.

Next, we show we can indeed achieve this lower bound. We show that certain recurrent models (concretely, a slight variant of \Based) can solve SD with $O(\min(|A|, |B|))$ space in the {\prompt} setting (App. \ref{app:theory-based}).
    \begin{theorem}
    Given a $\jrt$ prompt $\vu^{\jrt} \in \R^{2N \times (\sdDim + 1)}$ of the input $\bm{u} \in \R^{\inputLength \times (\sdDim+1)}$ for the set-disjointness (SD) problem $(A, B) \subseteq \{0,1\}^n$, there exists a \Based\ model (\BaseConv\ + MLP + \LinAtt\ + MLP)\footnote{This matches the architecture in our experiments.} that solves SD with space $O(\min\{|A|, |B|\}\cdot n)$.\footnote{This bound is for the case where the IP kernel is dependent on $A$ and $B$; if we use an {\em input-independent} IP kernel, then we get an upper bound of $O\paren{(\min\{|A|, |B|\})^2\cdot n}$ (see Remark \ref{rem: based-sd-2}). Further, this result needs one layer of \BaseConv\ where the convolution kernel is input dependent as well.}
\end{theorem}

Finally, we show that this improvement via JRT-prompting is not realizable for all possible architectures. In particular, we show that $\Omega(\max\{|A|, |B|\}) = \Omega(N)$ lower bounds for the \BaseConv\ model (a model that provably simulates any gated convolution, e.g. Hyena \citep{poli2023hyena}, H3 \citep{fu2023simple}, with just poly-log blowup in parameters and depth) (Theorems F.4, F.5, and F.6, \cite{arora2024simple}) for recall carry over even in the JRT-prompt setting~(see Theorems \ref{thm: JRT-layer-ar}, \ref{thm: mqar-1hot-JRT}, and \ref{thm: mqar-phot}).

\vspace{-2mm}
\subsection{Consequences of analysis on downstream in-context learning with large language models}
\label{sec:sec4-jrt-prompt}
We next show that our analysis holds consequences for in-context learning on real-world tasks. 

\vspace{-2mm}
\paragraph{{\prompt} approach.} In-context learning tasks take as input $(\mathcal{C}, \mathcal{Q}, \mathcal{Y})$ where $\mathcal{C}$ is some context (e.g., document or code repository), $\mathcal{Q}$ is some question or request to the model given the context, and $\mathcal{Y}$ is the answer. 
For standard in-context learning with autoregressive LM $\mathcal{A}$, we input $\mathcal{C}$ and $\mathcal{Q}$ and evaluate the generated output $\mathcal{\hat{Y}} = \mathcal{A}(\mathcal{C}, \mathcal{Q})$ against the true completion $\mathcal{Y}$.

We propose {\prompt}, an exceedingly simple method in which information from the prompt (e.g. questions and documents) is \textit{repeated} in-context before the model is prompted to output the answer, \textit{e.g.,} $\hat{\mathcal{Y}} = \mathcal{A}(\mathcal{C}, \mathcal{Q}, \mathcal{C}, \mathcal{Q})$, as depicted in Figure \ref{fig:main_fig_jrt} (Right). 
As a result, during the second occurrence of the context, the model can condition on a full view of the context when deciding what to store. We provide the prompts that we use in Appendix \ref{app:error_analysis}, and release our code to reproduce the table.

\begin{table*}[]
\centering
\scriptsize
\begin{tabular}{lcc|cccccc|c}
\toprule
\multirow{1}{*}{Architecture} &
\multirow{1}{*}{Params} &
\multirow{1}{*}{Tokens}  & 
  FDA &
  SWDE &
  NQ &
  SQUAD &
  TriviaQA & 
  Drop & 
  Average 
  \\ \hline \hline

Transformer++   &  1.3B  & 10B &  74.4/\textbf{86.1} & 41.4/\textbf{52.5}  & 28.2/\textbf{31.9}  & 39.0/\textbf{53.1}  & \textbf{49.5}/49.3 & 22.3/\textbf{33.6} & 42.5 / \textbf{51.1} \\
Mamba           &  1.3B  & 10B & 23.3/\textbf{40.3}  & 15.5/\textbf{31.8} &  19.4/\textbf{25.8}  & 26.6/\textbf{48.5}  & 46.4/\textbf{51.1}  &  21.3/\textbf{32.1}  &  25.1 / \textbf{38.2} \\
Based           &  1.3B  & 10B & 48.6/\textbf{58.9}  & 27.6/\textbf{44.7} &  19.7/\textbf{28.4}  & 31.0/\textbf{46.7} &  44.1/\textbf{51.9} &  19.5/\textbf{34.6} & 31.8 / \textbf{44.2}  \\
\hline
Transformer++   & 1.3B &  50B &  83.7/\textbf{89.2} &  50.8/\textbf{65.0} & 32.8/\textbf{37.5}  & 41.1/\textbf{58.1}  & 56.6/\textbf{58.8}  &  21.5/\textbf{37.9} & 47.8 / \textbf{57.8} \\
Mamba           & 1.3B &  50B & 41.9/\textbf{55.7}  & 32.6/\textbf{45.4}  & 26.9/\textbf{33.9} & 31.5/\textbf{53.5} & 54.9/\textbf{56.7} & 20.4/\textbf{33.8} & 34.7 / \textbf{46.5} \\
Based           & 1.3B &  50B &  60.2/\textbf{68.3} & 37.1/\textbf{54.0}  & 29.4/\textbf{35.2} & 38.9/\textbf{56.3} & 54.5/\textbf{57.6}  &  21.7/\textbf{39.1} & 40.3 / \textbf{51.8}  \\ 
\hline
GLA  & 1.3B & 100B  & 48.3/\textbf{68.6} & 37.7/\textbf{53.6} & 26.6/\textbf{31.3} & 34.7/\textbf{54.8}  & \textbf{55.5}/54.6 & 19.6/\textbf{33.3} &  36.7 / \textbf{48.9} \\ 
GLA & 2.7B  & 100B  & 47.1/\textbf{65.8}  & 43.6/\textbf{54.5} & 27.1/\textbf{32.9} & 37.2/\textbf{55.7} &  \textbf{57.9}/57.0 & 22.2/\textbf{34.0} & 39.2/ \textbf{50.0} \\ 
\hline
Mamba           & 130M &  300B  & 25.7/\textbf{32.8} &  17.5/\textbf{31.5} &  16.8/\textbf{21.7}   & 27.1/\textbf{51.9} & 43.5/\textbf{50.1} & 17.4/\textbf{30.7} & 24.7 / \textbf{36.5} \\
Mamba           & 370M &  300B & 41.9/\textbf{58.3}  &  27.6/\textbf{42.2} & 23.8/\textbf{31.1} & 34.9/\textbf{51.0} & \textbf{53.6}/51.7 & 19.3/\textbf{33.2} & 33.5 / \textbf{44.6}  \\
Mamba           & 1.4B &  300B & 45.8/\textbf{60.9}  &  37.6/\textbf{46.0} & 31.0/\textbf{36.6} & 39.9/\textbf{59.6} & 60.5/\textbf{61.3} & 20.9/\textbf{36.4} & 39.3 / \textbf{50.1}  \\
Mamba           & 2.8B &  300B &  54.3/\textbf{66.6} &  38.9/\textbf{48.9} & 33.5/\textbf{40.1}  & 43.9/\textbf{59.4} & \textbf{66.2}/63.9 & 19.8/\textbf{36.9} & 42.8 / \textbf{52.6} \\ 
\hline
Mamba-2           & 130M &  300B  & 32.2/\textbf{50.9} & 29.5/\textbf{43.3} & 20.6/\textbf{28.9}  & 30.4/\textbf{47.0} & 43.7/\textbf{47.2} & 18.0/\textbf{34.0} & 29.1 / \textbf{42.0} \\
Mamba-2           & 370M &  300B  & 60.8/\textbf{76.7} & 38.3/\textbf{52.1} &  26.6/\textbf{33.6}  & 35.3/\textbf{51.8} & 54.6/\textbf{54.7} & 22.4/\textbf{36.3} & 39.7 / \textbf{50.9} \\
Mamba-2           & 1.3B &  300B  & 66.8/\textbf{74.7} & 50.0/\textbf{59.6} & 33.6/\textbf{40.5}  & 42.9/\textbf{59.6} & \textbf{63.8}/62.4 & 23.2/\textbf{36.6} & 46.7 / \textbf{55.6} \\
Mamba-2           & 2.7B &  300B  & 68.7/\textbf{81.6} & 55.2/\textbf{60.8} & 34.4/\textbf{41.7}  & 45.4/\textbf{59.4} & 66.4/\textbf{66.5}  & 23.0/\textbf{42.5} & 48.9 / \textbf{58.8} 
 \\ \hline
\end{tabular}
\caption{\textbf{Evaluation of pre-trained language models.} 
In each cell, we report in-context learning accuracy for the default zero-shot / {\prompt} methods (using prompts provided in \Cref{app:prompts}). We evaluate across a suite of popular recall-intensive benchmarks. The zero-shot prompt includes up to 1k tokens in the input and {\prompt} includes up to 2k tokens in the input for all tasks (due to repeating twice).}
\label{table:main-quality}
\vspace{-3mm}
\end{table*}
\vspace{-2mm}
\paragraph{Evaluation.} {\prompt} can be used with off-the-shelf LLMs.  We evaluate the following LMs on a suite of recall-intensive in-context learning tasks, with zero-shot prompting:
\begin{itemize}
    \item \textbf{Based} \citep{arora2024simple} pretrained LMs at the $1.3$B parameter scale trained on $10-50$B tokens of the Pile \citep{pile}. Transformer++ and Mamba models trained on the exact same tokens and data order are provided for quality references: \url{https://huggingface.co/collections/hazyresearch/}
    \item \textbf{Mamba} \citep{gu2023mamba} pretrained LMs at the $130$M, $370$M, $1.4$B, $2.8$B parameter scales, trained on $300$B tokens of the Pile \citep{pile}: \url{https://huggingface.co/state-spaces}
    \item \textbf{Gated Linear Attention} \citep{yang2023gated} pretrained LMs at the $1.3$B and $2.7$B parameter scales, trained on $100$B tokens of SlimPajama data \citep{together2023redpajama}: \url{https://huggingface.co/fla-hub}
    \item \textbf{Mamba-2} \citep{dao2024transformers} pretrained LMs at the $130$M, $370$M, $1.3$B, $2.7$B parameter scales, trained on $300$B tokens of the Pile \citep{pile}: \url{https://huggingface.co/state-spaces}
\end{itemize}

The results are summarized in Table \ref{table:main-quality}. \citet{arora2024simple} finds that linear recurrent models like Mamba drastically underperform Transformers on these recall-intensive tasks. Architectures like Based increase the recurrent state size, improving both quality and efficiency, and recently Mamba-2 adopts this approach as well. Complementing the approach of increasing state size, we find the {\prompt} modification provides $11.0 \pm 1.3$ points of improvement, averaged across models and tasks: Based models with {\prompt} outperform the Transformer models with standard prompting on average. 
We also find that {\prompt} can benefit the Transformer models and that the method appears more effective than few-shot learning for these tasks (\Cref{app:error_analysis}). Notably, \citet{springer2024repetition} recently proposes repeating the context for the goal of generating embeddings using autoregressive Transformer-based models, and our findings are in similar spirit. We focus on sub-quadratic architectures and in-context learning tasks.

{\prompt} increases the context length due to repetition, however using using sub-quadratic recurrent architectures, this is still asymptotically more efficient than using quadratic Transformer models. We find at sequence length $N=32768$, batch size $16$, Based with {\prompt} ($2N$ the sequence length) can provide $11.9\times$ higher throughput than FlashAttention-2 ($N$ sequence length) on an NVidia H100 (see \Cref{sec5-results}).

\section{{\arch}: an encoder-decoder recurrent architecture}

\label{sec:sec4-jrt-arch}
We have shown that the recall quality of causal fixed-memory recurrent models varies depending on the order in which the information appears in context, making them brittle for in-context learning. To improve reliability, we next propose a simple linear attention architecture that goes beyond causal modeling. 

A long line of work has demonstrated the strength of non-causal bidirectional neural networks in language modeling \citep{schuster1997bidirectional, kosko1988bidirectional, graves2005framewise,
devlin2020transformers, 
2020t5, patel2023bidirectional}. 
However, it is challenging to use them for fast text generation because the context must be re-processed for each generated token \citep{tay2023ul2, dong2019unified, patel2023bidirectional}. 
Encoder-decoder architectures with a bidirectional encoder and causal decoder offer a way to achieve fast causal generation while reaping the benefits of bidirectional LMs.  
Nonetheless, decoder-only causal LMs remain the norm and encoder-decoder architectures have received little attention in the context of sub-quadratic efficient LLMs.

\vspace{-2mm}
\subsection{Preliminaries}
\paragraph{Baseline linear recurrent architecture.}
We start from a recurrent architecture, linear attention, introduced in \citep{katharopoulos2020transformers,tsai2019transformer,choromanski2020rethinking}. Current strong recurrent LMs (e.g., Based \citep{arora2024simple}, GLA \citep{yang2023gated}, Mamba-2 \cite{dao2024transformers}) adopt linear attention with large recurrent state sizes. 
Prior work also theoretically shows that linear attention and state space models like Mamba \cite{gu2023mamba} are closely related \cite{arora2023zoology, arora2024simple, dao2024transformers}.

Let $\vq$, $\vk$, $\vv$ be linear projections of the input $\vu \in \mathbb{R}^{N \times d}$. 
  The exponential in softmax attention is replaced by a feature map $\phi: \mathbb{R}^d \rightarrow \mathbb{R}^{\tilde{d}}$, from model dimension $d$ to feature dimension $\tilde{d}$, such that $\phi(\vq_i)^\top \phi(\vk_j) \approx \exp(\vq_i^\top \vk_j / \sqrt{d})$. 
The linear attention computation can then be written as:
\begin{equation}
\label{eq:linear_attention}
\vy_i
= \frac{\phi(\vq_i) \sum_{j = 1}^i \big( \phi(\vk_j)^\top \vv_j \big)
}{\phi(\vq_i) \sum_{j = 1}^i \phi(\vk_j)}
\end{equation}
Multiplying keys and values first, the time and space complexity is $\mathcal{O}(Nd\tilde{d})$ vs. $O(N^2d)$ for softmax attention.

Recurrent inference is split into two phases: \textit{prefill} to process the input prompt and \textit{decoding} to generate one token of the output at a time. 
During \textit{prefill}, a length-$\nPre$ prompt is processed in parallel according to \Cref{eq:linear_attention} resulting in a ``KV-state'' 
$\vs_\nPre = \sum_{j=1}^{\nPre} \phi(\vk_{j})^\top \vv_{{j}}$ and ``K-state'' $\vz_\nPre = \sum_{j=1}^{\nPre}\phi(\vk_j)^{\top}$. 
During \textit{decoding}, we can compute \Cref{eq:linear_attention} as:
\begin{equation}
\label{eq:linear_attention_recurrent_3} 
    \vs_i = \vs_{i-1} + \phi(\vk_i)^\top \vv_i, \qquad
    \vz_i = \vz_{i - 1} + \phi (\vk_i)^\top, \qquad
    \vy_i = \frac{\phi(\vq_i)\vs_i}{ \phi(\vq_i)\vz_i}
\end{equation}

where $\bm{s}_{i} \in \mathbb{R}^{d \times \tilde{d}}$ and $\bm{z}_i \in \mathbb{R}^{\tilde{d}}$. Each decode step has $O(1)$ time and space complexity as the sequence length grows, improving upon $O(N)$ for softmax attention with KV-caching.

\paragraph{Prefix-LM architecture.} Prefix-LM is a category of encoder-decoder models where inputs of length $N$ are split into two regions: the first of length $M$ is processed non-causally and the latter of length $(N-M)$ is processed causally \citep{2020t5}. During loss computation, the former tokens are ignored and next-token-prediction loss is computed on the latter region. Excitingly, the design is quite simple, however prior instantiations of Prefix-LMs use inefficient softmax attention  backbones and have not provided compelling benefits over decoder-only Transformers \citep{wang2022language}. Prior prefix LM architectures have seen limited adoption.

\subsection{${\arch}$ architecture}
{\arch} draws inspiration from Prefix-LMs, but focuses on expanding the Pareto frontier of the quality-efficiency tradeoff space. To improve quality, {\arch} uses separate $\vk_{e}$, $\vv_{e}$ projections on the encoder side and $\vk_{d}$, $\vv_{d}$ projections on the decoder side. While Prefix LM models use shared projection weights for the encoder and decoder regions, we find that using two sets of projections improves quality. This observation appears in early work on recurrent encoder-decoder architectures (Sutskever et al. \cite{sutskever2014sequence}).

For efficiency, {\arch} uses non-causal linear attention for the encoder plus standard causal linear attention for the decoder. We term this Prefix Linear Attention (PLA) (\Cref{fig:main_fig_jrt} (Right)):
\begin{equation}
\label{eq:linear_attention_jrt} 
     \vy_i = \frac{ \phi(\vq_i) (\sum_{j = 1}^i \phi(\vk_{d_j})^\top \vv_{d_j} + \sum_{j = 1}^{M} \phi(\vk_{e_j})^\top \vv_{e_j})}
     {\phi(\vv{q_i}) (\sum_{j = 1}^i \phi(\vk_{d_j})^\top +  \sum_{j = 1}^{M}  \phi(\vk_{e_j})^\top)}
\end{equation}

Prior work has proposed many different instantiations of linear attention by varying the feature map ${\phi}$ -- PLA is a general approach, agnostic to the choice of feature map.   

PLA retains the linear recurrent view, $\mathcal{O}(1)$ time and space complexity for the inference decode step and the sub-quadratic in sequence length training complexity  of standard causal linear attention \citep{katharopoulos-et-al-2020}. 
During \textit{prefill}, we process a length-$\nPre$ prompt in parallel according to \Cref{eq:linear_attention_jrt}. If $\nPre < M$, we left-pad the prefill to length $M$ and mask the padded region during the linear attention computation. The recurrent state is initialized as: 
\begin{equation}
\label{eq:linear_attention_recurrent_4} 
    \vs_M = \sum_{j = 1}^{M} (\phi(\vk_{e_j})^\top \vv_{e_j} +  \phi(\vk_{d_j})^\top \vv_{d_j}), \qquad
    \vz_M = \sum_{j=1}^{M} (\phi(\vk_{e_j})^{\top} + \phi(\vk_{d_j})^{\top})
\end{equation}

Decoding for outputs $y_i, i > M$ proceeds according to  \Cref{eq:linear_attention_recurrent_3}, without modification.

\paragraph{Efficiency.} Although linear attention is theoretically more efficient than softmax attention, existing implementations are generally \textit{slower} than well-optimized standard attention implementations (e.g., FlashAttention \cite{dao2023flashattention2}). 
Excitingly, \cite{arora2024simple} recently provides an IO-aware kernel that realizes the efficiency benefits of the Based linear attention architecture by carefully paritioning and storing the large matrix-valued recurrent state across warp-registers during prefill (Algorithm 1 in \cite{arora2024simple}).
We extend their algorithm to support PLA, using the Based feature map (defined in \Cref{app:architecture}) in Algorithm \ref{alg:jrt_rnn} and provide the efficiency results in Section \ref{sec5-results}.
Additional details of our implementation are provided in \Cref{app:architecture}.

The baseline causal linear attention takes $2BNHD$ FLOPS to compute the feature map on  $\vq_{d}$, $\vk_{d}$, and $4BNHdD$ FLOPS for the $\vk_{d}$, $\vv_{d}$ dot product, cumulative sum, $\vq_{d}$ dot product, and sum along the feature dimension $D$ respectively. 
PLA increases the FLOPS by $BMHD$ to compute the feature map on $\vk_{e}$ and $3BMHdD$ to compute the $\vk_{e}$, $\vv_{e}$ dot product, sum along $D$, and sum the state with the decoder KV-state.
PLA uses the same amount of memory (recurrent state size) during the inference decoding step as the original causal linear attention architecture.

\subsection{{\arch} training objective} Our baseline recurrent models are trained with a standard next token prediction (NTP) objective, learning a probability distribution $\mathrm{P}(u_{i+1}|\{u_1, ..., u_i\})$  from input sequences of tokens $\mathbf{u} = \{u_1, ..., u_N\}$ for sequence length $N$, and cross-entropy loss. For the pure decoder models, the loss ($\mathcal{L}_{\mathrm{NTP}}$) is computed using all $N$ tokens in $\mathbf{u}$. {\arch}, as is standard for Prefix-LMs, an only compute the NTP loss ($\mathcal{L}_{\mathrm{NTP}}$) for tokens $\{u_{M}, ..., u_{N}\}$, which are processed causally. 

Prefix LMs typically compute no loss on the non-causal region, however in {\arch}, we combine next token prediction with the masked language modeling (MLM) objective \citep{devlin2020transformers}. For the added MLM objective, we replace proportion $P$ of of tokens from the encoder region $\{u_{1}, ..., u_{M}\}$ with a $[\mathrm{MASK}]$ token 
and we measure the cross-entropy loss ($\mathcal{L}_{\mathrm{MLM}}$) in predicting the original token. The loss is: 
\begin{equation}
\label{eq:loss} 
\mathcal{L} = \frac{w_1\mathcal{L}_{\mathrm{NTP}} + w_2\mathcal{L}_{\mathrm{MLM}}}{w_1+w_2}
\end{equation}
where $w_1, w_2 \in \mathbb{R}$ are scalar weights. During inference, no $[\mathrm{MASK}]$ tokens are used; inference proceeds as with causal LMs.

\section{Results}

\label{sec5-results}

In this section, we validate the following quality and efficiency claims for {\arch}:
\begin{enumerate}[itemsep=0.1pt,topsep=0pt,leftmargin=*]
    \item \textbf{In-context learning (ICL) quality} {\arch} provides $99\%$ of Transformer quality at $360$M params./$30$Bn tokens, averaged across the recall-intensive ICL benchmarks. This represents $46.7\%$ improvement over Based and $78.8\%$ over Mamba. {\arch} provides $96\%$ of Transformer quality at $1.3$Bn params./ $50$Bn tokens, representing  $16.2\%$ improvement over Based and $34.5\%$ over Mamba on average.
    \item \textbf{Overall language modeling} Beyond outperforming in recall, we show that {\arch} matches the baselines in general natural language understanding (SuperGLUE). We give a detailed analysis of the pretrained LMs, comparing perplexity on slices of the Pile test set to show the strengths and limitations.
    \item \textbf{Generation} We show that {\arch} can provide $19.2\times$ higher prefill throughput than FlashAttention-2 at $32\mathrm{k}$ sequence length, batch size $16$ on an NVidia H100 GPU. 
\end{enumerate}

\paragraph{Models.}
We compare {\arch} to two state-of-the-art recurrent autoregressive models, Based \citep{arora2024simple} and Mamba \citep{gu2023mamba}. We also compare to the Transformer++ (Llama architecture \citep{touvron2023llama}), which adds rotary encodings \citep{su2023roformer} and gated linear units. 

For {\arch}, we start from the Based linear recurrent architecture, since it has been shown in prior work to outperform prior sub-quadratic architectures (e.g., Mamba, GLA) at recall.
An extended explanation of Based is in \Cref{app:architecture}. 
We reiterate that the approaches in {\prompt} and {\arch} can be combined with any linear recurrent model.

\vspace{-3mm}
\paragraph{Benchmarks.}
We evaluate on a range of ICL benchmarks. We use SuperGLUE to test general language understanding \citep{wang2019superglue}.
We next evaluate on a suite of recall-intensive tasks including:  SWDE and FDA information extraction tasks \citep{wu2021medai, deng2022domlm,  arora2023evaporate, arora2024simple}, where the model needs to extract values for a specified attribute from in-context passages, and SQUADv2 \citep{rajpurkar2018squad}, Natural Questions \citep{kwiatkowski-etal-2019-natural}, TriviaQA \citep{joshi2017triviaqa}, and Drop \citep{Dua2019DROP}. 
In these tasks, the model needs to ground its answers in in-context documents. 
We release code and models to reproduce our results and provide details on the benchmarks and evaluations in \Cref{app:expt-details}.
\begin{table*}[]
\centering
\scriptsize
\begin{tabular}{lc|cc|cc|cc|ccc|c}
\toprule
\multirow{3}{*}{Architecture} &
\multirow{3}{*}{Param/Tok} &
\multicolumn{2}{c}{\textbf{FDA}} &
\multicolumn{2}{c}{\textbf{SWDE}} & 
\multicolumn{2}{c}{\textbf{NQ}}  & 
\multicolumn{1}{c}{\textbf{SQUAD}}  & 
\multicolumn{1}{c}{\textbf{Trivia}}  & 
\multicolumn{1}{c}{\textbf{Drop}} &
\textbf{Avg.}
   \\
   & 
   &
  $\mathrm{512}$  &
  $\mathrm{1024}$ &
  $\mathrm{512}$  &
  $\mathrm{1024}$ &
  $\mathrm{512}$ &
  $\mathrm{1024}$ &
  $\mathrm{Full}$ &
  $\mathrm{Full}$ & 
  $\mathrm{Full}$ 
  
   \\
   & 
   &
  Acc $\uparrow$ &
  Acc $\uparrow$ &
  Acc $\uparrow$ &
  Acc $\uparrow$ &
  Acc $\uparrow$ &
  Acc $\uparrow$ &
  Acc $\uparrow$ &
  Acc $\uparrow$ &
  Acc $\uparrow$ & 
  Acc $\uparrow$ 
  \\ 
  \hline \hline
Transformer     &  360M/30B  & 
74.8 & \textbf{73.0} & \textbf{44.7} & \textbf{43.0} & \underline{27.8} & \textbf{22.9} & \underline{36.2} & \textbf{46.5} & \underline{21.8}  & \textbf{43.4} \\
Mamba              &  360M/30B  & 
41.1 & 24.3 & 22.2 & 13.6 & 16.4 & 12.5 & 25.5 & 43.0 & 17.3 & 24.0 \\
Based   &  360M/30B  & 
50.3 & 35.8 &  30.4 & 21.6 & 19.7 & 14.7 & 29.8 &  42.5 & 18.4 & 29.2\\
{\arch} &  360M/30B & 
\textbf{82.0} & \underline{66.0} & \underline{43.3} & 35.1 & \textbf{32.9} & \underline{16.2} & \textbf{41.7} & \underline{43.2} & \textbf{25.8} & \underline{42.9} \\
\hline
Transformer      &  1.3B/10B  & 
\underline{75.3} & \textbf{71.5}  & \textbf{41.6} & \textbf{41.0} & \textbf{29.6} & \textbf{25.8} & \underline{38.7} & \textbf{48.8}  & \underline{22.6} & \textbf{43.9} \\
Mamba              &  1.3B/10B  & 
37.4 & 23.3 &  23.0 & 15.1 & 19.6 & 16.1 & 26.1 & \underline{45.7} & 20.9 & 25.2\\
Based              &  1.3B/10B  & 
66.3 & 49.0 &  32.3 & 26.3 & 19.7 & 15.7 & 30.7 & 44.2 & 19.1 & 33.7  \\
{\arch}            &  1.3B/10B  & 
\textbf{78.5} & \underline{60.6 }& \underline{38.5} & \underline{32.7} & \underline{26.5} & \underline{16.7} & \textbf{51.6} & 44.8 & \textbf{28.4} & \underline{42.0} \\
\hline
Transformer      &  1.3B/50B  & 
\underline{85.6} & \textbf{83.5} & \textbf{55.7} & \textbf{56.0} & \underline{33.4} & \textbf{29.9} & \underline{40.1} &  \textbf{56.6} & \underline{21.4} &  \textbf{51.4}\\
 Mamba              &  1.3B/50B  & 
55.4  &  40.1 & 44.0 & 33.7 & 27.6 &  23.2 & 32.2 & \underline{54.5}  & 20.7 & 36.8 \\
Based              &  1.3B/50B  & 
69.3 & 58.8 & 47.6 & 40.4 & 29.1 & 24.4 & 38.5 & 54.3  & 20.8 & 42.6\\
{\arch}            &  1.3B/50B  & 
\textbf{86.7} & \underline{67.7} & \underline{49.4} & \underline{45.7} & \textbf{38.3} & \underline{25.4} & \textbf{50.4} &  53.0 &  \textbf{29.3} & \underline{49.5} \\
\hline
\end{tabular}
\caption{\textbf{Evaluation of {\arch} models.}
We compare {\arch} to strong LMs proposed in prior work (Based, Mamba, and Transformer++) across parameter scales. In the table, we specify the length (number of tokens) of the documents provided in context ($512$, $1024$, Full), where ``Full'' means the full document is included as prefill. \Cref{table: datasets-overview} contains the average number of tokens per document in each benchmark.
}
\label{table:jrt-rnn-quality}
\vspace{-4mm}
\end{table*}

\vspace{-2mm}
\subsection{In-context learning quality}
In \Cref{table:jrt-rnn-quality}, we find {\arch} outperforms the decoder-only baseline (Based) by $13.7$ points at $360\mathrm{M}$ parameters (30Bn tokens) and $6.9$ points at $1.3\mathrm{B}$ parameters (50Bn tokens) on average. {\arch} closes the gap to Transformer++ to within $0.5$ points on average at $360\mathrm{M}$ and $1.9$ points on average at $1.3\mathrm{B}$ parameters.

In \Cref{table:jrt-rnn-quality}, we left pad documents with length $<M$, where $M=1024$ is the encoder region's length during training (discussed in \Cref{sec:sec4-jrt-arch}) -- for the three results with length $512$ documents we pad using {\prompt} and otherwise with the tokenizer's space token (discussed further below).

\begin{table}
\label{tab:extrapolation}
\scriptsize
\centering
\begin{tabular}{lc|ccc}
\toprule
\multirow{2}{*}{Arch.} &
\multirow{2}{*}{Param/Tokens} &
\multicolumn{1}{c}{\textbf{FDA}} &
\multicolumn{1}{c}{\textbf{SWDE}} 
& 
\multicolumn{1}{c}{\textbf{NQ}} 
   \\
   & 
   &
  $\mathrm{2k}$ & 
  $\mathrm{2k}$ 
  & 
  $\mathrm{2k}$ 
  \\ 
  \hline \hline
Transformer      &  360M/10B  &  65.2 & 41.0 & 23.0   \\
Mamba    &  360M/10B  &  12.4  & 13.4  & 12.4 \\
Based    &  360M/10B  &  19.1  & 18.9  & 13.9 \\
{\arch}  &  360M/10B  &  28.4  & 26.1  & 15.4 \\ 
\hline
Transformer  &  1.3B/50B  & 79.7 & 55.5 & 30.2 \\
Mamba        &  1.3B/50B  & 21.0 & 29.9 & 23.1 \\
Based        &  1.3B/50B  & 36.1 & 37.7 & 23.4 \\
{\arch}      &  1.3B/50B  & 55.2 & 41.4 & 26.2 \\
\hline
\end{tabular}
\caption{Evaluation at prefill lengths $\mathrm{2k}$, i.e. beyond  the encoder region (length $M=1024$).}

\label{tab:inference_jrt}
\scriptsize
\centering
\begin{tabular}{lc|ccc}
\toprule
\multirow{2}{*}{Inference} &
\multirow{2}{*}{Param/Tokens} &
\multicolumn{1}{c}{\textbf{FDA}} &
\multicolumn{1}{c}{\textbf{SWDE}} 
& 
\multicolumn{1}{c}{\textbf{NQ}}  \\
   & 
   &
  $\mathrm{512}$ & 
  $\mathrm{512}$ & 
  $\mathrm{512}$ \\ 
  \hline \hline
Left-pad        &  360M/30B  & 61.9 & 38.1 & 24.6 \\
Read-$2\times$  &  360M/30B  & 82.0 & 43.3 & 32.9 \\
Iterate         &  360M/30B  & 76.3 & 40.7 & 29.2 \\
\hline
Left-pad        &  1.3B/50B  & 75.8 & 49.3 & 30.9 \\
Read-$2\times$  &  1.3B/50B  & 86.7 & 49.4 & 38.3 \\
Iterate         &  1.3B/50B  & 80.2 & 43.3 & 34.2 \\
\hline
\end{tabular}
\caption{{\arch} with alternate inference strategies when $\nPre < M$, for prefill and encoder lengths $\nPre$ and $M$.}
\vspace{-3mm}
\end{table}

\paragraph{Length extrapolation.} 
Though the encoder processes until length $M=1024$ for our trained LMs, we excitingly find that the benefits of JRT extend to prefill lengths $\nPre$ s.t. $\nPre > M$ as well. In \Cref{tab:extrapolation}, we 
evaluate at the $360\mathrm{M}$ and $1.3\mathrm{B}$ parameter scales with documents of length $2000$. 

\vspace{-2mm}
\paragraph{Inference strategies.} 
In \Cref{tab:inference_jrt}, we compare alternate inference strategies for {\arch} in the regime where the prefill length $\nPre$ is less than the encoder length $M$, $\nPre < M$:
\begin{itemize}[itemsep=0.1pt,topsep=0pt,leftmargin=*]
    \item \textbf{Decoding with padding}: We left-pad the prefill to length $M$ to match the training distribution the model sees. Causal decoding starts at position $M$. This is the default for {\arch}.
    \item \textbf{Read-twice pad}: Instead of padding with a special token, we can ``pad'' by repeating the context (i.e., {\prompt}). We use this at $\nPre = 512$ for FDA, SWDE, and NQ in \Cref{table:jrt-rnn-quality}. Padding is a fixed cost for {\arch}, so it can be used creatively. 
    
    \item \textbf{Iterative encoding}: We allow the model to \textit{non-causally view its previously generated tokens} during decoding. We generate token $\vy_{\nPre}$ given the length $\nPre$ prefill, append it to the prefill, and then compute $\vy_{\nPre+1}$ again using the parallel view on the new input of length $\nPre+1$. This protocol is expensive, but future work could consider \textit{periodically} updating the non-causal encoder-state when decoding many tokens.
\end{itemize}

\subsection{Overall natural language understanding}
While recall is important for in-context learning, it is important to validate that the models remain strong in their overall natural language understanding abilities.

\vspace{-2mm}
\paragraph{Language modeling perplexity.} A fundamental challenge is how to compare the inherent quality of models pre-trained with disparate objectives. In our setting, this is challenging since {\arch} additionally minimizes a masked language modeling objective beyond the standard causal next token prediction objective and sees $50\%$ less data than the decoder-only models for the next token prediction task (when $M=1024, N=2048$). Overall {\arch} computes losses on $65\%$ of the number of  training data tokens seen by the decoder-only models (with 15\% masked tokens in the encoder region).  

\begin{figure*}[t]
    \centering
    \includegraphics[width=0.475\linewidth]{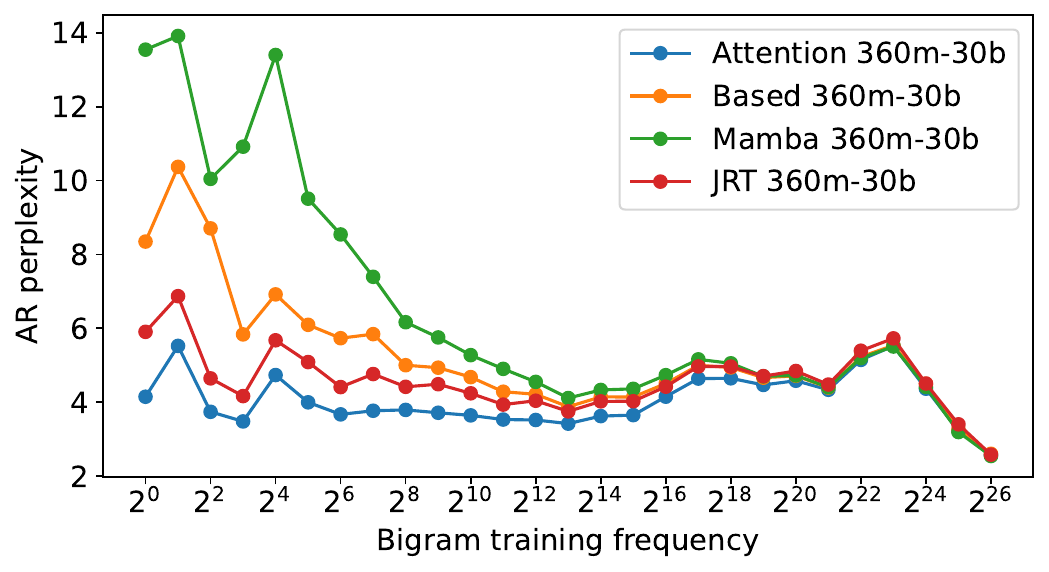}
    \includegraphics[width=0.475\linewidth]{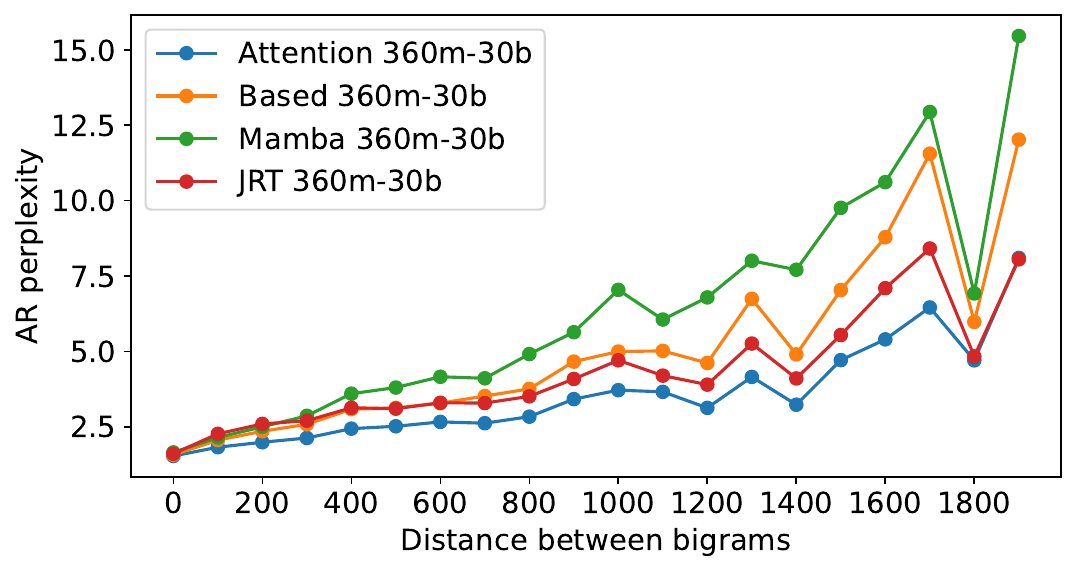}
    \includegraphics[width=0.49\linewidth]{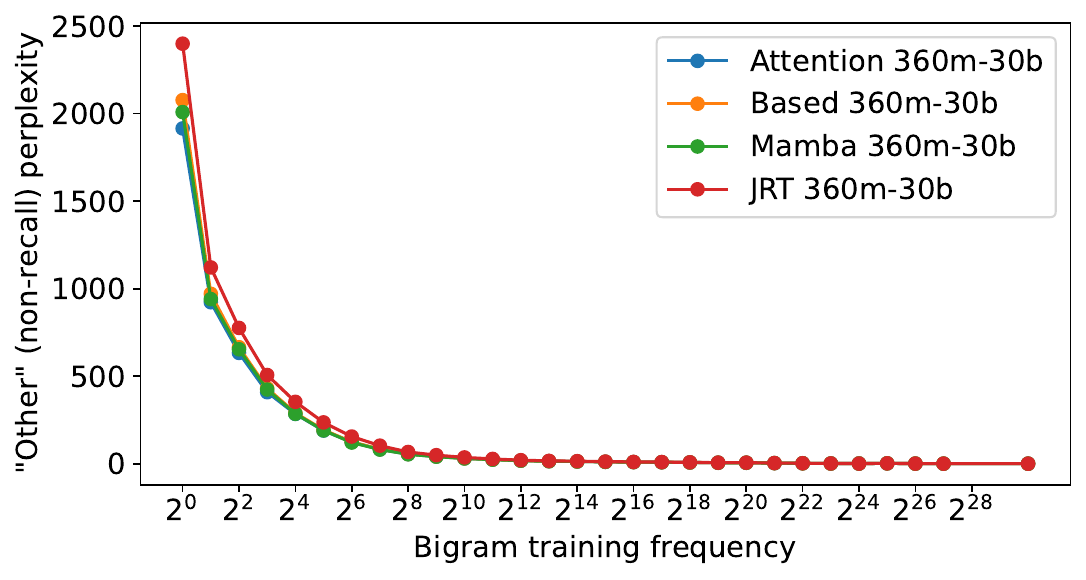}
    \caption{\textbf{Perplexity slices.} We slice the Pile test set perplexities of the pretrained LMs 
    into associative recall ``AR'' and non-recall ``Other'' slices. A token is an AR token if it corresponds to a bigram that is re-occurring in the context, since the LM can look to the prior occurrence to predict the next token (Def. in \Cref{sec:lm_ppl}). \textbf{Top left (recall frequencies)} We plot $y$ perplexity on AR bigram tokens that test the LMs' recall skills based on $x$ the bigram frequency in training. \textbf{Top right (recall distances)} We plot $y$ perplexity for AR tokens based on $x$ the distances between the re-occuring bigrams in context. \textbf{Bottom (non-recall frequencies)} We plot $y$ perplexity on non-recall tokens based on $x$ the bigram frequency in training. Further details are in \Cref{app:expt-details}.}
    \vspace{-3mm}
    \label{fig:ppl_slices}
\end{figure*}

 Despite these differences, we consider a simple proxy of evaluating the perplexity of decoder-baselines in comparison to encoder-decoder {\arch} in the overlapping non-causal regions of both model types (i.e. the last $1024$ tokens per input sequence of $N=2048$ for our trained models). Following prior work \citep{arora2023zoology}, we further \textit{slice} the perplexity in two groups: (1) the associative recall ``AR slice'' includes tokens, referred to as ``AR hits'', that require the model to perform recall in order to predict the next token correctly and (2) the ``Other slice'' containing the remaining tokens (e.g., memorized knowledge). \footnote{As a heuristic rule, a token is an ``AR hit'' if it is completes a bigram that  was previously seen in-context, and this bigram is infrequent during training (i.e., was not memorized by the model) \citep{arora2023zoology}. For instance, in the sequence ``In 1957, Dr. Seuss wrote ... In 1982, Dr. \textbf{\underline{Seuss}}'' the second \textbf{\underline{Seuss}} would be included as an ``AR hit'' if ``Dr. Seuss' is a rare bigram during training.}

Slicing the model predictions on the Pile test set, we observe the following. Our measurement protocols are described in further detail in \Cref{app:expt-details}. 
 \label{sec:lm_ppl}
 \begin{enumerate}
     \item \textbf{Recall frequencies.} {\arch} excels in the ``AR slice''. For infrequently seen bigrams during training (unlikely to be memorized in the model parameters), {\arch} improves in perplexity relative to Based and Mamba, two strong causal recurrent baselines (\Cref{fig:ppl_slices}, top right). 
     \item \textbf{Recall distances.} In the ``AR slice'', the gap between {\arch} and the decoder-only baselines grows as the distances between repeated bigrams seen in-context grows. This provides further support beyond \Cref{tab:extrapolation} that {\arch} can help with longer context recall tasks (\Cref{fig:ppl_slices}).
     \item \textbf{Non-recall frequencies.} {\arch} is worse in perplexity than the decoder-only LMs for the \textit{non-recall} ``Other slice'' for bigrams that are rarely seen during training. This slice tests the model's use of memorized knowledge (as opposed to knowledge provided in the context). This is expected as {\arch} computes losses $65\%$ of the tokens of the decoder-only LMs. We expect this gap to decrease with scale and longer training durations (seen as the bigram frequencies increases) (\Cref{fig:ppl_slices}, top left).  Future work could also consider decoupling sequence mixers from MLPs (knowledge stores) in training. How best to normalize training between encoder-decoder and decoder-only LMs is an open question.
 \end{enumerate}

\vspace{-2mm}
\paragraph{Natural language understanding benchmarks.} We use the downstream SuperGLUE benchmark, a canonical test of natural language understanding ability \citep{wang2019superglue}, to evaluate each architecture at the $360\mathrm{M}$ and $1.3\mathrm{B}$ parameter scales in \Cref{tab:super-glue}. 
We validate that the different architectures perform  similarly on average across these generic, short-context language tasks as observed in prior work \cite{fu2023monarch, karami2024orchid, arora2024simple}.

\vspace{-2mm}
\subsection{Generation throughput}
Generation can be decomposed into prompt ``prefill processing'' and decoding ``next token prediction'' steps. Since {\arch} does not modify the decoding step relative to standard decoder-only recurrent models, we focus our discussion on the prefill stage.

\begin{table*}[h!]
    \centering
    \scriptsize
    \caption{
    Latency (ms) of inference prefill for each implementation. Each point is the average of $20$ iterations, run on an NVIDIA H100 GPU. In \Cref{tab:gpu_inference_jrt_seq}, we vary the  sequence length at a fixed batch size of $16$. In \Cref{tab:gpu_inference_jrt}, we vary the batch size at a fixed sequence length of $16384$.}
    \label{tab:gpu_inference_jrt_seq}
    \begin{tabular}{@{}rccccc@{}}
    \toprule
    \textbf{Implementation}  & \textbf{2048}  & \textbf{4096}  & \textbf{8192}  & \textbf{16384} &  \textbf{32768}   \\ \midrule
    Based PyTorch        & 17.1 &  74.5  & 284.6  &  OOM   & OOM    \\
    Fast Transformer CUDA  & 11.4  &  23.0  &  47.0 &  96.0  & OOM  \\ 
    Based Triton (FLA)   & 1.0  &  2.8   & 9.3    &  32.6  & 123.7  \\
    Based Custom CUDA    & 0.3  &  0.6   & 1.2    &  2.3   & 4.5    \\
    \hline
    FlashAttention-2  &  0.5  &  1.8  & 6.8  &   26.6 & 107.8  \\
      \midrule
       {\arch} PyTorch   & 21.3  &  89.2  & OOM  &  OOM  & OOM  \\
       {\prompt} Custom CUDA & 0.6 & 1.2 & 2.3 & 4.5 & 9.0 \\
       {\arch} Custom CUDA   &  0.4   &  0.8  &  1.5 &  2.8  & 5.6  \\ 
    \bottomrule \\
    \end{tabular}

    \centering
    \label{tab:gpu_inference_jrt}
    \begin{tabular}{@{}rcccccc@{}}
    \toprule
    \textbf{Implementation}  & \textbf{2}  & \textbf{4}  & \textbf{8}  & \textbf{16} &  \textbf{32}  & \textbf{64}  \\ \midrule
    Based PyTorch        & 140.9 &  281.5  &  OOM &   OOM  &  OOM  &  OOM \\
    Based Triton (FLA)   &  4.6 &  8.7  & 16.7  &   32.4  &  64.2 & 127.8   \\
    Based Custom CUDA    &  1.2 &  1.3 & 1.5 & 2.3 & 4.5 & 8.9 \\
    \hline
    FlashAttention-2     & 3.5  &  6.7  & 13.4  &  26.6   &  52.9 & 108.2   \\
    Fast Transformer CUDA  & 17.1 &  26.7  &  50.7 &  95.5   &  OOM & OOM   \\
      \midrule
       {\arch} PyTorch   & 169.6 &  340.3  &  OOM &  OOM   &    OOM & OOM \\
       {\prompt} Custom CUDA & 2.3 &  2.5  & 2.9  &  4.5   &  9.0 & 17.8   \\
       {\arch} Custom CUDA   & 1.5  &   1.5 &  1.8 &   2.8  &  5.6 & 11.1   \\
    \bottomrule \\
    \end{tabular}

\end{table*}

Using the Based CUDA kernel proposed in \cite{arora2024simple}, 
{\prompt} gives $11.9\times$ and $13.7\times$ higher throughput in processing the prompt prefill than the FlashAttention-2 and FLA Triton kernels respectively (prefill length $32768$) (\Cref{tab:gpu_inference_jrt_seq}).
{\prompt} provides $6.1\times$ and $7.2\times$ higher throughput than the FlashAttention-2 and FLA kernels respectively as we increase the batch size to $64$ 
(\Cref{tab:gpu_inference_jrt}).
For {\prompt}, we \textit{double the prefill length} compared to the baselines, using $2\times$ the time of the original Based prefill. 

We next extend the Based kernel to support {\arch}  and demonstrate that the implementation achieves $19.2\times$ and $22.0\times$ higher throughput than FA2 and FLA as we increase sequence length to $32768$ (\Cref{tab:gpu_inference_jrt_seq}).
{\arch} provides $9.7\times$ and $11.5\times$ higher throughput respectively as we increase the batch size to $64$ (\Cref{tab:gpu_inference_jrt}). {\arch} takes $1.24\times$ the time of the Based prefill, improving efficiency over {\prompt}. 

We benchmark the inference efficiency of {\prompt} and {\arch} in \Cref{tab:gpu_inference_jrt_seq} (additional details in  \Cref{app:architecture}). As baselines, we consider popular and well-optimized softmax attention and linear attention implementation. For attention, we consider FlashAttention-2 \citep{dao2023flashattention2}. For linear attention, we consider the linear attention CUDA kernel from Fast Transformers \citep{katharopoulos-et-al-2020, vyas_et_al_2020} 
and a Triton parallel Based kernel from Flash Linear Attention (FLA) \citep{yang2024fla}.
We also compare to PyTorch implementations of JRT-RNN and Based. All numbers are benchmarked on a NVidia H100 GPU.

\section{Conclusion}

Recurrent LLMs promise drastically more efficient inference relative to Transformers, however they are brittle during in-context learning. 
We identify the role of data order as a key reason,
formalized via synthetics and theory.
Our analysis suggest that putting data in the right order in context or non-causally processing the context can help efficient recurrent models better use their limited memory. We translate these insights to {\prompt} and {\arch} respectively. {\prompt} improves the quality of recurrent models by $11.0 \pm 1.3$ points averaged across models and tasks, and our prototype architecture, {\arch}, provides a $13.7$ point improvement at $360\mathrm{M}$ parameters and  $6.9$ point improvement at $1.3\mathrm{B}$ parameters. Both methods increase throughput  relative to FlashAttention-2 using IO-aware CUDA implementations. 

While much of the effort on sub-quadratic LMs seeks to directly mimic the experience of using quadratic Transformer LMs, our work emphasizes that we can exploit the asymmetries in efficiency to close the quality gaps: \textit{multiple} linear passes over data is still asymptotically more efficient than quadratic attention. To facilitate reproducing this work, we release code and models at \url{https://github.com/HazyResearch/prefix-linear-attention}.

\section*{Acknowledgments}
We thank Michael Zhang, Michael Poli, Daniel Fu, Kawin Ethayarajh, John Thickstun, and Neel Guha for their helpful feedback and discussion during this work. We thank the Hazy Research lab and Together AI for supporting this work. 
We gratefully acknowledge the support of NIH under No. U54EB020405 (Mobilize), NSF under Nos. CCF2247015 (Hardware-Aware), CCF1763315 (Beyond Sparsity), CCF1563078 (Volume to Velocity), and 1937301 (RTML); US DEVCOM ARL under Nos. W911NF-23-2-0184 (Long-context) and W911NF-21-2-0251 (Interactive Human-AI Teaming); ONR under Nos. N000142312633 (Deep Signal Processing), N000141712266 (Unifying Weak Supervision), N000142012480 (Non-Euclidean Geometry), and N000142012275 (NEPTUNE); Stanford HAI under No. 247183; NXP, Xilinx, LETI-CEA, Intel, IBM, Microsoft, NEC, Toshiba, TSMC, ARM, Hitachi, BASF, Accenture, Ericsson, Qualcomm, Analog Devices, Google Cloud, Salesforce, Total, the HAI-GCP Cloud Credits for Research program,  the Stanford Data Science Initiative (SDSI), and members of the Stanford DAWN project: Facebook, Google, and VMWare. The U.S. Government is authorized to reproduce and distribute reprints for Governmental purposes notwithstanding any copyright notation thereon. Any opinions, findings, and conclusions or recommendations expressed in this material are those of the authors and do not necessarily reflect the views, policies, or endorsements, either expressed or implied, of NIH, ONR, or the U.S. Government. AR's research is supported by NSF grant CCF\#2247014.
\clearpage

\bibliographystyle{unsrtnat}
\bibliography{references}



\newpage
\appendix
\onecolumn
\clearpage
The appendix is organized as follows:
\begin{enumerate}
    \item \Cref{app:related-work} includes an extended related works discussion. 
    \item \Cref{app:expt-details} includes additional experimental details.
    \item \Cref{app:additional_experiments} includes additional experiments to supplement \Cref{sec5-results}.
    \item \Cref{app:architecture} includes details on the IO-aware implementation and benchmarking for {\arch}.
    \item \Cref{app:error_analysis} includes error analysis discussion for {\prompt}.
    \item \Cref{app:prompts} includes the prompts used for all in-context learning experiments in this work. 
    \item \Cref{app:theory} includes theoretical results and proofs.
\end{enumerate}

\clearpage

\section{Extended related work discussion}
\label{app:related-work}

The notion that causal models are limited because they need to ``predict the future'' when computing representations is well-known \citep{2020t5, schuster1997bidirectional, kosko1988bidirectional}. Yet, current large language models (e.g., Llama \cite{touvron2023llama}, GPT \cite{brown2020language}, and efficient Mamba \citep{gu2023mamba}, Griffin \citep{de2024griffin}, GLA \citep{yang2023gated}, RWKV \citep{peng2023rwkv}, Striped Hyena \citep{stripedhyena}) are causal. Here we provide an extended discussion of the related work.

\subsection{Prompting strategies} Most related to our work, \citet{springer2024repetition} recently proposes to produce embeddings from autoregressive Transformer models by repeating the context twice and taking embeddings from the activations of second occurrence. 
We focus on 1) sub-quadratic models / memory perspective, 2) recall-intensive tasks rather than producing embeddings.
Our findings build on these ideas and the key distinctions are: (1) our focus on \textit{sub-quadratic architectures}, which can provide asymptotically higher efficiency, (2) our focus on recall and in-context learning based tasks as opposed to embedding generation, and (3) our theoretical analysis on why {\prompt} impacts the memory requirement of recurrent LMs.

We are certainly not the first to try modifying the data order for recurrent LMs. The seminal Seq2seq paper from Sutskever et al. \cite{sutskever2014sequence} proposes to \textit{reverse} the order of the tokens in the source sequence when using encoder-decoder LSTM-based recurrent language models.

\subsection{Encoder-decoder language models} A long line of work has explored the use of bidirectional networks \citep{schuster1997bidirectional, kosko1988bidirectional, graves2005framewise,
devlin2020transformers, 
2020t5, patel2023bidirectional}. In early work,  \citet{schuster1997bidirectional} demonstrate synthetic math tasks that require recurrent models to use lagging and future values to produce outputs, favoring bidirectional networks. \citet{kosko1988bidirectional} explores associative recall style tasks in two layer bidirectional networks. We build on the ideas from this line of work and focus on our discussion on large language modeling architectures. 

Three popular language modeling architecture paradigms are encoder-only, decoder-only, or encoder-decoder. A popular use case for bidirectional, encoder-only, models is producing word or context embeddings \citep{peters2018deep, devlin2020transformers}. It is challenging to use these models for fast and open-ended generation \citep{tay2023ul2, dong2019unified}. Encoder-decoder models have emerged as a compelling alternative, combining non-causal bidirectional encoding for parts of the input text and causal decoding to generate responses. 

However, causal decoder-only language models currently prevail (e.g., Llama-3 \cite{llama3modelcard}, GPT \cite{ouyang2022instruct, brown2020language}, PaLM \cite{chowdhery2022palm}). Current research on efficient architectures also largely focuses on pure encoder-only (e.g. M2-BERT \cite{fu2023monarch}, Mamba-Caduceus \cite{schiff2024caduceus}, Orchid \cite{karami2024orchid}) or decoder-only causal LMs (e.g., Mamba \cite{gu2023mamba}, RWKV \cite{peng2023rwkv}, Griffin \cite{de2024griffin}, Striped Hyena \cite{stripedhyena}), as opposed to encoder-decoder. In contrast, our work on {\arch} explores encoder-decoder recurrent LMs in light of recent progress in sub-quadratic efficient architectures. \newline

\paragraph{Recurrent encoder-decoder language models}

Recurrent encoder-decoder language models were popular in the context of machine translation systems.  \citet{sutskever2014sequence} uses two LSTM RNNs, one to process the inputs and produce a fixed dimensional vector, and the other to decode the outputs from this vector.
\citet{wu2016googles} use a similar two-stack (encoder-stack and decoder-stack) architecture, using right-to-left and left-to-right RNNs for some encoder layers).

Instead of compressing the source sentence into a fixed recurrent state,  \citet{bahdanau2016neural} use \textit{attention} to refer back to encoder states. A key motivating observation for the switch to attention comes from \citet{cho2014properties}, which finds that the quality of RNN-based encoder-decoder language models degrades quickly as the sequence length increases. Following the rise of attention and the Transformer architecture \citep{vaswani2018attention} in popularity, subsequent work predominantly explores Transformer-based encoder-decoder LMs.

\paragraph{Transformer-based encoder-decoder language models} \citet{2020t5} propose the T5 architecture, which uses two separate Transformer stacks, one for non-causally encoding input text and one for causally decoding response. Cross-attention allows the decoder attention queries to attend to the final attention key and value states form the encoder stack. 
More recently, \cite{yen2024longcontext} trains a 7Bn parameter two-stack encoder-decoder model called CEPE, adapted form Llama-2 \citep{touvron2023llama} with cross-attention between stacks, following T5.\footnote{\url{https://huggingface.co/hyen/CEPED-LLaMA-2-Chat-7B}} We evaluate this model on the recall-intensive tasks and surprisingly find that ignoring its encoder altogether and placing documents and questions in the decoder far outperforms placing the document in the encoder and questions in the decoder on the recall-intensive benchmarks.
\begin{table}[h]
\centering
\begin{tabular}{l|cc}
\toprule
   & SWDE & FDA  \\
 &
  Acc. $\uparrow$ &
  Acc. $\uparrow$ \\ 
  \hline
CEPE Enc.-Dec.   & 51.0 &  5.9 \\ 
CEPE Dec.-Only       & 80.4 & 72.5 \\ 
\hline
\end{tabular}
\caption{Evaluating the CEPE 7Bn parameter model \citep{yen2024longcontext} on the document information extraction tasks, using $N=50$ random examples. For the encoder-decoder baseline, the document is inputted to the encoder and the question (i.e., name of the attribute to extract from the document) is sent to the decoder. In the decoder-only model, the standard prompt containing the document plus attribute are inputted to the decoder and the model's encoders are ignored (empty inputs). We observe the encoder-decoder model tends to produce irrelevant responses.}
\label{table:cepe}
\vspace{-3mm}
\end{table}

Prior work suggests that the T5 architecture struggles in open-ended generation \citep{patel2023bidirectional, tay2023ul2}. Some differences between {\arch} and the T5-style approach are that the T5 corruption pretraining objective deviates from how the models are used for downstream generation tasks, and training requires the use of multiple special sentinel tokens and unique positional encodings per stack of layers.

Instead of using separate encoder and decoder stacks, some prior work explores the use of Prefix-LMs. These models split the input into encoder and decoder regions \textit{within} each layer, where the former is processed non-causally and the latter is processed causally \citep{2020t5}. Next token prediction loss is computed on the causal tokens and no loss is computed on the prefix tokens. 

To better equip encoder-decoders with generation abilities,  UniLM \citep{dong2019unified}, UL2 \citep{tay2023ul2}, AlexaTM \citep{soltan2022alexatm} and others use different combinations of span corruption and prefix language modeling pretraining objectives. During training, given an input sequence, one of the suite of objectives is sampled with some pre-defined probability. Each of these architectures are Transformer-based, facing quadratic scaling in sequence length during training and linear scaling during inference. In GLM \citep{du-etal-2022-glm}, spans of text are masked and autoregressively in-filled during training, to endow the model with generation capabilities. We are inspired by these works in combining MLM and next token prediction objectives, and future work could explore alternate variations to the training objective used in {\arch}.

\paragraph{Discussing the differences in {\arch}} Recent work has made exciting progress in designing efficient LMs that extend the Pareto-frontier of the quality-efficiency tradeoff space relative to Transformers and prior recurrent architectures. However, these are decoder-only LMs, while {\arch} uses the encoder-decoder framework. Prior popular encoder-decoder LMs are Transformer-based with quadratic scaling and do not convincingly improve in quality over decoder-only models \citep{wang2022language}, so the motivation to use them is unclear. {\arch} improves efficiency (\Cref{tab:gpu_inference_jrt}) and quality (\Cref{table:jrt-rnn-quality}).

Within the encoder-decoder framework, {\arch} uses a prefix LM structure. Unfortunately, prior work and our ablations suggest this training strategy does not perform well (\cite{wang2022language} and \Cref{table:ablations}), and this architecture has not seen adoption. Instead {\arch} deviates by (1) adding a masked language modeling loss to the prefix alongside next token prediction for the suffix. {\arch} (2) \textit{reads the prefix twice}. Prefix LM models modify the attention mask of standard attention to make the prefix non-causal and use shared projection weights for the non-causal encoder and causal decoder regions. Instead, {\arch} uses two sets of key and value representations for encoding and decoding respectively.

\clearpage

\section{Experimental details}
\label{app:expt-details}

This section provides additional details for the synthetic, {\prompt} and {\arch} experimental protocols. We use NVidia A100-80GB GPUs for all training runs.

\subsection{Additional details for set disjointness synthetic experiments}
This section provides experimental details for \Cref{fig:synthetic}.

\paragraph{Dataset}

The procedure for generating training and evaluation data for our synthetic experiments is shown in \Cref{alg:synthetic}. We train on the following mixture of sequence lengths, where the tuple denotes $(|A|, |B|)$ for sets $A$ and $B$ in the sequence:
$$(4, 16), (16, 4), (8, 32), (32, 8), (64, 16), (16, 64), (4, 128), (128, 4), (16, 256), (256, 16), (4, 256), (256, 4)$$

We  evaluate on the following mixture of sequence lengths (requiring length extrapolation from training), where the tuple denotes $(|A|, |B|)$ for sets $A$ and $B$ in the sequence: $$(1, 32), (32, 1),
    (4, 32), (32, 4),
    (4, 128), (128, 4),
    (16, 256), (256, 16),
    (4, 256), (256, 4),
    (16, 512),$$ $$(512, 16), (4, 512), (512, 4), 
    (8, 768), (768, 8),
    (16, 768),(768, 16),
    (4, 768), (768, 4)$$
We include $20000$ data points per tuple above during training and $1000$ during evaluation. We use $V=2048$ as the vocabulary size.  

\begin{algorithm}[h!]
	\caption{Set Disjointness Synthetic Procedure}
	\begin{algorithmic}[1]
		\Require Vocabulary $V$, Sequence lengths $N_{A}$ and $N_{B}$ for sets $A$ and $B$, Special token IDs $\mathrm{prefix\_token\_id}$, $\mathrm{mask\_tok\_id}$, $\mathrm{sep\_sets\_token\_id}$, $\mathrm{sep\_answer\_tok\_id}$ \newline 
        \textbf{Output: Synthetic sequence}
		\State Let the first half of $V$, $V_A$, be prospective tokens for set $A$ and the second half, $V_B$, be prospective tokens for set $B$. 
  
        \State Randomly select $N_A$ tokens from $V_A$ for set $A$. Randomly select $N_B$ tokens from $V_B$ for set $B$.

        \State Randomly select a token $t$ from $A$ as the intersecting token between sets. Replace a random token (at a random position) from $B$ with $t$. 

        \State Construct the final input sequence as the concatenation: $$[\mathrm{prefix\_token\_id}], A, [\mathrm{sep\_sets\_token\_id}], B, \mathrm{sep\_answer\_tok\_id}], [t]$$

        \State The label sequence contains a ``-100'' (i.e., a token to ignore computing the loss) at all positions except for the final position. We mask $[t]$ (the final position) from the input sequence. 

        \State Output the synthetic input and label sequences.
	\end{algorithmic}
    \label{alg:synthetic}
\end{algorithm}

\paragraph{Models} We evaluate causal and non-causal variants of the Based recurrent model. Each model contains $4$ layers alternating gated-convolutions (with a short filter of size $3$) and linear attention with $2$ query key and value heads. For the non-causal variant, we simply replace the causal cumulative sum in linear attention with a sum, and we use non-causal circular convolutions. For the linear attention feature map, we use a Taylor approximation to the softmax-exponential function as in \cite{arora2024simple} (also defined in \Cref{app:based_def}). Each layer has an MLP with GeLU activations. We do not use any explicit positional embeddings, instead finding the short-convolutions sufficient for positional information. 

To sweep the state size, we vary the model width or dimension $\in \{36, 48, 64, 96, 128\}$ and linear attention feature dimension $\in \{4, 8, 16, 24\}$. 

\paragraph{Training} We train using cross-entropy loss on the predicted vs. true intersection token $t$ in \Cref{alg:synthetic}. For each point in \Cref{fig:synthetic}, we sweep learning rates $\in \{0.0001, 0.0005, 0.0008\}$ (after identifying that this regime is most effective for the architectures) and report the maximum accuracy after $48$ epochs of training. We use AdamW as the optimizer with $0.1$ weight decay.

We build our synthetic experiments using the synthetics repository provided by prior work \cite{arora2023zoology}: \url{https://github.com/HazyResearch/zoology}.

\subsection{Additional details for {\prompt} experiments}

For \Cref{table:main-quality} ({\prompt}), we use the following publicly available models pretrained and released by the baseline works:
\begin{itemize}
    \item Based \citep{arora2024simple} models are at \url{https://huggingface.co/collections/hazyresearch/based-65d77fb76f9c813c8b94339c}
    \item Gated Linear Attention \citep{yang2023gated} models are at \url{https://huggingface.co/fla-hub}.
    \item Mamba \citep{gu2023mamba} and Mamba-2 \citep{dao2024transformers} models are at \url{https://huggingface.co/state-spaces}
\end{itemize}

We integrate all tasks into the popular LM-Eval harness to run inference.
 We truncate long-documents (e.g., in NQ, FDA, SWDE) to length $1\mathrm{k}$ tokens for the default prompting and length $2\mathrm{k}$ tokens for {\prompt} so that both methods receive the same information in-context. We note that these lengths are chosen because the listed pretrained models have $2048$ context lengths. We ensure that the answer span is present in truncated documents. We do not use any task-specific prompt customization in this section, to highlight the effectiveness of {\prompt} despite little effort.

\subsection{Additional details for pre-training experiments}

\paragraph{Additional details for {\arch}} 
To facilitate comparisons to prior work, we start with the Based architecture \citep{arora2024simple} and replace its linear attention layers with {\arch} linear attention layers. Note that the Based architecture hybridizes gated convolution layers (kernel size $3$), sliding window attention layers (window size $128$), and linear attention layers (using a Taylor approximation to the exponential function as the feature map, with feature dimension $16$). We maintain the exact same order and number of each layer type as the Based work. We reduce the number of gated convolution layers by $1$ at $360\mathrm{M}$ parameters to account for the increase in parameters due to the encoder projections.

Next we include a description of the linear attention feature map used in our trained  models. Based uses a $2^{\mathrm{nd}}$-order Taylor approximation to the softmax-exponential function as the feature map $\phi: \mathbb{R}^d \rightarrow \mathbb{R}^{\Tilde{d}}$  \citep{zhang2024hedgehog}.
To approximate $\exp(\bm{q}_i^\top \bm{k}_j / \sqrt{d})$:
\begin{equation}
    \label{eq:taylor_feature_map}
    \exp(x) \approx 1 + x + \frac{x}{2!}
\end{equation}

\begin{equation}
    \label{eq:taylor_feature_map}
    \phi(\bm{q}_i)^\top \phi(\bm{k}_j) = 1 + \bm{q}_i^\top \bm{k}_j + \frac{(\bm{q}_i^\top \bm{k}_j)^2}{2}
\end{equation}

The second order term has large dimension $273$ if $\Tilde{d}=16$ as in \cite{arora2024simple}. As a result, a careful IO-aware implementation is key to efficiency.

\paragraph{Training protocol} 
For \Cref{table:jrt-rnn-quality}, we use the code provided by the baseline works, which has been adapted from the FlashAttention code base: \url{https://github.com/Dao-AILab/flash-attention/tree/main} for our pretraining runs~\cite{dao2023flashattention2}. The Pile data is tokenized using the GPT2BPETokenizer and all models see the data in the same order. 
Here we provide details on the hyperaparamters and configurations used for training each architecture. 

\begin{itemize}[itemsep=0.1pt,topsep=0pt,leftmargin=*]
    \item \textbf{{\arch}} We provide hyperparameters and settings used for {\arch} in \Cref{tab:jrt-training-details}. We integrate {\arch} into the Based implementation released by the prior work.
    \item \textbf{Based \citep{arora2024simple}} We train using the specifications in \Cref{tab:based-training-details} and the architecture implementation provided here: \url{https://github.com/HazyResearch/based}.
    
    \item \textbf{Transformer++ \citep{touvron2023llama}} We refer to the modern Llama architecture with Rotary encodings, RMSNorm and SwiGLU as Transformer++, following prior work \cite{gu2023mamba, yang2023gated}. We train using the the specifications in Table \ref{tab:attn-training-details} using the Flash Attention training code provided here: \url{https://github.com/Dao-AILab/flash-attention/tree/main}~\cite{dao2023flashattention2}. 

    \item \textbf{Mamba \citep{gu2023mamba} } We train using the specifications in \Cref{tab:mamba-training-details}, where the parameters are sourced from the Appendix of \cite{gu2023mamba}. The architecture implementation is from the reference at \url{https://github.com/state-spaces/mamba}.
     
\end{itemize}

\noindent We give all models the Transformer++ change (e.g., SwiGLU, Rotary) where relevant. 

\paragraph{Inference protocol} For {\arch}, we left-pad prefill when it is shorter than the encoder region and mask in the linear attention layer following Listing 3 \Cref{app:architecture}. We apply no changes if the prefill exceeds the encoder region. For all results reported in this work, we use the parallel view of {\arch} to process the prefill and compute initial states following \Cref{sec:sec4-jrt-arch}, then use the recurrent view to decode.

\subsection{Additional details for Pile perplexity slicing analysis}
In \Cref{sec:lm_ppl}, we analyze the perplexity of different models trained on the Pile, on the Pile test data. Here we provide additional details for the protocol.

We compute the training counts of bigrams across $10$M Pile training documents, each of length $2048$. We evaluate the models on $3,200$ sequences of length $2048$ ($6.6$M total tokens), and measure perplexity on the last $1024$ tokens per sequence (the causal, decoder region for {\arch}) ($3.3$M total tokens). We then evaluate perplexity on two slices of this test set:
\begin{enumerate}
    \item \textit{Associative recall (AR) hits.} Tokens in the final position of a bigram which previously occurred in context, and this bigram is infrequent during training. For instance, in the sequence ``While lunching at the Maison Bergey bistro near his apartment: he had been musing about the ... (723 tokens) ... the young waitress’s sigh at the Maison Bergey.'' the second ``Bergey'' would be included as an ``AR hit'' if ``Maison Bergey'' is a rare bigram during training. Intuitively, the model would need to rely on the context to predict the next token if the bigram were rare during training (i.e., was not memorized), testing the model's recall ability.
    \item \textit{Other tokens.}  All other tokens. Intuitively, these tokens test the knowledge memorized in the model parameters. 
\end{enumerate}

In \Cref{fig:ppl_slices}, for the \textbf{recall frequencies} plot, we restrict to ``AR hits'' where the bigram and the re-occurrence of the bigram in context are separated by at least $1024$
in distance within the context. In the \textbf{recall gaps} plot, we restrict to bigrams that are seen fewer than $1000$ times during training and vary the distance between bigram occurrences in-context on the $x$ axis.

\subsection{Evaluation datasets}

Here we provide additional details on the recall-intensive benchmark suite used in this work. The tasks include:
\begin{itemize}[itemsep=0.1pt,topsep=0pt,leftmargin=*]
    \item \textbf{FDA} FDA is an information extraction task where documents are FDA reports for pre-market medical devices and the model needs to extract attributes such as the device code, classification, and indications for use \citep{arora2023evaporate, arora2024simple}. These FDA reports are frequently analyzed by domain experts \cite{wu2021medai}. We use the dataset released at: \url{https://huggingface.co/datasets/hazyresearch/based-fda}, which is part of the LM-Eval Harness repository \cite{eval-harness}.   
    \item \textbf{SWDE} SWDE is an information extraction task where documents are HTML webpages spanning 14 different websites in the Movie and University topic domains (e.g., ``IMDB.com'', ``RottenTomatoes'', ``USNews'') and the model needs to extract attributes such as the Movie director / assistant director and University tuition \citep{lockard2020zeroshotceres, deng2022domlm, arora2023evaporate, arora2024simple}. We use the dataset released at:
    \url{https://huggingface.co/datasets/hazyresearch/based-swde}, which is part of the LM-Eval Harness repository \cite{eval-harness}.
    \item \textbf{SQUADv2} SQUADv2 is a document QA benchhmark where documents come from Wikipedia and answer to questions are a span of tokens in the document \citep{rajpurkar2018squad, arora2024simple}. We use the version of the dataset released at: \url{https://huggingface.co/datasets/hazyresearch/based-squad}, which is part of the LM-Eval Harness repository \cite{eval-harness}.
    \item \textbf{TriviaQA} TriviaQA is a popular document QA benchmark where documents come from both Wikipedia and the general web and the question structure varies \cite{joshi2017triviaqa}. We use the dataset released at: \url{https://huggingface.co/datasets/mandarjoshi/trivia_qa}
    \item \textbf{Natural Questions (NQ)} Natural Questions is a popular document QA benchmark where documents come from Wikipedia and the questions are real queries issued to the Google search engine \citep{kwiatkowski-etal-2019-natural}. The answers are spans of text from the documents. We use the dataset released at: \url{https://huggingface.co/datasets/natural_questions}.
    \item \textbf{Drop} DROP is a challenging document QA benchmark that requires discrete reasoning over paragraphs from Wikipedia articles \citep{Dua2019DROP}. The questions often require arithmetic operations, counting, or sorting of information found in the documents. We use the dataset released at: \url{https://huggingface.co/datasets/ucinlp/drop}.
\end{itemize}

\paragraph{Cloze Completion Formatting}
As the models in this work are not instruction fine-tuned and have been trained on next token prediction, they are more effective at producing relevant answers when the prompt format aligns with the pre-training task (next token prediction) as shown in prior work \citep{arora2022ama}. Therefore, we reformat the questions in these benchmarks to a cloze-completion format using Meta's Llama-3-70B model \citep{llama3modelcard}.

Given the question and the answer, the prompt we use is, where we provide the original question and answer from the task example:
\begin{mdframed}[frametitle={Converting to Cloze Format}]
\begin{codeframe}
\begin{lstlisting}
Can you rewrite this question and answer as a statement. Ensure that the answer is the last part of the statement. 


Question: {question}

Answer: {answers}

Rewrite:
\end{lstlisting}
\end{codeframe}
\end{mdframed}

As an example:
\begin{mdframed}[frametitle={Example}]
\begin{codeframe}
Input
\begin{lstlisting}
Can you rewrite this question and answer as a statement. Ensure that the answer is the last part of the statement. 


Question: Which team scored the final TD of the game?

Answer: Dallas

Rewrite:
\end{lstlisting}
\end{codeframe}
\begin{codeframe}
Answer
\begin{lstlisting}
The team that scored the final TD of the game is Dallas. 
\end{lstlisting}
\end{codeframe}
\end{mdframed}

We filter the dataset by picking the rewrite with the answer appearing in the end and we remove the answer (e.g., ``Dallas'') when producing the final dataset. We report the resulting dataset sizes in \Cref{table: datasets-overview} and release the datasets for reproducal.

\begin{table}[h]
\centering
\begin{tabular}{l|cc}
\hline
Dataset       & Size  &Token\\ \hline
FDA           & 1102  &1999.9\\
SWDE          & 1111  &1036.1\\
SQUAD         & 2984  &151.9\\
TriviaQA      & 1698  &310.1\\
NQ            & 3157  &8857.7\\
Drop          & 2084  &236.6\\
\hline
\end{tabular}
\caption{Evaluation Dataset Overview}
\label{table: datasets-overview}
\end{table}

\paragraph{Metrics} We evaluate whether the model generated answer contains the exact answer span specified in the task. We run inference using the newline character and max generation length of $48$ as stop-conditions.
 
\clearpage

\section{Additional experiments}
\label{app:additional_experiments}

\subsection{Overall language modeling}

While we focus on a suite of recall-intensive benchmarks in \Cref{sec5-results}, here we show that {\arch} maintains the quality of baseline models on other common in-context learning benchmarks. We use SuperGLUE \cite{wang2019superglue} suite.
We run these evaluations using the LM-Eval Harness repository's default settings \cite{eval-harness}.

In \Cref{tab:super-glue} and \Cref{tab:super-glue-1b}, we observe that all models achieve comparable quality. These results align with prior work suggesting that while alternate architectures provide similar overall language modeling perplexity, their quality on recall-intensive tasks is much more variable \cite{arora2023zoology, gu2023mamba, akyurek2024incontext,  arora2024simple}. 

\paragraph{Padding} We note that the SuperGLUE inputs are quite short in sequence length, meaning that {\arch} sees pad tokens in the majority of the encoder region of the input until we reach length $M=1024$. We use the space-token as the pad token in our evaluations, as discussed in \Cref{app:expt-details}. Since we do not train with pad tokens in this work, this such sequences are relatively out of distribution, but with masking the padding portion of the sequence, we can recover quality. In \Cref{tab:super-glue_no_padding}, we evaluate {\arch} where we do not mask on the linear attention layers and observe quality starkly degrades on certain tasks (e.g., Copa and WSC).

\begin{table*}[]
\centering
\scriptsize 
\begin{tabular}{lcccccccccccc}
\toprule
\textbf{Model}    & \multicolumn{1}{l}{\textbf{Shots}} & \multicolumn{1}{l}{\textbf{BoolQ}}  & \multicolumn{2}{c}{\textbf{CB}}          & \multicolumn{1}{l}{\textbf{COPA}}   & \multicolumn{1}{l}{\textbf{MultiRC}} & \multicolumn{2}{l}{\textbf{ReCoRD}}  & \multicolumn{1}{l}{\textbf{RTE}}    & \multicolumn{1}{l}{\textbf{WiC}}    & \multicolumn{1}{l}{\textbf{WSC}}    & \multicolumn{1}{l}{\textbf{Avg}} \\
& \multicolumn{1}{l}{}               & \multicolumn{1}{l}{Acc. $\uparrow$} & \multicolumn{1}{l}{Acc. $\uparrow$} & \multicolumn{1}{l}{F1 $\uparrow$} & \multicolumn{1}{l}{Acc. $\uparrow$} & \multicolumn{1}{l}{Acc. $\uparrow$}  & \multicolumn{1}{l}{F1 $\uparrow$} & \multicolumn{1}{l}{EM $\uparrow$} & \multicolumn{1}{l}{Acc. $\uparrow$} & \multicolumn{1}{l}{Acc. $\uparrow$} & \multicolumn{1}{l}{Acc. $\uparrow$} & \multicolumn{1}{l}{}                 \\ \hline\hline
\multirow{3}{*}{\begin{tabular}[c]{@{}l@{}}{\arch}\\ 
(356m/30b)\end{tabular}}     
   & 0   &  49.2 &  33.9 &  17.4 & 65.0 &  57.2 &  16.5 & 15.8 & 53.1  & 50.0  & 37.5  &  39.6 \\ 
   & 1   &  46.5 & 37.5  & 26.9  &  65.0 & 51.9 & 18.9 &  18.1  &  46.2 & 46.6  &  55.8 &   41.3\\ 
   & 5   &  49.1 &  44.6 &  30.5 & 71.0 & 56.3  &26.7  &   25.8 & 48.0  &  50.5 & 50.0  &   45.3\\ 
   \hline
\multirow{3}{*}{\begin{tabular}[c]{@{}l@{}}Based \\ 
(360m/30b)\end{tabular}}     
   & 0   &  57.6 &  32.1 &  21.7 & 65.0 & 57.2  & 17.4 &  17.0  & 54.5  & 50.0 & 36.5  &  40.9 \\ 
  
   & 1  & 54.9  &  35.7 &  25.7 & 70.0 &  55.3 &  21.8 & 21.1   & 48.0  & 48.1 & 55.8  &   43.6\\ 
   & 5  & 53.5  & 53.6  & 36.7  & 76.0 &  56.4 & 25.3  &  24.4  &  50.5 & 53.6 &   51.0 &   48.1\\ 
   \hline
\multirow{3}{*}{\begin{tabular}[c]{@{}l@{}}Transformer \\ 
(360m/30b)\end{tabular}}    
   & 0   & 59.3  &  41.1 &  24.1&   68.0 & 57.2  & 14.6 &  14.2  & 54.9  & 50.0  &  36.5 &   42.0\\ 
   & 1   &  54.9 & 37.5  & 26.9  & 70.0 &  54.2 &  21.1 &  20.4  & 43.7  & 46.4 &  53.8  &   42.9\\ 
   & 5   & 49.1  &  46.4 & 30.9  & 68.0 & 55.2  & 23.7 & 23.0   & 52.7  &  51.1 & 52.9  &   45.3\\ 
   \hline
\multirow{3}{*}{\begin{tabular}[c]{@{}l@{}}Mamba\\ (358m/30b)\end{tabular}}         
    & 0   & 56.4  & 35.7  &  25.8 & 68.0 &  57.2 & 27.2 &   26.6 &  53.4 & 50.0 &  36.5 &   43.7\\ 
    & 1   & 51.1  &  41.1 &  28.5 &  70.0 & 52.3  & 25.8 &  25.1  &  50.2 & 46.4 & 55.8  &  44.6 \\ 
    & 5   & 50.0  &  51.8 & 34.8  & 70.0 & 54.5 & 23.2 &  22.5  & 46.9  & 50.3 & 51.0  &  45.5 \\ 
\hline
\end{tabular}
\caption{\textbf{SuperGLUE benchmark evaluations.} We evaluate the models from \Cref{table:jrt-rnn-quality} on the SuperGLUE benchmark \citep{wang2019superglue} using the EleutherAI LM Eval harness \citep{eval-harness}. 
}
\label{tab:super-glue}
\end{table*}

\begin{table*}[]
\centering
\scriptsize 
\begin{tabular}{lccccccccccc}
\toprule
\textbf{Model}    & \multicolumn{1}{l}{\textbf{Shots}} & \multicolumn{1}{l}{\textbf{BoolQ}}  & \multicolumn{2}{c}{\textbf{CB}}          & \multicolumn{1}{l}{\textbf{COPA}}   & \multicolumn{1}{l}{\textbf{MultiRC}} & 
{\textbf{RTE}}    & \multicolumn{1}{l}{\textbf{WiC}}    & \multicolumn{1}{l}{\textbf{WSC}}    & \multicolumn{1}{l}{\textbf{Avg}} \\
& \multicolumn{1}{l}{}               & \multicolumn{1}{l}{Acc. $\uparrow$} & \multicolumn{1}{l}{Acc. $\uparrow$} & \multicolumn{1}{l}{F1 $\uparrow$} & \multicolumn{1}{l}{Acc. $\uparrow$} & \multicolumn{1}{l}{Acc. $\uparrow$}  & 
\multicolumn{1}{l}{Acc. $\uparrow$} & \multicolumn{1}{l}{Acc. $\uparrow$} & \multicolumn{1}{l}{Acc. $\uparrow$} & \multicolumn{1}{l}{}                 \\ \hline\hline
\multirow{2}{*}{\begin{tabular}[c]{@{}l@{}}{\arch}\\ (1.3B/50B)\end{tabular}}      & 0   &  57.4 & 33.9  & 22.4  & 74.0 &  57.2 & 52.7 & 50.0 & 36.5  & 50.9  \\ 
& 5   & 52.1  &  50.0 &  34.5 & 75.0 & 53.9  &  49.8&  50.0  & 55.8  & 54.1  \\ 
   \hline
\multirow{2}{*}{\begin{tabular}[c]{@{}l@{}}Based \\ (1.3B/50B)\end{tabular}}     & 0   &  55.1 & 41.1  & 19.4  & 71.0 & 56.8  & 53.1 &  50.0  & 53.8  & 52.9  \\ 
& 5   & 52.5  & 50.0  &  33.7 &  75.0 & 51.4  &  49.1&  53.1  & 53.8  & 53.8 \\ 
   \hline
\multirow{2}{*}{\begin{tabular}[c]{@{}l@{}}Transformer \\ (1.3B/50B)\end{tabular}}   
& 0   & 57.6  &  41.1 &  28.8 & 72.0 & 56.0 & 54.2  & 50.0 &  53.8  &  54.1  \\ 
& 5   &  54.8 & 41.1  & 26.2  & 73.0 & 51.7  & 57.4 &  50.3  & 47.1  & 52.9 \\ 
   \hline
\multirow{2}{*}{\begin{tabular}[c]{@{}l@{}}Mamba\\ (1.3B/50B)\end{tabular}}  & 0   & 54.8  &  25.0 & 25.2  & 73.0 &  56.4 &  51.3&  50.0  & 40.4  & 50.1  \\ 
& 5   & 55.6  & 53.6  &  45.5 & 75.0 & 53.7  & 53.8&  51.7  & 56.7  & 56.6  \\ 
\hline
\end{tabular}
\caption{Same as \Cref{tab:super-glue} at the 1.3b parameter scale, trained on 50b tokens.}
\label{tab:super-glue-1b}
\end{table*}

\begin{table*}[]
\centering
\scriptsize 
\begin{tabular}{lcccccccccccc}
\toprule
\textbf{Model}    & \multicolumn{1}{l}{\textbf{Shots}} & \multicolumn{1}{l}{\textbf{BoolQ}}  & \multicolumn{2}{c}{\textbf{CB}}          & \multicolumn{1}{l}{\textbf{COPA}}   & \multicolumn{1}{l}{\textbf{MultiRC}} & \multicolumn{2}{l}
{\textbf{ReCoRD}}  & \multicolumn{1}{l}
{\textbf{RTE}}    & \multicolumn{1}{l}{\textbf{WiC}}    & \multicolumn{1}{l}{\textbf{WSC}}    & \multicolumn{1}{l}{\textbf{Avg}} \\
& \multicolumn{1}{l}{}               & \multicolumn{1}{l}{Acc. $\uparrow$} & \multicolumn{1}{l}{Acc. $\uparrow$} & \multicolumn{1}{l}{F1 $\uparrow$} & \multicolumn{1}{l}{Acc. $\uparrow$} & \multicolumn{1}{l}{Acc. $\uparrow$}  & \multicolumn{1}{l}{F1 $\uparrow$} & \multicolumn{1}{l}{EM $\uparrow$} & \multicolumn{1}{l}{Acc. $\uparrow$} & \multicolumn{1}{l}{Acc. $\uparrow$} & \multicolumn{1}{l}{Acc. $\uparrow$} & \multicolumn{1}{l}{}                 \\ \hline\hline
{\arch}       & 5  & 53.5  & 53.6  & 36.7  & 76.0 &  56.4 & 25.3  &  24.4  &  50.5 & 53.6 &   51.0 & 44.2  \\ 
+No Pad Mask &  5    &  49.1 &  55.4 & 38.2  & 56.0 & 56.3  &  26.7 &   25.8 &  51.6 &  49.7 & 40.4 & 41.3 \\ 
\hline
\end{tabular}
\caption{\textbf{Few-shot downstream evaluation on SuperGLUE of pre-trained language models.} Same protocol as \Cref{tab:super-glue}, however \textbf{we do not mask the left-padding} in the linear attention layers.}
\label{tab:super-glue_no_padding}
\end{table*}

\subsection{{\arch} ablations}

\paragraph{Training without MLM Loss} {\arch} inspired by Prefix LM due to its simplicity. Prior work and our own finds that Prefix LM underperforms in quality \cite{wang2022language}. 
Here we compare {\arch} with and without the masked language modeling (MLM) loss. Excluding the MLM loss matches the protocol in prior Prefix-LM training. 
In Table \ref{table:ablations}, we find that the model is decent at longer sequences, but drops quality on short-context prompts.

\begin{table*}[h!]
\centering
\scriptsize
\begin{tabular}{l|cc|cc|cc}
\toprule
\multicolumn{1}{c}{} &
  \multicolumn{2}{c}{N=512} &
  \multicolumn{2}{c}{N=1024} &
  \multicolumn{2}{c}{N=2048} \\
  \multicolumn{1}{c}{\multirow{2}{*}} &
  SWDE &
  FDA &
  SWDE &
  FDA &
  SWDE &
  FDA 
   \\
\multicolumn{1}{c}{} &
  Acc. $\uparrow$ &
  Acc. $\uparrow$ &
  Acc. $\uparrow$ &
  Acc. $\uparrow$ &
  Acc. $\uparrow$ &
  Acc. $\uparrow$ \\ 
  \hline

Based   &  25.4 & 51.0 & 19.1 & 30.1 & 15.7 & 13.4  \\ 
{\arch}, no MLM loss  
& 23.9 & 38.7 & 21.6 & 39.2  & 18.5 & 18.3  \\ %
\hline
\end{tabular}
\caption{\textbf{Ablations of design choices in {\arch}} All models are 360M param variants of {\arch}, trained to 10 billion tokens on the Pile.}

\label{table:ablations}
\vspace{-3mm}
\end{table*}

\paragraph{Training with Based ablations} Based is a hybrid architecture with some linear attention, sliding window attention, and gated short-convolution layers. In \Cref{table:ablations_based}, we train with the {\arch} vs. decoder-only approaches while ablating the mixture of layer types. The results suggest prefix linear attention remains useful for these recall-intensive tasks. 

\begin{table*}[h!]
\centering
\scriptsize
\begin{tabular}{l|cc|cc|cc}
\toprule
\multicolumn{1}{c}{} &
  \multicolumn{2}{c}{N=512} &
  \multicolumn{2}{c}{N=1024} &
  \multicolumn{2}{c}{N=2048} \\
  \multicolumn{1}{c}{\multirow{2}{*}} &
  SWDE &
  FDA &
  SWDE &
  FDA &
  SWDE &
  FDA 
   \\
\multicolumn{1}{c}{} &
  Acc. $\uparrow$ &
  Acc. $\uparrow$ &
  Acc. $\uparrow$ &
  Acc. $\uparrow$ &
  Acc. $\uparrow$ &
  Acc. $\uparrow$ \\ 
  \hline
\hline 
Linear attention (Taylor map)  &  29.6 & 25.5 & 21.5 & 16.0 & 23.0 & 4.6  \\ 
Prefix linear attention (Taylor map) & 36.8 & 57.7 & 27.1 & 48.7 & 23.9 & 8.2  \\ 
\hline
Linear + Sliding attention  & 25.4 & 10.3 & 21.2 & 8.1 & 20.8 & 3.0 \\ 
Prefix Linear + Sliding attention & 35.5 & 53.3 & 34.8 & 46.5 & 32.1 & 30.0 \\ 
\bottomrule
\end{tabular}
\caption{\textbf{Ablations of the types of sequence mixers in the LMs.} The default Based and {\arch} architectures in the main paper use a hybrid of sliding window attention (SWA), gated convolutions, and linear attention (LA). Here we also evaluate pure linear attention variations (top two rows, no SWA, no Convs.) and linear attention plus SWA (bottom two rows, no Convs.). All models are 360M param variants of {\arch}, trained to 30 billion tokens on the Pile using the same learning rates and schedules. In \cite{arora2024simple}, it is also observed that the short convolution layers are helpful for such tasks.}

\label{table:ablations_based}
\vspace{-3mm}
\end{table*}

\clearpage

\section{{\arch} implementation details}
\label{app:architecture}

In this section, we first provide a PyTorch reference for {\arch} and then discuss the IO-aware CUDA implementation.

\subsection{Reference code for {\arch}}
Below we include a PyTorch reference for the proposed layer, showing the parallel and recurrent views. 

\begin{lstlisting}[language=Python,style=mystyle,caption={Minimal PyTorch implementation of JRT RNN.}]
from einops import rearrange
import torch
from torch import nn


def encoder(k, v):
    k, v = k.unsqueeze(-2), v.unsqueeze(-1)
    kv_state = (k * v).sum(dim=2, keepdim=True)
    k_state = k.sum(dim=2, keepdim=True)
    return kv_state, k_state

def decoder(q, k, v):
    q, k, v = q.unsqueeze(-2), k.unsqueeze(-2), v.unsqueeze(-1)
    kv_state_dec = (k * v).cumsum(dim=2)
    k_state_dec = k.cumsum(dim=2)
    return q, kv_state_dec, k_state_dec

def compute_linear_output(q_dec, k_dec, v_dec, k_enc, v_enc):
    kv_state_enc, k_state_enc = encoder(k_enc, v_enc)
    q, kv_state_dec, k_state_dec = decoder(q_dec, k_dec, v_dec)

    kv_state_dec = kv_state_enc + kv_state_dec
    k_state_dec = k_state_enc + k_state_dec

    z = 1 / ( q * k_state_dec).sum(dim=-1) 
    y = ( (q * kv_state_dec).sum(dim=-1))
    output = y * z 
    output = rearrange(output, 'b h l d -> b l (h d)')
    return output

def compute_parallel_output(q_dec, k_dec, v_dec, k_enc, v_enc):

    # Scaling
    k_state = k_enc.sum(dim=2, keepdim=True) +  k_dec.cumsum(2)
    z = 1 / ((q_dec * k_state).sum(dim=-1))

    # standard attention
    A_qk = torch.einsum("bhnd,bhmd->bhnm", q_dec, k_dec) 
    A_qk = torch.tril(A_qk)
    y = torch.einsum("bhnm,bhme->bhne", A_qk.to(q_dec.dtype), v_dec.to(q_dec.dtype))
    y = y * z[..., None]
    output_1 = rearrange(y, 'b h l d -> b l (h d)')

    # cross attention
    A_qk_2 = torch.einsum("bhnd,bhmd->bhnm", q_dec, k_enc)
    y = torch.einsum("bhnm,bhme->bhne", A_qk_2.to(q_dec.dtype), v_enc.to(q_dec.dtype))
    y = y * z[..., None]
    output_2 = rearrange(y, 'b h l d -> b l (h d)')
    output_ref = output_1 + output_2
    return output_ref

# Inputs
enc_len, dec_len = seqlen // 2, seqlen
q_dec = torch.randn((batch, heads, dec_len, head_dim))
k_dec = torch.randn((batch, heads, dec_len, head_dim))
v_dec = torch.randn((batch, heads, dec_len, head_dim))
k_enc = torch.randn((batch, heads, enc_len, head_dim))
v_enc = torch.randn((batch, heads, enc_len, head_dim))

q_dec = feature_map(q_enc) # head_dim to expanded_dim
k_enc = feature_map(k_enc)
k_dec = feature_map(k_dec)

out = compute_linear_output(q_dec, k_dec, v_dec, k_enc, v_enc)
out_ref = compute_parallel_output(q_dec, k_dec, v_dec, k_enc, v_enc)
\end{lstlisting}

\label{masking}
\begin{lstlisting}[language=Python,style=mystyle,caption={PyTorch implementation linear attention masking}]
if mask is not None and q.shape[2] > 1: # Check that we're in prefill
    if len(mask.shape) == 4:
        lin_attn_mask = (mask == 0)[:, :1, -1, :][..., None]  # b,1,k_len,1
    else:
        lin_attn_mask = mask[:, None, :, None]  # b,1,k_len,1
    lin_attn_mask = lin_attn_mask.to(torch.bool)
    k = k.masked_fill(~lin_attn_mask, 0)
    k_enc = k_enc.masked_fill(~lin_attn_mask, 0)
\end{lstlisting}


\subsection{IO-aware implementation}

We build our implementation from the custom kernel for the Based architecture released in prior work \cite{arora2024simple} (Algorithm 1). \footnote{\url{https://github.com/HazyResearch/ThunderKittens}} Letting $\mathrm{fn_{based}}$ be the prior kernel, we use \Cref{alg:jrt_rnn} as the IO-aware implementation of {\arch}. We modify $\mathrm{fn_{based}}$ to (1) avoid multiplications with queries in the first call and to simply compute the KV-state, and (2) we use the final row (row $M$) of the KV-state, representing the sum of $(\mathrm{k_e} * \mathrm{v_e})$ along the sequence dimension. 

\begin{algorithm*}
    
  \caption{\label{alg:jrt_rnn} {\arch} CUDA Kernel Pseudocode}
  \small
  \begin{algorithmic}
    \Require{Input decoder representations $q_d, k_d, v_d \in \mathbb{R}^{N \times d}$ and encoder representations $k_e, v_e \in \mathbb{R}^{M \times d}$.}

    \Ensure{Output $y \in \mathbb{R}^{N \times d}$}

    \Statex Initialize SRAM buffers and register file fragments following Algorithm 1 \cite{arora2024simple}. Including registers $A0, A1, A2$ to store the KV-state (for the $0^{th}, 1^{st}, 2^{nd}$ order terms of the Based linear attention kernel Taylor approximation respectively) and SRAM buffer $y$ for storing the final output \newline
    
    \Statex Run $\mathrm{fn_{based}}(\mathrm{k_e, v_e})$ to compute KV-state for the encoder, where the result is held in registers $A0, A1, A2$. We modify the previously proposed Based implementation by using the non-causal sum instead of cumsum for the KV states. We don't multiply with queries in this step, as is done in the original algorithm. \newline 
    
    \Statex Run $\mathrm{fn_{based}}(\mathrm{q_d, k_d, v_d})$, from the register state initialized by the encoder computation. This computes the output $y$, held in SRAM. \newline 

    \Statex Store $y$ from SRAM to HBM. 
    
  \end{algorithmic}
\end{algorithm*}

\clearpage

\section{Analysis}
\label{app:error_analysis}

In this section, we provide qualitative analysis for {\prompt} using three representative recurrent LMs, Mamba pretrained for $300$b tokens on the Pile at the $370$M, $1.4$B, and $2.8$B parameter scales. 

We first bucket the common error modes, finding three primary categories: (1) No Answer (N/A), (2) Repetition, and (3) Irrelevant outputs. The statistics for each category are shown in \Cref{table:error-analysis}. Compared to the standard default zero-shot prompting approach, {\prompt} tends to increase the No Answer error and repetition errors, while reducing errors related to irrelevant outputs.
\begin{table*}[!ht]
\centering
\begin{tabular}{l|ccc|ccc|ccc}
\hline
Model & \multicolumn{3}{c|}{Mamba-370m} & \multicolumn{3}{c|}{Mamba-1.4B} & \multicolumn{3}{c}{Mamba-2.8B}\\
Error Type & N/A & Rep & Irrel & N/A & Rep & Irrel & N/A & Rep & Irrel \\ \hline
\hline
FDA-default & 0.2 & 35.4 & 22.7 &0.1 & 31.1&23.0 & 0.2 &27.5 &18.3  \\
FDA-{\prompt} & 0.0 &29.4 & \textbf{12.3} & 0.1 &29.2 &\textbf{9.8} &0.0 &23.3 & \textbf{9.8}  \\
\hline
SWDE-default & 39.1 &20.2 & 13.1 &37.3 &17.3 &7.8 &32.3 &18.9 & 9.7  \\
SWDE-{\prompt} & 23.6 & 17.0 &17.2 &28.0 & 15.0 & 11.1 &26.9 & 14.7 &9.6  \\
\hline
SQUAD-default & 0.0 & 6.6 &58.6 & 0.0 &5.9 &54.2 &0.0 &5.5 &51.3  \\
SQUAD-{\prompt} & 0.0 & 12.2 &\textbf{37.0} & 0.1 & 10.7 &\textbf{30.0} & 1.6 &32.9 &\textbf{13.8} \\
\hline
\end{tabular}
\caption{\textbf{Error Mode Statistics} We calculate the percentage ratio of different error types to the total number of test data points. N/A: No Answer; Rep: Repetition; Irrel: Irrelevant.}
\label{table:error-analysis}
\vspace{-3mm}
\end{table*}

\paragraph{No Answer} 
One error observed in the models is the output of an empty string, especially in tasks with complex text. We believe this is due to formatting sensitivity and could reduce  with model scale.

\begin{mdframed}[frametitle={No Answer Example}]
Input
\begin{codeframe}
\begin{lstlisting}
Information about the applicant in the text: SUBSTANTIAL EQUIVALENCE DETERMINATION DECISION SUMMARY A. 510(k) Number: K172333 B. Purpose for Submission: To expand the use of previously cleared assay reagents for Factor V Leiden; ...... D. Type of Test: Quantitative clot-based applications E. Applicant: Siemens Healthcare Diagnostics Product GmbH F. Proprietary and Established Names: ...... G. Regulatory Information: ...... Protein C with Protein C Reagent Antithrombin (AT) with INNOVANCE Antithrombin Protein C with Berichrom Protein C \n
Information about the applicant in the text: SUBSTANTIAL EQUIVALENCE DETERMINATION DECISION SUMMARY A. 510(k) Number: K172333 B. Purpose for Submission: To expand the use of previously cleared assay reagents for Factor V Leiden; ...... D. Type of Test: Quantitative clot-based applications E. Applicant: Siemens Healthcare Diagnostics Product GmbH F. Proprietary and Established Names: ...... G. Regulatory Information: ...... Protein C with Protein C Reagent Antithrombin (AT) with INNOVANCE Antithrombin Protein C with Berichrom Protein C applicant:
\end{lstlisting}
\end{codeframe}
Prediction
\begin{codeframe}
\begin{lstlisting}
""
\end{lstlisting}
\end{codeframe}
Ground Truth
\begin{codeframe}
\begin{lstlisting}
Siemens Healthcare Diagnostics Product GmbH
\end{lstlisting}
\end{codeframe}
\end{mdframed}

\clearpage

\paragraph{Repetition}
If the model reads repeated phrases (e.g., documents and questions), it may merely repeat the document and question again rather than providing an answer, when using {\prompt}. These models are not instruction fine-tuned and identifying the relevant task may be difficult.

\begin{mdframed}[frametitle={Repetition Error Example}]
Input
\begin{codeframe}
\begin{lstlisting}
Information about the applicant in the text: SUBSTANTIAL EQUIVALENCE DETERMINATION DECISION SUMMARY A. 510(k) Number: K182513 B. Purpose for Submission: New device 510(k) clearance for the FluChip-8G Influenza A+B Assay C. Measurand: Influenza A and influenza B viral nucleic acids. ...... E. Applicant: InDevR, Inc. F. ...

Information about the applicant in the text:) SUBSTANTIAL EQUIVALENCE DETERMINATION DECISION SUMMARY A. 510(k) Number: K182513 B. Purpose for Submission: New device 510(k) clearance for the FluChip-8G Influenza A+B Assay C. Measurand: Influenza A and influenza B viral nucleic acids. ...... E. Applicant: InDevR, Inc. F. ...
applicant:
\end{lstlisting}
\end{codeframe}
Prediction
\begin{codeframe}
\begin{lstlisting}
Information about the applicant in the text:) SUBSTANTIAL EQUIVALENCE DETERMINATION SUMMARY A. 510(k) Number: K182513 B. Purpose for Submission: New device 510(k) clearance
\end{lstlisting}
\end{codeframe}
Ground Truth
\begin{codeframe}
\begin{lstlisting}
InDevR, Inc.
\end{lstlisting}
\end{codeframe}
\end{mdframed}

\paragraph{Irrelevant Output}
Sometimes model outputs are undesirable and unrelated to the input text. For instance, the model may provide new \textit{continuations} of the text as opposed to referring back to the context and outputting previously seen information. {\prompt} appears to help reduce these types of errors.

\begin{mdframed}[frametitle={Irrelevant Output Example}]
Input
\begin{codeframe}
\begin{lstlisting}
 "Title: Martin_Luther\nBackground: At the heart of scholars' debate about Luther's influence is whether it is anachronistic to view his work as a precursor of the racial antisemitism of the Nazis...
 Title: Martin_Luther\nBackground: At the heart of scholars' debate about Luther's influence is whether it is anachronistic to view his work as ...... His position was entirely religious and in no respect racial.\"Martin Brecht referred to Luther's stand on the Jews as 
\end{lstlisting}
\end{codeframe}
Prediction
\begin{codeframe}
\begin{lstlisting}
a very important and important part of the history of the German people.
\end{lstlisting}
\end{codeframe}
Ground Truth
\begin{codeframe}
\begin{lstlisting}
misguided agitation
\end{lstlisting}
\end{codeframe}
\end{mdframed}

\paragraph{Few shot prompting}
A common hypothesis for why few-shot prompting is more effective than  zero-shot prompting is that it provides the model with a better understanding of the task at hand. Here we evaluate the few-shot baselines on recall-intensive tasks.

The in-context learning results for different models are shown in \Cref{table:incontext-result}. The improvement of few-shot in-context learning in smaller models is less obvious than in larger models. {\prompt} appears more effective than few-shot ICL on average, suggesting that there is benefit from reading twice, beyond simply improving the model's understanding of the task via few-shot examples.

One failure mode we observe with few-shot prompts is that the model sometimes outputs the attribute-value (e.g. director name given HTML text from different movie web pages) from the example documents instead of the relevant input document from which we seek to extract information.

\begin{table*}[!ht]
\centering
\begin{tabular}{l|ccc|ccc|ccc|ccc}
\hline
&\multicolumn{3}{c|}{Mamba-130m} & \multicolumn{3}{c|}{Mamba-370m} & \multicolumn{3}{c|}{Mamba-1.4B} & \multicolumn{3}{c}{Mamba-2.8B}\\
 & DF & FS & JP  & DF & FS & JP & DF & FS & JP & DF & FS & JP \\ \hline
\hline
FDA & 25.7 & 22.0 & 32.8 & 41.9 & 35.3 & 58.3 & 45.8&46.0 &60.9 & 54.3 & 54.8 & 66.6\\
SWDE & 17.5 & 19.7 &31.5 & 27.6 &35.0 & 42.2 & 37.6 & 47.1 &46.0 & 38.9 &51.9 & 48.9 \\
SQUAD&27.1 &25.2 &51.9 &34.9 & 36.0 &51.0 &39.9 &45.5 &59.6 & 43.9 &53.2 &59.4 \\
\hline
\end{tabular}
\caption{\textbf{{\prompt} ablations}. Here we evaluate three ICL baselines: DF is default prompt; FS is a prompt with $2$ in-context examples; JP is {\prompt}.}
\label{table:incontext-result}
\vspace{-3mm}
\end{table*}

\clearpage

\section{Prompts}
\label{app:prompts}

Below we include the prompts for the default and {\prompt} in-context learning results that produced the numbers in \Cref{table:main-quality}. We use the exact same prompt structure for all examples in the task and across all models. We use a shared structure across groups of tasks e.g., information extraction tasks SWDE and FDA use the same prompt structure and document QA tasks (NQ, TriviaQA, Drop, SQUAD). 

\subsection{SWDE}

\begin{mdframed}[frametitle={SWDE (Default)}]
Input
\begin{codeframe}
\begin{lstlisting}
The Evil Dead Movie Facts and Details click here amc home | movie guide Genres\nLists\nRatings amctv.com>movie guide>The Evil Dead>details The Evil Dead details\nOverall Rating Total Ratings: 1 Overview\nDetails\nCast & Credits\nAwards\nReview Movie Details: Director: Sam Raimi\nProduced By: New Line Cinema, Renaissance Pictures\nYear: 1983\nRun Time: 85 minutes\nCountry: USA\nLanguage: English MPAA Rating: R\nCategory: Feature\nGenre/Type: Horror\nFilmed In: Color Key Cast: Bruce Campbell, Ellen Sandweiss, Betsy Baker, Hal Delrich

... many document tokens ...

cranked up the story's comic aspects several dozen notches for the rollicking semi-remake, Evil Dead 2: Dead by Dawn. by Cavett Binion, Rovi Keywords: atrocity\nbook\ncabin\ncellar\nchainsaw\ndemon\ndismemberment\ngateway-to-hell\nmonster\ndemonic-possession rampage\nsatanic\nSatanism\nslasher\ntree\nweekend\nwoods [place]\ncollege-student\ninvocation Themes: Zombies\nDemonic Possession\nNightmare Vacations\nCurses and Spells Exclusive coverage Get Dragged to Hell With This Ultimate Sam Raimi Fan Quiz - Horror Hacker - AMCfrom AMC Blogs\nInside the Unlikely Cult of Road House - AMC Movie Blog - AMCfrom AMC Blogs\nU.S. Marshals and Five Other Stealth. Year:
\end{lstlisting}
\end{codeframe}
Ground Truth
\begin{codeframe}
\begin{lstlisting}
1983
\end{lstlisting}
\end{codeframe}
\end{mdframed}

\begin{mdframed}[frametitle={SWDE (Twice)}]
Input
\begin{codeframe}
\begin{lstlisting}
Information about Year. The Evil Dead Movie Facts and Details click here amc home | movie guide Genres\nLists\nRatings amctv.com>movie guide>The Evil Dead>details The Evil Dead details\nOverall Rating Total Ratings: 1 Overview\nDetails\nCast & Credits\nAwards\nReview Movie Details: Director: Sam Raimi\nProduced By: New Line Cinema, 

... many document tokens ...

U.S. Marshals and Five Other Stealth.
The Evil Dead Movie Facts and Details click here amc home | movie guide Genres\nLists\nRatings amctv.com>movie guide>The Evil Dead>details The Evil Dead details\nOverall Rating Total Ratings: 1 Overview\nDetails\nCast & Credits\nAwards\nReview Movie Details: Director: Sam Raimi\nProduced By: New Line Cinema, Renaissance Pictures\nYear: 1983

... many document tokens ...

 With This Ultimate Sam Raimi Fan Quiz - Horror Hacker - AMCfrom AMC Blogs\nInside the Unlikely Cult of Road House - AMC Movie Blog. Year:
\end{lstlisting}
\end{codeframe}
Ground Truth
\begin{codeframe}
\begin{lstlisting}
1983
\end{lstlisting}
\end{codeframe}
\end{mdframed}

\clearpage

\subsection{Natural Questions }

\begin{mdframed}[frametitle={Natural Questions  (Default)}]
Input
\begin{codeframe}
\begin{lstlisting}
List of Nobel laureates in Physics - wikipedia <H1> List of Nobel laureates in Physics </H1> Jump to : navigation, search Front side ( obverse ) of the Nobel Prize Medal for Physics presented to Edward Victor Appleton in 1947 <P> The Nobel Prize in Physics ( Swedish : Nobelpriset i fysik ) is awarded annually by the Royal Swedish Academy of Sciences to scientists in the various fields of physics. 

... many document tokens ...


The first Nobel Prize in Physics was awarded to
\end{lstlisting}
\end{codeframe}
\begin{codeframe}
\begin{lstlisting}
Wilhelm Conrad Rontgen, of Germany
\end{lstlisting}
\end{codeframe}
\end{mdframed}

\begin{mdframed}[frametitle={Natural Questions  (Twice)}]
Input
\begin{codeframe}
\begin{lstlisting}
Who got the first nobel prize in physics? List of Nobel laureates in Physics - wikipedia <H1> List of Nobel laureates in Physics </H1> Jump to : navigation, search Front side ( obverse ) of the Nobel Prize Medal for Physics presented to Edward Victor Appleton in 1947 <P> The Nobel Prize in Physics ( Swedish : Nobelpriset i fysik ) is awarded annually by the Royal Swedish Academy of Sciences to scientists in the various fields of physics. 

... many document tokens ...

for their joint researches on the radiation phenomena discovered by Professor Henri Becquerel
List of Nobel laureates in Physics - wikipedia <H1> List of Nobel laureates in Physics </H1> Jump to : navigation, search Front side ( obverse ) of the Nobel Prize Medal for Physics presented to Edward Victor Appleton in 1947 <P> The Nobel Prize in Physics ( Swedish : Nobelpriset i fysik ) is awarded annually by the Royal Swedish Academy of Sciences to scientists in the various fields of physics. 

... many document tokens ... 

for their joint researches on the radiation phenomena discovered by Professor Henri Becquerel. The first Nobel Prize in Physics was awarded to
\end{lstlisting}
\end{codeframe}
\begin{codeframe}
\begin{lstlisting}
Wilhelm Conrad Rontgen, of Germany
\end{lstlisting}
\end{codeframe}
\end{mdframed}

\clearpage

\subsection{FDA}

\begin{mdframed}[frametitle={FDA (Default)}]
Input
\begin{codeframe}
\begin{lstlisting}
510(k) SUBSTANTIAL EQUIVALENCE DETERMINATION DECISION SUMMARY A. 510(k) Number: K153137 B. Purpose for Submission: Clearance of a new device C. Measurand: Anti-PF4/Heparin Total Antibodies D. Type of Test: Automated, latex enhanced immuno-turbidimetric assay E. Applicant: Instrumentation Laboratory (IL) Co. F. Proprietary and Established Names: 
HemosIL HIT-Ab 
HemosIL HIT-Ab 
Controls G. Regulatory Information: 1. Regulation section: 21 CFR 864.7695, Platelet factor 4 radioimmunoassay 21 CFR 864.5425, Multipurpose system for in vitro coagulation studies 2. 

... many document tokens ...

Low HIT Control: 
Control intended for the assessment of precision and accuracy of the assay at PF4/H antibody levels at or below the cut-off. 
High HIT Control: Control intended for the assessment of precision and accuracy of the assay at abnormal PF4/H antibody levels. J. Substantial Equivalence Information: 1. 
Predicate device name(s): Asserachrom HPIA Test kit from Diagnostica Stago 2. Predicate 510(k) number(s): K003767 3. Comparison with predicate: 4 Similarities Item Device Predicate Trade Names HemosIL HIT-Ab(PF4-H) HemosIL HIT-Ab (PF4-H) Controls (K153137) Asserachrom HPIA Test Kit (kit includes two control levels) (K003767) Measurand Anti-PF4/Heparin Total Antibodies AntiPF. Purpose for submission:
\end{lstlisting}
\end{codeframe}
\begin{codeframe}
\begin{lstlisting}
Clearance of a new device
\end{lstlisting}
\end{codeframe}
\end{mdframed}

\begin{mdframed}[frametitle={FDA (Twice)}]
Input
\begin{codeframe}
\begin{lstlisting}
Information about Purpose for submission. 510(k) SUBSTANTIAL EQUIVALENCE DETERMINATION DECISION SUMMARY A. 510(k) Number: K153137 B. Purpose for Submission: Clearance of a new device C. Measurand: Anti-PF4/Heparin Total Antibodies D. Type of Test: Automated, latex enhanced immuno-turbidimetric assay E. Applicant: Instrumentation Laboratory (IL) Co. F. 

... many document tokens ...

Predicate device name(s): Asserachrom HPIA Test kit from Diagnostica Stago 2. Predicate 510(k) number(s): K003767 3. Comparison with predicate: 4 Similarities Item Device Predicate Trade Names HemosIL HIT-Ab(PF4-H) HemosIL HIT-Ab(PF4-H) Controls (K153137) Asserachrom HPIA Test Kit (kit includes two control levels) (K003767) Measurand Anti-PF4/Heparin Total Antibodies Anti-PF.
510(k) SUBSTANTIAL EQUIVALENCE DETERMINATION DECISION SUMMARY A. 510(k) Number: K153137 B. Purpose for Submission: Clearance of a new device C. Measurand: Anti-PF4/Heparin Total Antibodies D. Type of Test: Automated, latex enhanced immuno-turbidimetric assay E. Applicant: Instrumentation Laboratory (IL) Co. F. 

... many document tokens ...

Predicate device name(s): Asserachrom HPIA Test kit from Diagnostica Stago 2. Predicate 510(k) number(s): K003767 3. Comparison with predicate: 4 Similarities Item Device Predicate Trade Names 
HemosIL HIT-Ab(PF4-H) HemosIL HIT-Ab(PF4-H) Controls (K153137) Asserachrom HPIA Test Kit (kit includes two control levels) (K003767) 
Measurand Anti-PF4/Heparin Total Antibodies Anti-PF. Purpose for submission:
\end{lstlisting}
\end{codeframe}
\begin{codeframe}
\begin{lstlisting}
Clearance of a new device
\end{lstlisting}
\end{codeframe}
\end{mdframed}

\clearpage

\subsection{SQUAD}

\begin{mdframed}[frametitle={SQUAD (Default)}]
Input
\begin{codeframe}
\begin{lstlisting}
Super Bowl 50 was an American football game to determine the champion of the National Football League (NFL) for the 2015 season.
The American Football Conference (AFC) champion Denver Broncos defeated the National Football Conference (NFC) champion Carolina Panthers 24-10 to earn their third Super Bowl title. 
The game was played on February 7, 2016, at Levi's Stadium in the San Francisco Bay Area at Santa Clara, California. As this was the 50th Super Bowl, the league emphasized the "golden anniversary" with various gold-themed initiatives, as well as temporarily suspending the tradition of naming each 
Super Bowl game with Roman numerals (under which the game would have been known as "Super Bowl L"), so that the logo could prominently feature the Arabic numerals 50.The NFL team that represented the AFC at Super Bowl 50 was the
\end{lstlisting}
\end{codeframe}
\begin{codeframe}
\begin{lstlisting}
Denver Broncos
\end{lstlisting}
\end{codeframe}
\end{mdframed}

\begin{mdframed}[frametitle={SQUAD (Twice)}]
Input
\begin{codeframe}
\begin{lstlisting}
Which NFL team represented the AFC at Super Bowl 50? Super Bowl 50 was an American football game to determine the champion of the National Football League (NFL) for the 2015 season. The American Football Conference (AFC) champion Denver Broncos defeated the National Football Conference (NFC) champion Carolina Panthers 24-10 to earn their third Super Bowl title. The game was played on February 7, 2016, at Levi's Stadium in the San Francisco Bay Area at Santa Clara, California. As this was the 50th Super Bowl, the league emphasized the "golden anniversary" with various gold-themed initiatives, as well as temporarily suspending the tradition of naming each Super Bowl game with Roman numerals (under which the game would have been known as "Super Bowl L"), so that the logo could prominently feature the Arabic numerals 50.
Super Bowl 50 was an American football game to determine the champion of the National Football League (NFL) for the 2015 season. The American Football Conference (AFC) champion Denver Broncos defeated the National Football Conference (NFC) champion Carolina Panthers 24-10 to earn their third Super Bowl title. The game was played on February 7, 2016, at Levi's Stadium in the San Francisco Bay Area at Santa Clara, California. As this was the 50th Super Bowl, the league emphasized the "golden anniversary" with various gold-themed initiatives, as well as temporarily suspending the tradition of naming each Super Bowl game with Roman numerals (under which the game would have been known as "Super Bowl L"), so that the logo could prominently feature the Arabic numerals 50.The NFL team that represented the AFC at Super Bowl 50 was the
\end{lstlisting}
\end{codeframe}
\begin{codeframe}
\begin{lstlisting}
Denver Broncos
\end{lstlisting}
\end{codeframe}
\end{mdframed}

\clearpage
\subsection{TriviaQA}

\begin{mdframed}[frametitle={TriviaQA (Default)}]
Input
\begin{codeframe}
\begin{lstlisting}
81 years since the first inflight movie was shown...81 years since the first inflight movie was shown - Travelers United Travelers United 81 years since the first inflight movie was shown 
October 8, 2010 Filed Under: Today By Charlie Leocha Leave a Comment Our government at work - This is the daily ''Profile America'' feature from the U.S. Census Bureau for today, Friday, October 8th. 
This is the 81st anniversary of the first inflight movie ever shown. A little-known travel gem. 
Friday, October 8th, celebrates one of the few joys left in long-distance flying, sitting back and enjoying a feature-length movie. 
But recently, one major airline announced it will be ending this entertainment, joining several low-cost airlines in the policy. 
While movies have been generally available on long flights for decades, the first movies shown in the air were a newsreel and two cartoons. 
These were shown on this date in 1929 aboard a Ford Trimotor operated by Transcontinental Air Transport. Regular in-flight movie service began in July 1961 on a Trans World airline flight from New York to Los Angeles. 
Now, more than 3.9 million passengers fly between New York and Los Angeles every year. You can find these and more facts about America from the U.S. Census Bureau online. The first in-flight movie was shown on an internal flight in the USA in
\end{lstlisting}
\end{codeframe}
\begin{codeframe}
\begin{lstlisting}
1929
\end{lstlisting}
\end{codeframe}
\end{mdframed}

\begin{mdframed}[frametitle={TriviaQA (Twice)}]
Input
\begin{codeframe}
\begin{lstlisting}
In what year was the first in-flight movie shown on an internal flight in the USA? 81 years since the first inflight movie was shown...81 years since the first inflight movie was shown - Travelers United  Travelers United  81 years since the first inflight movie was shown  October 8, 2010  Filed Under: Today By Charlie Leocha Leave a Comment  .... These were shown on this date in 1929 aboard a Ford Trimotor operated by Transcontinental Air Transport. Regular in-flight movie service began in July 1961 on a Trans World airline flight from New York to Los Angeles. Now, more than 3.9 million passengers fly between New York and Los Angeles every year. You can find these and more facts about America from the U.S. Census Bureau online at.
81 years since the first inflight movie was shown...81 years since the first inflight movie was shown - Travelers United  Travelers United  81 years since the first inflight movie was shown  October 8, 2010  Filed Under: Today By Charlie Leocha Leave a Comment  ...  These were shown on this date in 1929 aboard a Ford Trimotor operated by Transcontinental Air Transport. Regular in-flight movie service began in July 1961 on a Trans World airline flight from New York to Los Angeles. Now, more than 3.9 million passengers fly between New York and Los Angeles every year. You can find these and more facts about America from the U.S. Census Bureau online at. The first in-flight movie was shown on an internal flight in the USA in
\end{lstlisting}
\end{codeframe}
\begin{codeframe}
\begin{lstlisting}
1929
\end{lstlisting}
\end{codeframe}
\end{mdframed}

\clearpage
\subsection{Drop}

\begin{mdframed}[frametitle={Drop (Default)}]
Input
\begin{codeframe}
\begin{lstlisting}
Hoping to rebound from their loss to the Patriots, the Raiders stayed at home for a Week 16 duel with the Houston Texans.  Oakland would get the early lead in the first quarter as quarterback JaMarcus Russell completed a 20-yard touchdown pass to rookie wide receiver Chaz Schilens.  The Texans would respond with fullback Vonta Leach getting a 1-yard touchdown run, yet the Raiders would answer with kicker Sebastian Janikowski getting a 33-yard and a 30-yard field goal.  Houston would tie the game in the second quarter with kicker Kris Brown getting a 53-yard and a 24-yard field goal. Oakland would take the lead in the third quarter with wide receiver Johnnie Lee Higgins catching a 29-yard touchdown pass from Russell, followed up by an 80-yard punt return for a touchdown.  The Texans tried to rally in the fourth quarter as Brown nailed a 40-yard field goal, yet the Raiders' defense would shut down any possible attempt. The first touchdown of the game was scored by
\end{lstlisting}
\end{codeframe}
\begin{codeframe}
\begin{lstlisting}
Chaz Schilens
\end{lstlisting}
\end{codeframe}
\end{mdframed}

\begin{mdframed}[frametitle={Drop (Twice)}]
Input
\begin{codeframe}
\begin{lstlisting}
Who scored the first touchdown of the game? Hoping to rebound from their loss to the Patriots, the Raiders stayed at home for a Week 16 duel with the Houston Texans.  Oakland would get the early lead in the first quarter as quarterback JaMarcus Russell completed a 20-yard touchdown pass to rookie wide receiver Chaz Schilens.  The Texans would respond with fullback Vonta Leach getting a 1-yard touchdown run, yet the Raiders would answer with kicker Sebastian Janikowski getting a 33-yard and a 30-yard field goal.  Houston would tie the game in the second quarter with kicker Kris Brown getting a 53-yard and a 24-yard field goal. Oakland would take the lead in the third quarter with wide receiver Johnnie Lee Higgins catching a 29-yard touchdown pass from Russell, followed up by an 80-yard punt return for a touchdown.  The Texans tried to rally in the fourth quarter as Brown nailed a 40-yard field goal, yet the Raiders' defense would shut down any possible attempt.
Hoping to rebound from their loss to the Patriots, the Raiders stayed at home for a Week 16 duel with the Houston Texans.  Oakland would get the early lead in the first quarter as quarterback JaMarcus Russell completed a 20-yard touchdown pass to rookie wide receiver Chaz Schilens.  The Texans would respond with fullback Vonta Leach getting a 1-yard touchdown run, yet the Raiders would answer with kicker Sebastian Janikowski getting a 33-yard and a 30-yard field goal.  Houston would tie the game in the second quarter with kicker Kris Brown getting a 53-yard and a 24-yard field goal. Oakland would take the lead in the third quarter with wide receiver Johnnie Lee Higgins catching a 29-yard touchdown pass from Russell, followed up by an 80-yard punt return for a touchdown.  The Texans tried to rally in the fourth quarter as Brown nailed a 40-yard field goal, yet the Raiders' defense would shut down any possible attempt. The first touchdown of the game was scored by
\end{lstlisting}
\end{codeframe}
\begin{codeframe}
\begin{lstlisting}
Chaz Schilens
\end{lstlisting}
\end{codeframe}
\end{mdframed}

\clearpage 
\newpage
\section{Theoretical results}
\label{app:theory}
We begin by setting notation.
\paragraph{Notation.} We will be denoting the all \(1\) row vector of size $k$, given by \(\begin{bmatrix}1&1&\ldots&1&1\end{bmatrix}\),  and the all \(0\) row vector of size $k$, given by \(\begin{bmatrix}0&0&\ldots&0& 0\end{bmatrix}\), as \(\bm{1}^{k}\) and \(\bm{0}^{k}\), respectively. We will also construe the standard basis vector $\mathbf{e}_i$ as a column vector in these notes, and adhere to the following matrix indexing convention: $\tbf{M}[i,j]$ is the entry in the $i$th row and the $j$th column,  $\tbf{M}[i,:] \in \F^{1 \times n}$ denotes the $i$th row, and $\tbf{M}[:,j] \in \F^{m \times 1}$ denotes the $j$th column of $\tbf{M} \in \F^{m \times n},$ where $\F$ is a field and the reader can substitute $\F$ for $\R$ for convenience. We then use $\bm{1}^{m \times n}, \mathbf{0}^{m \times n} \in \F^{m \times 1}$ to denote the matrix of all $1$s and $0$s, respectively.

Next, we denote the {\em Hadamard product} of vectors $\tbf{u}, \tbf{v} \in \F^n$ as $\tbf{u}\odot \tbf{v}$; the operation can be extended to matrices by applying the Hadamard product column-wise across the matrices. This is commonly referred to as {\em(element-wise) gating}. For vectors $\tbf{u}, \tbf{v} \in \F^n$, we also denote their {\em linear (or acyclic) convolution} as $\tbf{u} \ast \tbf{v}$ and {\em cyclic convolution} as $\tbf{u} \circledast \tbf{v}$.

 We also recall the definition of \BaseConv\ for the reader's convenience:
\begin{definition}[\BaseConv~\citep{arora2023zoology}]
\label{def:baseconv}
Given an input sequence $\vu \in \R^{\inputLength \times \modelDim},$ where $\inputLength$ is the sequence length and $\modelDim$ is the model dimension, a learned weight matrix $\mW^B \in \R^{\modelDim \times \modelDim}$ and biases $\mB^B, \mB^K \in \R^{\inputLength \times \modelDim}$ and a matrix of convolution filters $\mK \in \R^{\inputLength \times \modelDim}$, a \BaseConv\ layer computes the following:
\begin{equation}
\label{eq: baseconv}
    \bm{z}^{\BaseConv} := (\vu{\mW}^B+\mB^B) \odot \paren{{\mK} \ast \vu + \mB^K} \in \R^{\inputLength \times \modelDim},
\end{equation}
where the convolutions are applied across the input length $\inputLength$.
\end{definition}

We will need the following ``$5$-tuple" notation for \BaseConv\ model:
\begin{definition}
\label{def: gated-conv}
   An $\modelTuple{\inputLength}{\depth}{\headDim}{\innerN}{\innerD}-\BaseConv$ is a stacked sequence to sequence model with $L$ layers such that:
   \begin{enumerate}
     \item  input and output are $\inputLength\times \headDim$ matrices, 
    \item  each layer corresponds to the a \BaseConv\ layer as defined in \Cref{def:baseconv},  and
    \item all the individual gated convolution layers take in $\innerN \times \innerD$ matrices and output $\innerN \times \innerD$ matrices. We refer to the tuple $\innerDim$ as the \emph{inner dimension} of the model.  
    \end{enumerate}
\end{definition}
We also assume that the input $\modelInput \in \R^{\inputDim}$ is embedded into $\modelInput' \in \R^{\innerN \times \innerD}$  such that

\[
    \modelInput'[n,t] = \begin{cases}
    \modelInput[n,t] \ \ \text{ if } n < \inputLength, \ t < \headDim \ \\
    0 \ \ \text{ otherwise. }
    \end{cases}
\]
The output from the last layer $\vz \in \R^{\innerN \times \innerD}$ is transformed into output $\bm{y} \in R^{\inputDim}$ by extracting the top left $\inputDim$ entries in $\vz$.

\begin{definition}
\label{def:mlp}
    An MLP layer is map $\R^{\inputDim}\to \R^{\inputDim}$ defined via matrices $\mW^1,\mW^2\in\R^{\headDim\times\headDim}$ and ``bias" matrices $\mB^1,\mB^2\in\R^{\inputDim}$ as follows:
    \[\MLP(\modelInput)=\mathrm{ReLU}(\modelInput\mW^1+\mB^1)\mW^2+\mB^2.\]
\end{definition}

\subsection{JRT Lower Bounds for \BaseConv}
First, we formally define JRT prompts below.
\begin{definition}[JRT Prompts]
For any model $\calM$ with input $\bm{u} \in \R^{\inputLength \times \modelDim}$, a {\em JRT prompt} for input $\bm{u}$ is the repeated input $\bm{u}^{\mathrm{JRT}} \in {\R^{2\inputLength \times \modelDim}}$ given by
\[
\bm{u}^{\mathrm{JRT}}[i,:] := \begin{cases}
    \bm{u}[i,:] & \text{if } i < N \\
    \bm{u}[i-N,:] & \text{otherwise.} \\
\end{cases}
\]
\end{definition}
\subsubsection{Lower Bound on the Number of Layers for AR}
\label{sec: lower-bound-AR-layers}
In this section, we will provide a lower bound on the number of layers needed to solve the standard associative recall problem with JRT prompts. We formally recall the associative recall problem:
\begin{displayquote}
     The AR problem takes key-value pairs $\{\bm{k}_i, \bm{v}_i\}_{i = 0}^{N-1}$ along with a query $\bm{q}$ appended at the end as input and the goal is to output $\bm{v}_i$ if $\bm{q} = \bm{k}_i$ for some $i \in [0, N-1]$.
\end{displayquote} 
 We also require a randomized communication complexity lower bound result for the {\em index problem}:
\begin{displayquote}
    The index problem has two agents, Alice and Bob, where Alice has a string $\bm{x} \in \{0,1\}^n$ and Bob has an index $i \in [n]$, and the goal for the players is to output the $i$-th entry $\bm{x}_i$. Moreover, we also require the communication to be {\em one-way}: only Alice is allowed to send a single message to Bob and Bob needs to output the answer.
\end{displayquote} 
We will use the following well-known lower bound for the index problem.
\begin{theorem}[\cite{jayram2008one}]
\label{thm: space-index}
    The one-way randomized communication complexity\footnote{The randomized communication complexity of function $f$ is defined as $\min_{\pi} \norm{\pi}$, where $\pi$ ranges over all randomized protocols that can solve $f$ with probability of success at least $2/3$.} of the index problem for an $n$-length bit string is $\Omega(n)$.
\end{theorem}

We will now mirror the argument from \citep[Theorem F.4]{arora2024simple} to show that the lower bound on the number of layers for a \BaseConv\ model solving AR still holds for JRT prompts.
\begin{theorem}
\label{thm: JRT-layer-ar}
    Given a {\em JRT prompt} $\vu^{\mathrm{JRT}} \in \{0,1\}^{2\inputLength \times \modelDim}$ for input $\vu \in \{0,1\}^{\inputLength \times \modelDim}$ to the AR problem with any encoding such that $\log{c} \le d \le 2^{(\log{N})^{1-\epsilon}}$ for $\epsilon > 0$,
    and $c$ possible tokens from the vocabulary with $c \le N$,
    a data-independent \BaseConv\ model with model parameters taking $O(\log{N})$ bits needs $\Omega(\epsilon\log\log{N})$ layers to solve AR.
\end{theorem}
\begin{proof}
    Given a $\BaseConv$ model $\calM$ solving AR, regardless of the input length $\inputLength$, we know that there exists an equivalent polynomial $P(\vu^{\mathrm{JRT}})$ of degree at most $2^L$ that solves AR for any $\vu^{\mathrm{JRT}} \in \{0,1\}^{2\inputLength \times \modelDim}$, where $L$ denotes the number of layers.\footnote{See the proof of \citep[Theorem F.4]{arora2024simple} for justification.}
    Now, take the instance $(\bm{x}, i)$ of the index problem with $\bm{x} \in \{0,1\}^N$ and the corresponding JRT prompt of the AR problem as before
    \begin{equation}
        \label{eq: ar-index-dup}
        \vu^{\mathrm{JRT}} := \{j, \bm{x}_j\}_{j = 0}^{N-1}, i, \{j, \bm{x}_j\}_{j = 0}^{N-1}, i
    \end{equation}
    Next, we build the following one-way protocol for solving the index problem using the \BaseConv\ model from the hypothesis that it solves AR. Alice with their access of $\bm{x} \in \{0,1\}^N$ will again generate a JRT input  $\vu^{\mathrm{JRT}}$ for AR (without the query) as in \eqref{eq: ar-index-dup}. More specifically, Alice takes the values $\va:= \vu^{\mathrm{JRT}}[0: N-2, :] \equiv \vu^{\mathrm{JRT}}[N: 2N-2, :] \in \{0,1\}^{2(N-1)\times d}$ while leaving out the query $\vq:= \vu^{\mathrm{JRT}}[\inputLength-1,:] = \vu^{\mathrm{JRT}}[2\inputLength - 1,:]$, and substitutes these known $2(N-1)d$ values to define the following polynomial:
    \begin{equation}
        \label{eq: q-poly}\
        Q^{\mathrm{JRT}}(\vq)=P(\va,\vq, \va, \vq).
    \end{equation}
    Crucially, $Q^{\mathrm{JRT}}$ is still a polynomial in $d$ variables, corresponding to the values $\vu^{\mathrm{JRT}}[N-1, :] = \vu^{\mathrm{JRT}}[2N-1,:]$ that Bob has and trivially has degree $D \le 2^L$. 
    As in the proof of \citep[Theorem F.4]{arora2024simple}, Alice can run the model $\calM$, retrieve the coefficients of $Q^{\mathrm{JRT}}$, and send it to Bob. Since we assume that $P$ solves AR, Bob can take the coefficients of $Q^{\mathrm{JRT}}$ and substitute $\vu^{\mathrm{JRT}}[N-1, :] = \vu^{\mathrm{JRT}}[2N-1,:]$ to $Q^{\mathrm{JRT}}$ to compute $P(\vu^{\mathrm{JRT}})$ which is the value $\vx_i$.

    Moreover, the polynomial $Q^{\mathrm{JRT}}$ that Alice sends still has at most $d^{2^L}$ coefficients as each term in $Q^{\mathrm{JRT}}$ can have degree at most $2^L$.  If each such coefficient has $B$ bits, then using \cref{thm: space-index}, the total number of bits being communicated must satisfy $B \cdot d^{2^L} \ge \Omega(N)$. This follows from the fact that if $B\cdot d^{2^L} \le o(N)$, then since the associated value of $i$ in \eqref{eq: ar-index-dup} is the answer to the indexing problem, we have shown that a one-way communication protocol for solving the index problem uses $o(N)$ communication complexity, which then contradicts \cref{thm: space-index}. This is the same equation we get in the proof of \citep[Theorem F.4]{arora2024simple}, 
    which yields the following lower bound on the number of layers:
    \begin{align}
    L 
    &\ge {\log\paren{\frac{\log{N} - \log{B}}{(\log{N})^{1-\epsilon}}}}\label{eq: L-N-B}.
    \end{align} 
    Recall here that the model parameters are assumed to be $O(\log{N})$ bits, so any coefficient in $Q^{\mathrm{JRT}}$ should have absolute value at most $\paren{2^{O(\log{N})}\cdot 2Nd}^{2^L}$ as each coefficient can be a product of at most $2Nd$ variables. That is, for some $\alpha>0$, we have the following bound on each coefficient:
    \[
       2^B \le (2\cdot N^{\alpha+1}d)^{2^L} \le (2N^{(\alpha+2)})^{2^L}
    \]
    where the last equality uses the fact that $d \le 2^{\log{N}^{(1-\epsilon)}} \le N$. We thus have 
    \begin{equation}
        \label{eq: q-coefficient}
    \log(B) \le \log(\alpha+2) + L + \log\log\paren{2N}.
    \end{equation}
 Substituting \eqref{eq: q-coefficient} to \eqref{eq: L-N-B}, we get
    \begin{align}
        L \ge \log\paren{\frac{\log{N} - \log(\alpha+2) - L - \log\log\paren{2N}}{(\log{N})^{1-\epsilon}}} \label{eq: L-lower}
    \end{align}
    Now, if $L > \log\log{2N}$, we are done. Otherwise, if $L \le \log\log\paren{2N}$, then we can substitute this to \eqref{eq: L-lower} to get
    \begin{align}
        L &\ge  \log\paren{\frac{\log{N} - \log(\alpha+2) - 2\log\log\paren{2N}}{(\log{N})^{1-\epsilon}}} \nonumber\\
        &= \log\paren{\log{N} - \log(\alpha+2) - 2\log\log{2N}} - (1-\epsilon)\log\log{N}\label{eq: L-explicit}
    \end{align}
    We now claim that first term in \eqref{eq: L-explicit} satisfies the following:
    \begin{equation}
        \label{eq: claim}
        \log\paren{{\log{N} - \log(\alpha+2) - 2\log\log\paren{2N}}} \ge (1-\frac{\epsilon}{2}) \log\log{N}.
    \end{equation}
    To see this, note that, for sufficiently large enough $N$, the following holds:
    \begin{align*}
        \frac{\log{N}}{2} \ge \log(\alpha + 2) + 2\log\log\paren{2N},
    \end{align*}
    hence, we get
    \[
    \log\paren{{\log{N} - \log(\alpha+2) - 2\log\log\paren{2N}}} \ge \log\paren{\frac{\log{N}}{2}} \ge \log\log{N} -1 \ge (1-\frac{\epsilon}{2}) \log\log{N}.
    \]
    This proves the claim in \eqref{eq: claim}. Finally, using \eqref{eq: claim}, \eqref{eq: L-explicit} leads to the following:
    \[
    L \ge (1-\frac{\epsilon}{2}) \log\log{N} - (1-\epsilon)\log\log{N} = \frac{\epsilon}{2}\log\log{N},
    \]
    which still provides the lower bound $L = \Omega(\epsilon\log\log{N})$, as desired. 
\end{proof}
\subsubsection{Lower Bounds for MQAR with $d = \log_2{c}$}
Next, we present lower bounds for the {\em mulitple-query associative recall} (MQAR) problem which generalizes the AR problem~\citep{arora2023zoology}. To this end, we recall the definition of MQAR below.
\begin{displayquote}
    Suppose we are given an input sequence 
$
\bm{u}[0 \cdots 3N-1] \triangleq \left\{\paren{\bm{k}_0, \bm{v}_0, \bm{q}_0}, \ldots, \paren{\bm{k}_{{N}-1}, \bm{v}_{{N}-1}, \bm{q}_{{N}-1}}\right\}
$
with each $\bm{k}_i, \bm{v}_i, \bm{q}_i \in C$ is a token drawn from a vocabulary of size $c = |C|$.
Our goal is then to check, for each $1 \le i \le {N}-1$, whether there exists $0 \le j < i$ such that $\bm{q}_i \equiv \bm{k}_j$, and if so, output $\bm{v}_{j}$. 
\end{displayquote}
We now present the following lower bound from \citep{arora2024simple} for the MQAR problem $d = \log_2{c}$ to encode all $c$ possible tokens from $C$ using the natural binary encoding, which also holds for JRT input. This is because the result (Theorem F.5) in \citep{arora2024simple} is derived using Lemma 5.1 in \citep{arora2024simple} (degree of multilinear polynomial computed by \BaseConv\ in terms of its number of layers) and Lemma 5.2 in \citep{arora2024simple} (degree of multilinear polynomial for the MQAR problem), both of which are independent of the input length $\inputLength$.
\begin{theorem}
\label{thm: mqar-1hot-JRT}
    A data-independent \BaseConv\ model needs $\log(2d)$-layers to solve $\Task$ with a JRT prompt $\bm{u}\in \{0,1\}^{2 \cdot 3N\times d}$ for the original input $\bm{u}\in \{0,1\}^{3N\times d}$ with $d = \log_2(c)$.
\end{theorem}
\subsubsection{Lower Bounds for MQAR via the Equality (EQ) Problem}
\citep{arora2024simple} also contains lower bounds on the number of layers solving MQAR due to the lower bounds on the {\em equality problem} (EQ), where we define the equality problem (EQ) as checking whether the two encodings are equal: $\bm{u}_1 \equiv \bm{u}_2$ for an input pair $\bm{u}_1, \bm{u}_2$ where each $\bm{u}_i$ is a token drawn from a vocabulary of size $c = |C|$ and embedded in $\{0,1\}^d$. 

We next show that any model with JRT prompts solving MQAR also solves EQ.
\begin{proposition}
    \label{prop: eq_mqar_JRT}
    Any model $M_{\Task}$ that solves MQAR with JRT prompt also solves EQ using the same number of layers.
\end{proposition}
\begin{proof}
    If there exists a model $\textsc{M}_{\Task}$ that solves MQAR using $L$ layers with JRT prompt, then for an arbitrary input instance for EQ given by $\bm{u}_1, \bm{u}_2 \in \R^{2 \times d}$,
    we can produce the following input instance for MQAR: $\bm{u}:= \{(\bm{u}_1, \bbone, \bm{u}_1), (\bm{u}_2, \bbone, \bm{u}_2), (\bm{u}_1, \bbone, \bm{u}_1), (\bm{u}_2, \bbone, \bm{u}_2)\}$ and solve EQ using $L$ layers with $\textsc{M}_{\Task}$ returning $\bbone$ iff there is a match.
\end{proof}

Due to \cref{prop: eq_mqar_JRT}, we obtain the following corollary.
\begin{corollary}
\label{cor: eq-mqar-lower-JRT}
    Any lower bound $\overline{L}$ on the number of layers $L$ of \BaseConv\ to solving EQ is also a lower bound on the number of layers required for solving $\Task$ with JRT prompts. 
\end{corollary}
The lower bounds for the EQ problem in \citep{arora2024simple} depends on showing that the polynomial $P$ representing EQ in $p$-hot encoding has $\deg(P)\ge 2p$, which does not depend on the sequence length (Proposition F.5). Since \cref{cor: eq-mqar-lower-JRT} also holds in the JRT setting, we inherit the lower following  lower bound for $\BaseConv$ solving MQAR in the $p$-hot encoding setting, which we recall here for the reader's convenience.
\begin{definition}[$p$-Hot Encoding]
\label{def: phot}
    We define the {\em $p$-hot encoding} to be the collection of embeddings for a token $\bm{x}_t$ with $0\le t <c$ such that we express $t$ in base $\sqrt[p]{c}:(t_0,..,t_{p-1}) \in [0,\sqrt[p]{c})^p$ and represent each $t_i$ as one hot encoding in $\{0,1\}^{\sqrt[p]{c}}$. That is, we take $d = p\cdot\sqrt[p]{c}$.
\end{definition}

\begin{theorem}
\label{thm: mqar-phot}
     A data-independent \BaseConv\ model needs at least $\floor{\log(2p)}$-layers to solve $\Task$ for a JRT prompt $\bm{u}^{\mathrm{JRT}}\in \{0,1\}^{2\cdot 3N\times d}$ for the original input $\bm{u}\in \{0,1\}^{3N\times d}$ in the {$p$-hot encoding} setting, where $d = p\cdot\sqrt[p]{c}$.
\end{theorem}

\subsection{Recurrent Models and Set Disjointness}
\label{sec:rec-models-set-disjoint}
In this section, we will provide upper bounds on the class of recurrent models defined in \citep{arora2024simple} solving the set disjointness (SD) problem. First, we recall the definition of recurrent models below.
\begin{definition}[Recurrent Models]
\label{def: reg-model}
A model $\calM$ taking an input $\vu \in \R^{\inputLength \times \modelDim}$, where $\inputLength$ is the input length and $d$ is the model dimension, is termed a \emph{recurrent model} if its $i$-th state, representing the output at location $i$, $\mZ_{\calM}^i \in \R^{\innerD}$, with $\innerD$ denoting the state size, is determined exclusively by the preceding elements of the input $\vu[0 \ldots i-1]$. The state $\mZ_{\mathcal{M}}^i$ represents the accumulated information of the model depending on the inputs up to the $i$-th element, and is distinct from learned parameters that are static with respect to the input sequence.

Specifically, $\mZ_{\calM}^i(\vu) = \phi(\vu[0 \ldots i-1])$, indicating that the state is a function of the input history but not of the entire input sequence simultaneously. Moreover, we can express this as:
\begin{equation}
\label{eq: reg-dep}
\mZ_{\calM}^i(\vu) = f_{\calM}^i(\mZ_{\calM}^{i-1}, \vu[i]),
\end{equation}
for a sequence of functions $\{f_{\calM}^i\}_{i \in [\inputLength]}$, where each function is tailored to evolve the state based on the immediate past state and the current input. 
\end{definition}
\begin{remark}
    Note that \cref{def: reg-model} excludes models that inherently require the entire input sequence for computation at any state, such as those based on non-causal convolutional operations over the full input.
\end{remark}

\begin{remark}
\label{remark:sd-input}
Given sets $A, B \subseteq \{0,1\}^{\sdDim}$, the {\em set disjointness} (SD) problem seeks to check whether $A$ and $B$ are disjoint, that is, $A \cap B = \emptyset$. First, we clarify the format of the input $\bm{u} \in \{0,1\}^{\inputLength \times (\sdDim+1)}$ for the set-disjointness problem with $N = \abs{A} + \abs{B} + 1$. The rows of the input $\bm{u} \in \{0,1\}^{\inputLength \times (\sdDim+1)}$ correspond to elements in $A$ and $B$. That is, $\bm{u}[i, 0:\sdDim-1] \in A\cup B \cup \{\bm{0}^{\sdDim}\}$, where $\{[\bm{0}^{\sdDim}::1]\}$ is a separator element which separates the contiguously placed (in any arbitrary order) elements of each set with the last entry of non-separator rows equal to $0$.
\end{remark}

\begin{theorem}
\label{thm:rec-gen-sd}
    For any recurrent model $\calM$, there exists a function of the input history $\mZ_{\calM}^i(\vu^{\jrt}) = \phi(\vu^{\jrt}[0 \ldots i-1])$ that solves the set disjointness problem with $\mZ_{\calM}^{2N}$ of size $\calO(\sdDim \cdot \min\{|A|, |B|\})$ for the $\jrt$ prompt $\vu^{\jrt} \in \{0,1\}^{2N \times (\sdDim + 1)}$ of the input $\bm{u} \in \{0,1\}^{\inputLength \times (\sdDim+1)}$ for the set-disjointness problem.
\end{theorem}
\begin{proof}
    Given a JRT prompt $\bm{u}^{\jrt} \in \{0,1\}^{2N \times (\sdDim+1)}$ corresponding to the input for the set-disjointness problem, for a recurrent model $\calM$, we define the state $\mZ_{\calM}^i$ in Algorithm \ref{alg:sd-rec}.
    
    \begin{algorithm*}[!h]
  \caption{\label{alg:sd-rec} Recurrent Model for Set Disjointness}
  \small
  \begin{algorithmic}[1]
    \Require\ an input $\bm{u}^{\jrt} \in \{0,1\}^{2N \times (\sdDim+1)}$ for the set-disjointness problem
    \Ensure\ state size $\mZ_{\calM}^{2N-1}$.
    \State\label{line1} {\tt firstSeparator} $\leftarrow$ {\tt False}
    \State\label{line2} {\tt secondSeparator} $\leftarrow$ {\tt False}
    \State\label{line3} {\tt smallFirst} $\leftarrow$ {\tt False}
    \For{$i \leftarrow$ $0$ to ${2N-1}$}\label{line4}
        \If{$\bm{u}^{\jrt}[i,\sdDim] = 1$}\label{line5}
            \If{{\tt firstSeparator} = {\tt False}}\label{line6}
                \State {\tt firstSeparator} $\leftarrow$ {\tt True}\label{line7}
                \If{$i \le \floor{\frac{N}{2}}$}\label{line8}
                    \State {\tt smallFirst} $\leftarrow$ {\tt True}\label{line9}
                \EndIf
            \Else\label{line10}
                \State {\tt secondSeparator} $\leftarrow$ {\tt True}\label{line11}
            \EndIf
        \Else\label{line12}
            \If{{\tt firstSeparator} = {\tt True}}\label{line13}
                \If{{\tt smallFirst} = {\tt True}}\label{line14}
                    \If{{\tt secondSeparator} = {\tt False}}\label{line15}
                        \If{$i \ge N$}\label{line16}
                            \State Add $\bm{u}^{\jrt}[i, :]$ to $\mZ_{\calM}^{i}$ \label{line17}
                        \EndIf
                    \Else\label{line18}
                        \If{there exists $j$ s.t. $\bm{u}^{\jrt}[i,:] = \mZ_{\calM}^{i-1}[j, :]$}\label{line19}
                            \State $\mZ_{\calM}^{i-1}[j, \sdDim] = 1$ \label{line20}
                        \EndIf
                    \EndIf
                \Else\label{line21}
                    \If{{\tt secondSeparator} = {\tt False}}\label{line22}
                        \If{$i \le N$}\label{line23}
                            \State Add $\bm{u}^{\jrt}[i, :]$ to $\mZ_{\calM}^{i}$ \label{line24}
                        \Else\label{line25}
                            \If{there exists $j$ s.t. $\bm{u}^{\jrt}[i,:] = \mZ_{\calM}^{i-1}[j, :]$}\label{line26}
                                \State $\mZ_{\calM}^{i-1}[j, \sdDim] = 1$ \label{line27}
                            \EndIf
                        \EndIf
                    \EndIf
                \EndIf
            \EndIf
        \EndIf
    \EndFor
    \For{all $j$ s.t. $\mZ_{\calM}^{i-1}[j, \sdDim] = 1$}\label{line28}
        \State \Return $\mZ_{\calM}^{i-1}[j, 0:\sdDim-1]$.\label{line29}
    \EndFor
  \end{algorithmic}
\end{algorithm*}
    Semantically, we take a JRT input $\bm{u}^{\jrt} \in \{0,1\}^{2N \times (\sdDim+1)}$ for the set-disjointness problem, and find the first separator (lines \ref{line5} to \ref{line9}). If the index $i$ of the first separator is less than or equal to $\floor{\frac{N}{2}}$ (line \ref{line8}), then we know that the smaller set is placed before the larger set. Otherwise, the smaller set is placed later (see Figure \ref{fig: rec-sd}).

    \begin{figure*}[!h]
        \centering

            \tikzset{every picture/.style={line width=0.5pt}} 
            
            \begin{tikzpicture}[x=0.75pt,y=0.75pt,yscale=-0.85,xscale=0.85]
            
            \draw  [fill={rgb, 255:red, 216; green, 210; blue, 210 }  ,fill opacity=1 ] (21,52) -- (111,52) -- (111,86.92) -- (21,86.92) -- cycle ;
            \draw    (111,52) -- (111,86) ;
            \draw    (140,52) -- (140,86) ;
            \draw    (439,53) -- (439,87) ;
            \draw    (468,53) -- (468,87) ;
            \draw    (203.67,35.83) -- (346.67,36.81) ;
            \draw [shift={(349.67,36.83)}, rotate = 180.39] [fill={rgb, 255:red, 0; green, 0; blue, 0 }  ][line width=0.08]  [draw opacity=0] (8.93,-4.29) -- (0,0) -- (8.93,4.29) -- cycle    ;
            \draw    (26.67,35.83) -- (167.67,35.83) ;
            \draw [shift={(23.67,35.83)}, rotate = 0] [fill={rgb, 255:red, 0; green, 0; blue, 0 }  ][line width=0.08]  [draw opacity=0] (8.93,-4.29) -- (0,0) -- (8.93,4.29) -- cycle    ;
            \draw    (537.67,37.83) -- (673,38) ;
            \draw [shift={(676,38)}, rotate = 180.07] [fill={rgb, 255:red, 0; green, 0; blue, 0 }  ][line width=0.08]  [draw opacity=0] (8.93,-4.29) -- (0,0) -- (8.93,4.29) -- cycle    ;
            \draw    (352.67,36.83) -- (506.67,36.83) ;
            \draw [shift={(349.67,36.83)}, rotate = 0] [fill={rgb, 255:red, 0; green, 0; blue, 0 }  ][line width=0.08]  [draw opacity=0] (8.93,-4.29) -- (0,0) -- (8.93,4.29) -- cycle    ;
            \draw  [fill={rgb, 255:red, 239; green, 202; blue, 140 }  ,fill opacity=1 ] (140,52) -- (348.67,52) -- (348.67,86.92) -- (140,86.92) -- cycle ;
            \draw  [fill={rgb, 255:red, 216; green, 210; blue, 210 }  ,fill opacity=1 ] (349,52.08) -- (439,52.08) -- (439,87) -- (349,87) -- cycle ;
            \draw  [fill={rgb, 255:red, 239; green, 202; blue, 140 }  ,fill opacity=1 ] (468,53) -- (676.67,53) -- (676.67,87.92) -- (468,87.92) -- cycle ;
            \draw  [fill={rgb, 255:red, 216; green, 210; blue, 210 }  ,fill opacity=1 ] (256.33,178.08) -- (346.33,178.08) -- (346.33,213) -- (256.33,213) -- cycle ;
            \draw    (203,161.83) -- (346,162.81) ;
            \draw [shift={(349,162.83)}, rotate = 180.39] [fill={rgb, 255:red, 0; green, 0; blue, 0 }  ][line width=0.08]  [draw opacity=0] (8.93,-4.29) -- (0,0) -- (8.93,4.29) -- cycle    ;
            \draw    (26,161.83) -- (167,161.83) ;
            \draw [shift={(23,161.83)}, rotate = 0] [fill={rgb, 255:red, 0; green, 0; blue, 0 }  ][line width=0.08]  [draw opacity=0] (8.93,-4.29) -- (0,0) -- (8.93,4.29) -- cycle    ;
            \draw    (537,163.83) -- (672.33,164) ;
            \draw [shift={(675.33,164)}, rotate = 180.07] [fill={rgb, 255:red, 0; green, 0; blue, 0 }  ][line width=0.08]  [draw opacity=0] (8.93,-4.29) -- (0,0) -- (8.93,4.29) -- cycle    ;
            \draw    (352,162.83) -- (506,162.83) ;
            \draw [shift={(349,162.83)}, rotate = 0] [fill={rgb, 255:red, 0; green, 0; blue, 0 }  ][line width=0.08]  [draw opacity=0] (8.93,-4.29) -- (0,0) -- (8.93,4.29) -- cycle    ;
            \draw  [fill={rgb, 255:red, 239; green, 202; blue, 140 }  ,fill opacity=1 ] (22.33,178) -- (231,178) -- (231,212.92) -- (22.33,212.92) -- cycle ;
            \draw  [fill={rgb, 255:red, 216; green, 210; blue, 210 }  ,fill opacity=1 ] (584.33,178.08) -- (675.67,178.08) -- (675.67,211.83) -- (584.33,211.83) -- cycle ;
            \draw  [fill={rgb, 255:red, 239; green, 202; blue, 140 }  ,fill opacity=1 ] (346.33,178.08) -- (555,178.08) -- (555,213) -- (346.33,213) -- cycle ;
            
            \draw (179,28.4) node [anchor=north west][inner sep=0.75pt]    {$N$};
            \draw (514,29.4) node [anchor=north west][inner sep=0.75pt]    {$N$};
            \draw (178.33,154.4) node [anchor=north west][inner sep=0.75pt]    {$N$};
            \draw (513.33,155.4) node [anchor=north west][inner sep=0.75pt]    {$N$};
            \draw (573.78,179.63) node [anchor=north west][inner sep=0.75pt]  [font=\scriptsize,rotate=-90.93]  {$\mathbf{0}^{n} ::1$};
            \draw (246.78,179.94) node [anchor=north west][inner sep=0.75pt]  [font=\scriptsize,rotate=-90.93]  {$\mathbf{0}^{n} ::1$};
            \draw (128.78,52.94) node [anchor=north west][inner sep=0.75pt]  [font=\scriptsize,rotate=-90.93]  {$\mathbf{0}^{n} ::1$};
            \draw (456.78,52.94) node [anchor=north west][inner sep=0.75pt]  [font=\scriptsize,rotate=-90.93]  {$\mathbf{0}^{n} ::1$};

            \end{tikzpicture}
            \newline
        \caption{Placement of the smaller set is determined by when we first encounter the separator.}
        \label{fig: rec-sd}
    \end{figure*}
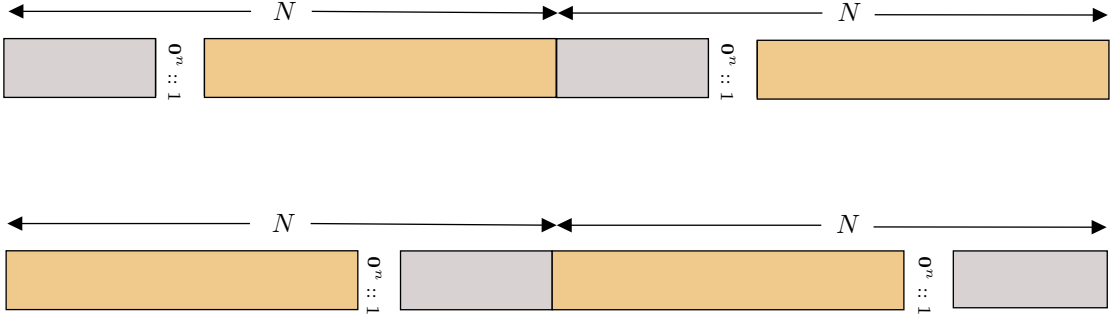
    
    Either way, we want to store the smaller set and compare it against the larger set for intersections. To this end, if the smaller set comes first (line \ref{line14}), then we continue until the beginning of the repeated input (line \ref{line16}) and collect the smaller set (line \ref{line17}), which we then use after we encounter the second separator (lines \ref{line19} to \ref{line20}) to compare against the larger set. If the smaller set comes second (lines \ref{line21} to \ref{line27}), then after the first separator, we collect the smaller set (lines \ref{line23} to \ref{line24}) and compare it against the larger set that comes right after (lines \ref{line25} to \ref{line27}). 
    
    For comparison (lines \ref{line28} to \ref{line29}), we use the separator flag at the end. Recall that non-separator elements of the input have $0$ in the separator flag index, and thus, so do the elements from the smaller set collected in the state $\bm{Z}_{\calM}$. When comparing against the elements from the larger set, we simply set the flag to $1$ for an element that is in the intersection of two sets. 

    Now, we examine the space requirement for the state $\bm{Z}_{\calM}$ of the model $\calM$. Note that we only add an element to $\bm{Z}_{\calM}$ in lines \ref{line17} and \ref{line24}. In both cases, the elements are from the smaller set, and thus, $|\bm{Z}_{\calM}| = \min\{|A|, |B|\}$. Moreover, each element in $A$ and $B$ is of size $\sdDim$, and thus, we can conclude that the model $\calM$ with state $\bm{Z}_{\calM}$ can solve the set-disjointness problem with JRT input in $\mathcal{O}(\sdDim \cdot \min\{|A|, |B|\})$.
\end{proof}

\subsection{Based Solving SD}
\label{app:theory-based}
In this section, we will show that \Based\ can solve the set disjointness problem with $\jrt$ inputs. Specifically, this section implements Algorithm \ref{alg:sd-rec} in the \Based\ architecture. Recall here that the \Based\ model combines two layer types: \BaseConv\ (see \cref{def:baseconv}) and \LinAtt\ defined below.
\begin{definition}[Linear Attention with Kernels]
    Given an input sequence $\vu \in \R^{\inputLength \times \modelDim},$ where $\inputLength$ is the sequence length and $\modelDim$ is the model dimension, kernel projections\footnote{By kernel projections of a matrix $\vu \in \R^{m \times n},$ we mean applying some kernel map $\phi: \R^{N \times d} \to \R^{N \times \featureDim}$ to each row of $\vu$.} $\ProjectionN_q, \ProjectionN_k \in \R^{\modelDim \times \featureDim}$, $ \ProjectionN_v \in \R^{\modelDim \times \modelDim}$, where $\featureDim$ is the feature dimension, the \LinAtt\ layer computes the following:
\begin{equation}
\label{eq: linatt}
    \bm{z}^{\LinAtt} := \paren{{\mQ}\ {\mK}^\top} \mV \in \R^{\inputLength \times \modelDim},
\end{equation}
where $\mQ := \ProjectionN_q(\vu), \mK := \ProjectionN_k(\vu), \mV := \ProjectionN_v(\vu)$.
\end{definition}

\subsubsection{SD with \LinAtt}
We first show that with appropriate placement of the two sets, we can solve the set disjointness problem using a class of kernel maps defined below.
\begin{definition}[$IP$-Kernel]
\label{def:ip-kernel}
    We define the {\em IP-Kernel} to be the kernel map $\phi_{\epsilon, \featureDim}: \R^d \to \R^\featureDim$ that takes elements from $[c]$ to $\R^\featureDim$ so that, for any $x,y \in [c]$, we have 
    \[
    {\angles{\phi_{\epsilon, \featureDim}(x),\phi_{\epsilon, \featureDim}(y)}} = 1 \text{ if }x = y\text{ and }\abs{\angles{\phi_{\epsilon, \featureDim}(x),\phi_{\epsilon, \featureDim}(y)}} \le
        \epsilon \text{ otherwise}.
    \]
    That is, an IP-kernel projects elements from the universal set $[c]$ so that the inner products are approximately orthogonal. Note that the feature dimension $\featureDim$ is dependent on the tolerance $\epsilon$.
\end{definition}

We now show that if there exists an IP kernel with small enough $\epsilon$, then it can be used to solve the set-disjointness problem with a {\tt Linear Attention} layer followed by an MLP layer.
\begin{proposition}
\label{prop: lin-att-sd}
    Given an input $\bm{u}\in \R^{N \times d}$ encoding the input $(A, B)$ to the set-disjointness problem (SD) on sets $A, B \subseteq [c]$, there exists a {\tt Linear Attention} (+ MLP) layer with state space $O(df)$ that solves the set disjointness problem for $\vu\in \R^{N \times d}$ with the IP kernel $\phi_{\epsilon, \featureDim}$ applied on $\mQ, \mK$ for $\epsilon = \frac{1}{3|A|}$.\footnote{Our notion of `solves' is a bit non-standard so we clarify it here. If $\vz\in\R^{N \times d}$ is the output then it encodes the result as follows. If the $i$th element in $B$ appears in $A$ then $\vz[|A|+i,:]$ has all entries  in $\brac{\frac 13,1}$, otherwise it is $\bm{0}^d$. If we want a single value as an answer (since SD has a Boolean output) we can apply $O(1)$ \BaseConv\ layers on $\vz$ to sum up all the values in the last $|B|$ rows of $\vz$. Then if $A\cap B\ne\emptyset$ then this value is at least $\frac d3$, otherwise it is $0$.}
\end{proposition}
\begin{proof}
We first define the keys and queries along with the values for the {\tt Linear Attention} layer as follows:
\[
\mQ[i,:] = \mK[i,:] = \phi_{\epsilon, f}(\vu[i,:]) \text{ and } \mV[i,j] := \begin{cases}
    1 & \text{if } i < |A| \\
    0 & \text{otherwise}.
\end{cases}
\]
Note that $\mQ,\mK\in\R^{N\times f}$ and $\mV\in\R^{N\times d}$.
    \begin{align*}
        \paren{{\mQ}\ {\mK}^\top}[i,j] 
        &:= {\mQ}[i,:]{\mK}^\top[:,j] \\
        &= \angles{{\mQ}[i,:],{\mK}[j,:]} \\
        &= \angles{\phi_{\epsilon, f}(\vu[i,:]), \phi_{\epsilon, f}(\vu[j,:])}
    \end{align*}
Next, the key-query product yields the following
    \begin{align*}
        \bm{z}^{\LinAtt}[i,j] 
        &:= \paren{{\mQ}\ {\mK}^\top}[i,:] \mV[:,j] \\
        &= \sum_{k=0}^{N-1} \paren{{\mQ}\ {\mK}^\top}[i,k] \cdot \mV[k,j] \\
        &= \sum_{k=0}^{N-1} \angles{\phi(\vu[i,:]), \phi(\vu[k,:])} \cdot \mV[k,j] \\
        &= \sum_{k < |A|} \angles{\phi_{\epsilon, f}(\vu[i,:]), \phi_{\epsilon, f}(\vu[k,:])} \\
        &=: \rho_i,
    \end{align*}
    where the second-last equality follows from the definition of $\mV$ and we can specify $\rho_i$ as follows:
    \begin{align}
    \rho_i = 1 \pm \epsilon\cdot |A|\text{ if there exists } k \in [0 \cdots |A|-1] \text{ s.t. }\vu[k,:] \equiv \vu[i,:],\text{ and otherwise, }
    \rho_i \le \epsilon|A|.\label{eq: rho}
    \end{align}
For the MLP layer, we define the following parameters (see \Cref{def:mlp} for notation):
\[
\mW^1=\mI_{d\times d},\quad \mB^1_{\mathrm{MLP}} := -\frac{1}{3} \bm{1}_{N \times d},\quad \mW^2_{\mathrm{MLP}} = \mI_{d \times d}, \quad\mB^2_{\mathrm{MLP}} = \bm{0}_{N\times d}
\]
Next, we note that for $0\le \ell<N$ and $0\le j<d$:
\begin{align*}
\bm{y}[\ell, j] &:= \paren{\bm{z}^{\LinAtt}\mW^1_{\mathrm{MLP}} + \mB^1_{\mathrm{MLP}}}[\ell, j]\\
&= \paren{\bm{z}^{\LinAtt} -\frac{1}{3} \bm{1}_{|B| \times d}}[\ell, j]\\
&= \paren{\rho_{\ell} - \frac{1}{3}}. 
\end{align*}
We now use the fact that $\epsilon \le \frac{1}{3|A|}$ to get bounds on the above. To this end, for $0\le \ell<N$, due to \eqref{eq: rho}, if there exists $k \in [0 \cdots |A|-1] \text{ such that }\vu[k,:] \equiv \vu[\ell,:]$, we have
\[
\bm{y}[\ell, j] = \paren{\rho_{\ell} - \frac{1}{3}} := \paren{\paren{1 \pm \epsilon \cdot |A|} -\frac{1}{3}} \in \brac{\frac{2}{3}, \frac{4}{3}} - \frac{1}{3} = \brac{\frac{1}{3}, 1}
\]
Otherwise, if there is no match, then we have
\[
\bm{y}[\ell, j] = \paren{\rho_{\ell} - \frac{1}{3}} \le \epsilon \cdot |A| - \frac{1}{3} \le \frac{1}{3} - \frac{1}{3} \le 0.
\]
We then get the final output as 
\[
\bm{z} := \mathrm{ReLU}(\bm{y})\mW^2_{\mathrm{MLP}} + \mB^2_{\mathrm{MLP}} = \mathrm{ReLU}(\bm{y}),
\]
which reduces to
\[
\bm{z}[\ell, j] \in \brac{\frac{1}{3}, 1}\text{ if there exists } k \in [0 \cdots |A|-1] \text{ such that }\vu[k,:] \equiv \vu[i,:],\text{ and $0$ otherwise}.
\]
Therefore, the last $|B|$ rows of the output $\bm{z}$ will have non-zero values if and only if $A \cap B  \neq \phi$. Finally, the claim on $O(df)$ space follows from the well-known recurrent view of \LinAtt\ (see \eqref{eq:linear_attention_recurrent_3}).\footnote{To incorporate the MLP part, note that as soon as each row of $\vz^\LinAtt$ is generated, we can generate the output of the corresponding row in $\MLP(\vz^\LinAtt)$ with $O(d)$ space by noting that MLP operates independently on each row of its input.}
\end{proof}
\subsubsection{Realization of IP Kernels}
In this section, we will provide some instances of realizing the IP kernels from Definition \ref{def:ip-kernel}.
\paragraph{Exponential  Kernels.}
The first IP-kernel that we define is the exponential kernel $\phi^{\exp}: \R^d \to \R^f$ such that for any $x, y\in [c]$, we have
\[
\angles{\phi(x), \phi(y)} = \exp\paren{\angles{x,y}},
\]
where $x$ and $y$ are encoding of the corresponding elements of $[c]$ in $\{-1,1\}^{d}, d = O(\log(c))$ with large enough distance\footnote{Specifically, we will need to use well-known construction of Binary codes with constant rate and constant relative distance~\cite{guruswami2019essential}.}. If $x = y$, we have
\[
\angles{\phi(x), \phi(y)} = \angles{\phi(x), \phi(x)} 
\]
\[
= \exp\paren{\angles{x,x}} = \exp\paren{\sum_{i \in [d]} x_i^2} = \exp\paren{\sum_{i \in [d]} 1} = \exp(d). 
\]
Next, if $x \neq y$, we instead have
\[
0 < \angles{\phi(x), \phi(y)} = \exp\paren{\angles{x,y}} \le \exp\paren{\gamma\cdot d}
\]
for some $\gamma < 1$ as the code has constant relative distance. Here, we want the match $\exp(d)$ to be large enough. That is, we want
\[
\frac{\exp(d)}{\exp(\gamma\cdot d)} \gg c
\]
So, we want to pick $d$ large enough so that 
\[
(1-\gamma)\cdot d \gg \ln{c}.
\]
\paragraph{Data-Dependent Kernels.}
Here, we define the kernel $\phi$ based on the smaller set $A$. We start by letting $d: = |A| + \log{c}$ so that we define the embeddings as
\begin{equation}
\label{eq: data-dep-kernels}
\begin{aligned}
\phi: [c] &\to \R^{|A|+\log{c}} \\
A \ni a &\mapsto \begin{bmatrix}
    \ve_a & \bm{0}^{\log{c}}
\end{bmatrix} \\
A \not\ni b &\mapsto \begin{bmatrix}\bm{0}^{|A|}& \mB_{b}
\end{bmatrix}    
\end{aligned}
\end{equation}
where $\ve_a \in \{0,1\}^{|A|}$ is the $1$-hot encoding of the element $a$ in $A$ and $\mB_{b}$ is the natural binary encoding in $[c]$ on the element $b$. Using this kernel $\phi$, we achieve orthogonality:
\[
\angles{\phi(x),\phi(y)}=\delta_{xy}.
\]
That is, we have the tolerance $\epsilon = 0$ with feature dimension $f = |A| + \log_2{c}$.

\paragraph{Randomized Kernels.}
We can also define a random kernel map 
\begin{equation}
    \label{eq: rand-kernels}
    \phi: [c] \to \frac{1}{\sqrt{f}}[-1,1]^f.
\end{equation}
That is, for each $x \in [c]$, we pick a random vector in $\{-1,1\}^f$ and normalize it by dividing by $\sqrt{f}$. Here, it is easy to see that for every $x \in [c]$, we have
\[
\angles{\phi(x),\phi(x)} = \frac{1}{f}\sum_{i \in [f]} 1 = 1.
\]
Now, for every $x \neq y$, we can apply known concentration inequalities on Rademacher random variables to get
\[
\Pr\brac{\angles{\phi(x),\phi(y)} > \frac{t}{\sqrt{f}}} \le e^{\frac{-t^2}{2}}.
\]
We then pick $t  = O(\sqrt{\log{c}})$ so that over all $c^2$ pairs, we have
\[
\Pr\brac{ \angles{\phi(x),\phi(y)} > \frac{O(\sqrt{\log{c}})}{\sqrt{f}}} < \frac{1}{100c^2}.
\]
Then with a union bound on all $c^2$ pairs, with high probability, we get that $\phi$ has $\epsilon = \frac{t}{\sqrt{f}}$. We then want the threshold to satisfy the following:
\[
t/\sqrt{f} < \frac{1}{3|A|}  \implies f = \Omega(|A|^2\log{c}).
\]
That is, for $\epsilon = \frac{1}{3|A|}$,  $f = \Theta(\min\{|A|, |B|\}^2\log{c})$ suffices.

\begin{remark}[Practical Justification]
    Empirically, prior works shows a variety of kernels that are competitive with softmax attention quality while using a small amount of space. For instance, Zhang et al. \cite{zhang2024hedgehog} show that either training MLP projections to mimic softmax attention weights or using a $2^{\mathrm{nd}}$-order Taylor approximation to the softmax-exponential function are two effective kernel function choices. The $2^{\mathrm{nd}}$-order polynomial is only a high fidelity approximation within a small band of real values, however empirically results in Arora et al. \cite{arora2024simple} suggest that the normalized query-key dot products often fall within this range, resulting in competitive quality with softmax attention. Arora et al. \cite{arora2024simple}, Chen et al. \cite{chen2021scatterbrain}, and others further suggest that combining efficient sparse plus low-rank attentions (e.g., linear attention plus dense, local sliding window attention) further diminishes quality gaps versus full attention.
\end{remark}

\subsubsection{Shifts with \BaseConv}
Next, we will show that we can use \BaseConv\  layers to move the smaller set to the start of the sequence.
 First, based on whether the smaller set is at the start or not, we need to define separate convolution kernels based on the input. To this end, we use the following \BaseConv\ model to derive these kernels.
\begin{lemma}
\label{lem:bc-ker}
    There exists $\coyoteTuple{2 N}{O(1)}{(n+1)}{\paren{2 N + \frac{N}{2}}}{(n+1)}$ model that takes in a $\jrt$ prompt $\vu^{\jrt} \in \R^{2N \times (\sdDim + 1)}$ of the input $\bm{u} \in \R^{\inputLength \times (\sdDim+1)}$ for the set-disjointness (SD) problem $(A, B) \subseteq \{0,1\}^n$ and outputs the kernel $\vh_{\mathrm{shift}}$ that shifts the input $\vu^{\jrt}$ to get the smaller set at the start of the sequence, where
    \begin{equation}
        \label{eq: shift-kernel}
        \vh_{\mathrm{shift}}(X) := \begin{cases}
            X^{|A|+1} &\text{if } |A| \ge |B| \\
            1 &\text{otherwise.}
        \end{cases}
    \end{equation}
\end{lemma}
\begin{proof}
    Following the proof of Proposition \ref{prop: lin-att-sd}, we know that it suffices to find the location of the separator to determine the location of the smaller set. More specifically, if the separator is within $\brac{0, \frac{N}{2}-1}$ row index range, then we know that the smaller set is at the start, and the kernel being generated is the identity. Otherwise, we generate the kernel $X^{|A|+1}$ which will be used in the proof of \Cref{prop: baseconv-sd-shift}.

    We first increase the inner dimension of the JRT input $\vu^{\jrt} \in \R^{2N \times (n+1)}$ to $\vu^{JRT}_{\mathrm{inner}} \in \R^{\paren{2 N + \frac{N}{2}} \times (n+1)}$ so that we introduce a zero-block between the first seperator and the start of set $B$. That is, we have
    \[
    \vu^{\jrt}_{\mathrm{inner}}[i,:] = \begin{cases}
        \vu^{\jrt}[i,:] &\text{if } i < \frac{N}{2} \\
        \bm{0}^{n+1} &\text{if } \frac{N}{2} \le i < N \\
        \vu^{\jrt}[i-\frac{N}{2},:] &\text{if } i \ge N.
    \end{cases}
    \]
    We can achieve this by simply using the {\em remembering primitive} from \citep[Definition F.15, Proposition F.13]{arora2024simple} using a $\coyoteTuple{\paren{2 N + \frac{N}{2}}}{8}{(n+1)}{\paren{2 N + \frac{N}{2}}}{(n+1)}$ to remember $\vu^{\jrt}[\frac{N}{2}:2N-1,:]$ while applying the identity kernel to preserve $\vu^{\jrt}[0:\frac{N}{2}-1,:]$. 

    We again apply the {\em remembering primitive} from \citep[Definition F.15, Proposition F.13]{arora2024simple} to get
    \[
    \mY \gets {\tt remember}(\vu^{\jrt}_{\mathrm{inner}}, 0, N, f),
    \]
    using \(\coyoteTuple{\paren{2 N + \frac{N}{2}}}{8}{(n+1)}{\paren{2 N + \frac{N}{2}}}{(n+1)}\),
    where $f$ is applied over $\vx:= \vu^{\jrt}_{\mathrm{inner}}[0:N-1,:]$, the first $N$ rows of $\vu^{\jrt}_{\mathrm{inner}}$. That is, we want to remember the last $\paren{N + \frac{N}{2}}$ rows of $\vu^{\jrt}_{\mathrm{inner}}$. We define $f:= f_2 \circ f_1$, where $f_1$ is the cumulative sum of the first $N$ rows computed using  $\coyoteTuple{\inputLength}{O(1)}{(n+1)}{\inputLength}{(n+1)}$ followed by $f_2$ which is the shifting down by $N-1$ using $\coyoteTuple{\inputLength}{3}{(n+1)}{\inputLength}{(n+1)}$~\cite[Propositions F.41 and F.38]{arora2024simple}. That is, for $i \in [0: N-1]$, we have
    \[
    \begin{aligned}
    f_1(\vx)[i,:] &= 
        \sum_{k=0}^{i} \vx[k,:]; \\
        f_2(f_1(\vx))[i,:] &= f_1(\vx)[i-(N-1), :].
    \end{aligned}
    \]
    For the $n$th column, we know that for $0\le i<N$:
    \[
    \vx[i,n] = \vu^{\jrt}_{\mathrm{inner}}[i,n] = \vu^{\jrt}[i,n] = \begin{cases} 
        1 &\text{if }|A| \le |B|\text{ and } i = |A| \\
        0 &\text{otherwise.}
    \end{cases}
    \]
    This is because if $|A| \le |B|$, the separator is within $\brac{0, \frac{N}{2}-1}$ and its $n$th bit is $1$, where $|A| =: i_s \in \brac{0, \frac{N}{2}-1}$ to be the location of the separator. We then get
    \[
    \begin{aligned}
    f_1(\vx)[i,n] &= \begin{cases}
        1 &\text{if }|A| \le |B|\text{ and } i \ge i_s \\
        0 &\text{otherwise.}
    \end{cases} \\
        f_2(f_1(\vx))[i,n] &= \begin{cases}
        1 &\text{if }|A| \le |B|\text{ and } i = 0 \\
        0 &\text{otherwise.}
    \end{cases}
    \end{aligned}
    \]
    We can thus characterize the $n$th column of the output $\mY \in \R^{\paren{2 N + \frac{N}{2}} \times (n+1)}$ as follows:
    \[
    \mY[i,n] = \begin{cases}
        1 &\text{if } |A| \le |B|\text{ and } i = 0 \\
        0 &\text{if } |A| > |B|\text{ and } i = 0 \text{ or } 1 \le i < N \\
        \vu^{\jrt}[i+\frac{N}{2},n] &\text{if } i \ge N.
    \end{cases}
    \]
    We now remember $\mY[0: \frac{N}{2}-1,:]$ while shifting down $\mY[\frac{N}{2}:2 N + \frac{N}{2}-1,:]$ by $\frac{N}{2}-1$~\cite[Proposition F.13 and F.38]{arora2024simple} to get $\mY'$ such that:
    \[
    \begin{aligned}
    \mY'[i,:] &= \begin{cases}
        \mY[i,:] &\text{if }i<\frac{N}{2} \\
        \mY[i-\frac{N}{2},:] &\text{if } i \ge \frac{N}{2}
    \end{cases} \\
    &= \begin{cases}
        \mY[i,:] &\text{if }i<\frac{N}{2} \\
        \vu^{\jrt}[i,:] &\text{if } \frac{N}{2}\le i < 2N-1\\
        \bm{0}^n &\text{otherwise.}
    \end{cases}
    \end{aligned}
    \]
    
    Focusing on the $n$th column, we see that we get for $0\le i<N$:
    \[
    \begin{aligned}
    \mY'[i,n] &= \begin{cases}
        1 &\text{if } |A| \le |B|\text{ and } i = 0 \text{ or } |A| > |B| \text{ and } i = |A|\\
        0 &\text{otherwise}
    \end{cases}.    
    \end{aligned}
    \]
    Or equivalently
    \[
    \begin{aligned}
    \mY'[0:N-1, n] &= \begin{cases}
        \ve_0 &\text{if }|A| \le |B| \\
        \ve_{|A|} &\text{if } |A| > |B|.
    \end{cases}
    \end{aligned},
    \]
    which is exactly what we need as the shift kernel $\vh_{\mathrm{shift}}$. A schematic representation of this process is provided in \Cref{fig:sch-sd-shift}. The final claim on the overall parameters follows from the fact that we can `stack' \BaseConv\ layers with the same internal dimension~\cite{arora2024simple}.
\begin{figure*}[!h]
    \centering
\tikzset{every picture/.style={line width=0.75pt}} 

\begin{tikzpicture}[x=0.75pt,y=0.75pt,yscale=-1,xscale=1]

\draw   (84,58.5) -- (534,58.5) -- (534,90) -- (84,90) -- cycle ;
\draw [color={rgb, 255:red, 155; green, 155; blue, 155 }  ,draw opacity=1 ] [dash pattern={on 4.5pt off 4.5pt}]  (189.5,53) -- (189.5,96) ;
\draw    (128.5,58.5) -- (128.5,90) ;
\draw    (153.5,58) -- (153.5,89.5) ;
\draw    (237,58.5) -- (237,90) ;
\draw    (262,58) -- (262,89.5) ;
\draw [color={rgb, 255:red, 155; green, 155; blue, 155 }  ,draw opacity=1 ] [dash pattern={on 4.5pt off 4.5pt}]  (318.5,51.5) -- (318.5,94.5) ;
\draw   (83,184) -- (626,184) -- (626,215.5) -- (83,215.5) -- cycle ;
\draw [color={rgb, 255:red, 139; green, 87; blue, 42 }  ,draw opacity=1 ]   (193,90) .. controls (195.33,89.69) and (196.66,90.7) .. (196.98,93.03) .. controls (197.3,95.36) and (198.63,96.37) .. (200.96,96.06) .. controls (203.29,95.74) and (204.62,96.75) .. (204.94,99.08) .. controls (205.25,101.41) and (206.58,102.42) .. (208.91,102.11) .. controls (211.24,101.8) and (212.57,102.81) .. (212.89,105.14) .. controls (213.21,107.47) and (214.54,108.48) .. (216.87,108.17) .. controls (219.2,107.86) and (220.53,108.87) .. (220.85,111.2) .. controls (221.17,113.53) and (222.5,114.54) .. (224.83,114.23) .. controls (227.16,113.91) and (228.49,114.92) .. (228.81,117.25) .. controls (229.13,119.58) and (230.46,120.59) .. (232.79,120.28) .. controls (235.12,119.97) and (236.45,120.98) .. (236.77,123.31) .. controls (237.08,125.64) and (238.41,126.65) .. (240.74,126.34) .. controls (243.07,126.03) and (244.4,127.04) .. (244.72,129.37) .. controls (245.04,131.7) and (246.37,132.71) .. (248.7,132.4) .. controls (251.03,132.08) and (252.36,133.09) .. (252.68,135.42) .. controls (253,137.75) and (254.33,138.76) .. (256.66,138.45) .. controls (258.99,138.14) and (260.32,139.15) .. (260.64,141.48) .. controls (260.96,143.81) and (262.29,144.82) .. (264.62,144.51) .. controls (266.95,144.2) and (268.28,145.21) .. (268.59,147.54) .. controls (268.91,149.87) and (270.24,150.88) .. (272.57,150.57) .. controls (274.9,150.25) and (276.23,151.26) .. (276.55,153.59) .. controls (276.87,155.92) and (278.2,156.93) .. (280.53,156.62) .. controls (282.86,156.31) and (284.19,157.32) .. (284.51,159.65) .. controls (284.83,161.98) and (286.16,162.99) .. (288.49,162.68) .. controls (290.82,162.37) and (292.15,163.38) .. (292.47,165.71) .. controls (292.78,168.04) and (294.11,169.05) .. (296.44,168.74) .. controls (298.77,168.42) and (300.1,169.43) .. (300.42,171.76) .. controls (300.74,174.09) and (302.07,175.1) .. (304.4,174.79) .. controls (306.73,174.48) and (308.06,175.49) .. (308.38,177.82) .. controls (308.7,180.15) and (310.03,181.16) .. (312.36,180.85) .. controls (314.69,180.54) and (316.02,181.55) .. (316.34,183.88) -- (316.5,184) -- (316.5,184) ;
\draw    (127.5,184) -- (127.5,215.5) ;
\draw    (152.5,183.5) -- (152.5,215) ;
\draw    (386,184) -- (386,215.5) ;
\draw    (411,183.5) -- (411,215) ;
\draw [color={rgb, 255:red, 155; green, 155; blue, 155 }  ,draw opacity=1 ] [dash pattern={on 4.5pt off 4.5pt}]  (317.5,177) -- (317.5,220) ;
\draw [color={rgb, 255:red, 139; green, 87; blue, 42 }  ,draw opacity=1 ]   (84,90) .. controls (85.65,91.68) and (85.63,93.35) .. (83.95,95) .. controls (82.26,96.65) and (82.24,98.31) .. (83.89,100) .. controls (85.54,101.68) and (85.52,103.35) .. (83.84,105) .. controls (82.16,106.65) and (82.14,108.32) .. (83.79,110) .. controls (85.44,111.69) and (85.42,113.35) .. (83.73,115) .. controls (82.05,116.65) and (82.03,118.32) .. (83.68,120) .. controls (85.33,121.68) and (85.31,123.35) .. (83.63,125) .. controls (81.94,126.65) and (81.92,128.31) .. (83.57,130) .. controls (85.22,131.68) and (85.2,133.35) .. (83.52,135) .. controls (81.84,136.65) and (81.82,138.32) .. (83.47,140) .. controls (85.12,141.69) and (85.1,143.35) .. (83.41,145) .. controls (81.73,146.65) and (81.71,148.32) .. (83.36,150) .. controls (85.01,151.68) and (84.99,153.35) .. (83.31,155) .. controls (81.63,156.65) and (81.61,158.32) .. (83.26,160) .. controls (84.91,161.69) and (84.89,163.35) .. (83.2,165) .. controls (81.52,166.65) and (81.5,168.32) .. (83.15,170) .. controls (84.8,171.68) and (84.78,173.35) .. (83.1,175) .. controls (81.41,176.64) and (81.39,178.3) .. (83.04,179.99) -- (83,184) -- (83,184) ;
\draw  [fill={rgb, 255:red, 255; green, 150; blue, 161 }  ,fill opacity=1 ] (82.5,282) -- (625.5,282) -- (625.5,313.5) -- (82.5,313.5) -- cycle ;
\draw    (385.5,282) -- (385.5,313.5) ;
\draw    (410.5,281.5) -- (410.5,313) ;
\draw [color={rgb, 255:red, 155; green, 155; blue, 155 }  ,draw opacity=1 ] [dash pattern={on 4.5pt off 4.5pt}]  (317,275) -- (317,318) ;
\draw  [fill={rgb, 255:red, 201; green, 201; blue, 201 }  ,fill opacity=1 ] (317,274.7) .. controls (317,266.31) and (323.81,259.5) .. (332.2,259.5) -- (620.8,259.5) .. controls (629.19,259.5) and (636,266.31) .. (636,274.7) -- (636,320.3) .. controls (636,328.69) and (629.19,335.5) .. (620.8,335.5) -- (332.2,335.5) .. controls (323.81,335.5) and (317,328.69) .. (317,320.3) -- cycle ;
\draw   (83.5,401) -- (626.5,401) -- (626.5,432.5) -- (83.5,432.5) -- cycle ;
\draw    (108,401.5) -- (108,433) ;
\draw    (386.5,401) -- (386.5,432.5) ;
\draw    (411.5,400.5) -- (411.5,432) ;
\draw [color={rgb, 255:red, 155; green, 155; blue, 155 }  ,draw opacity=1 ] [dash pattern={on 4.5pt off 4.5pt}]  (318,394) -- (318,437) ;
\draw   (82.5,526.5) -- (532.5,526.5) -- (532.5,558) -- (82.5,558) -- cycle ;
\draw [color={rgb, 255:red, 155; green, 155; blue, 155 }  ,draw opacity=1 ] [dash pattern={on 4.5pt off 4.5pt}]  (190.5,516) -- (190.5,559) ;
\draw    (107,526.5) -- (107,558) ;
\draw    (238,526) -- (238,557.5) ;
\draw    (264,527) -- (264,558.5) ;
\draw [color={rgb, 255:red, 155; green, 155; blue, 155 }  ,draw opacity=1 ] [dash pattern={on 4.5pt off 4.5pt}]  (319,520.5) -- (319,563.5) ;
\draw [color={rgb, 255:red, 139; green, 87; blue, 42 }  ,draw opacity=1 ]   (319,433.5) .. controls (318.63,435.83) and (317.28,436.81) .. (314.95,436.44) .. controls (312.62,436.07) and (311.28,437.04) .. (310.91,439.37) .. controls (310.54,441.7) and (309.19,442.68) .. (306.86,442.31) .. controls (304.53,441.94) and (303.18,442.91) .. (302.81,445.24) .. controls (302.44,447.57) and (301.09,448.55) .. (298.76,448.18) .. controls (296.43,447.81) and (295.09,448.79) .. (294.72,451.12) .. controls (294.35,453.45) and (293,454.42) .. (290.67,454.05) .. controls (288.34,453.68) and (286.99,454.66) .. (286.62,456.99) .. controls (286.25,459.32) and (284.91,460.3) .. (282.58,459.93) .. controls (280.25,459.56) and (278.9,460.53) .. (278.53,462.86) .. controls (278.16,465.19) and (276.81,466.17) .. (274.48,465.8) .. controls (272.15,465.43) and (270.8,466.4) .. (270.43,468.73) .. controls (270.06,471.06) and (268.72,472.04) .. (266.39,471.67) .. controls (264.06,471.3) and (262.71,472.28) .. (262.34,474.61) .. controls (261.97,476.94) and (260.62,477.91) .. (258.29,477.54) .. controls (255.96,477.17) and (254.62,478.15) .. (254.25,480.48) .. controls (253.88,482.81) and (252.53,483.78) .. (250.2,483.41) .. controls (247.87,483.04) and (246.52,484.02) .. (246.15,486.35) .. controls (245.78,488.68) and (244.43,489.66) .. (242.1,489.29) .. controls (239.77,488.92) and (238.43,489.89) .. (238.06,492.22) .. controls (237.69,494.55) and (236.34,495.53) .. (234.01,495.16) .. controls (231.68,494.79) and (230.33,495.77) .. (229.96,498.1) .. controls (229.59,500.43) and (228.25,501.4) .. (225.92,501.03) .. controls (223.59,500.66) and (222.24,501.64) .. (221.87,503.97) .. controls (221.5,506.3) and (220.15,507.27) .. (217.82,506.9) .. controls (215.49,506.53) and (214.15,507.51) .. (213.78,509.84) .. controls (213.41,512.17) and (212.06,513.15) .. (209.73,512.78) .. controls (207.4,512.41) and (206.05,513.38) .. (205.68,515.71) .. controls (205.31,518.04) and (203.96,519.02) .. (201.63,518.65) .. controls (199.3,518.28) and (197.96,519.25) .. (197.59,521.58) .. controls (197.22,523.91) and (195.87,524.89) .. (193.54,524.52) -- (191.5,526) -- (191.5,526) ;
\draw [color={rgb, 255:red, 139; green, 87; blue, 42 }  ,draw opacity=1 ]   (83.5,432.5) .. controls (85.15,434.18) and (85.13,435.85) .. (83.45,437.5) .. controls (81.76,439.15) and (81.74,440.81) .. (83.39,442.5) .. controls (85.04,444.18) and (85.02,445.85) .. (83.34,447.5) .. controls (81.66,449.15) and (81.64,450.82) .. (83.29,452.5) .. controls (84.94,454.19) and (84.92,455.85) .. (83.23,457.5) .. controls (81.55,459.15) and (81.53,460.82) .. (83.18,462.5) .. controls (84.83,464.18) and (84.81,465.85) .. (83.13,467.5) .. controls (81.44,469.15) and (81.42,470.81) .. (83.07,472.5) .. controls (84.72,474.18) and (84.7,475.85) .. (83.02,477.5) .. controls (81.34,479.15) and (81.32,480.82) .. (82.97,482.5) .. controls (84.62,484.19) and (84.6,485.85) .. (82.91,487.5) .. controls (81.23,489.15) and (81.21,490.82) .. (82.86,492.5) .. controls (84.51,494.18) and (84.49,495.85) .. (82.81,497.5) .. controls (81.13,499.15) and (81.11,500.82) .. (82.76,502.5) .. controls (84.41,504.19) and (84.39,505.85) .. (82.7,507.5) .. controls (81.02,509.15) and (81,510.82) .. (82.65,512.5) .. controls (84.3,514.18) and (84.28,515.85) .. (82.6,517.5) .. controls (80.91,519.14) and (80.89,520.8) .. (82.54,522.49) -- (82.5,526.5) -- (82.5,526.5) ;

\draw (178,101.4) node [anchor=north west][inner sep=0.75pt]  [font=\tiny,color={rgb, 255:red, 155; green, 155; blue, 155 }  ,opacity=1 ]  {$\frac{N}{2} -1$};
\draw (307.5,103.4) node [anchor=north west][inner sep=0.75pt]  [font=\tiny,color={rgb, 255:red, 155; green, 155; blue, 155 }  ,opacity=1 ]  {$N-1$};
\draw (130.5,61.9) node [anchor=north west][inner sep=0.75pt]  [font=\scriptsize,color={rgb, 255:red, 255; green, 0; blue, 0 }  ,opacity=1 ]  {$1$};
\draw (239,61.9) node [anchor=north west][inner sep=0.75pt]  [font=\scriptsize,color={rgb, 255:red, 0; green, 120; blue, 255 }  ,opacity=1 ]  {$1$};
\draw (116.5,43.9) node [anchor=north west][inner sep=0.75pt]  [font=\scriptsize,color={rgb, 255:red, 255; green, 0; blue, 0 }  ,opacity=1 ]  {$|A|\ \leq \ |B|$};
\draw (224,42.9) node [anchor=north west][inner sep=0.75pt]  [font=\scriptsize,color={rgb, 255:red, 0; green, 118; blue, 255 }  ,opacity=1 ]  {$|A|\ \geq \ |B|$};
\draw (251.5,74.9) node [anchor=north west][inner sep=0.75pt]  [font=\scriptsize,color={rgb, 255:red, 255; green, 0; blue, 0 }  ,opacity=1 ]  {$0$};
\draw (141.5,76.4) node [anchor=north west][inner sep=0.75pt]  [font=\scriptsize,color={rgb, 255:red, 0; green, 120; blue, 255 }  ,opacity=1 ]  {$0$};
\draw (306.5,228.9) node [anchor=north west][inner sep=0.75pt]  [font=\tiny,color={rgb, 255:red, 155; green, 155; blue, 155 }  ,opacity=1 ]  {$N-1$};
\draw (129.5,187.4) node [anchor=north west][inner sep=0.75pt]  [font=\scriptsize,color={rgb, 255:red, 255; green, 0; blue, 0 }  ,opacity=1 ]  {$1$};
\draw (388,187.4) node [anchor=north west][inner sep=0.75pt]  [font=\scriptsize,color={rgb, 255:red, 0; green, 120; blue, 255 }  ,opacity=1 ]  {$1$};
\draw (115.5,169.4) node [anchor=north west][inner sep=0.75pt]  [font=\scriptsize,color={rgb, 255:red, 255; green, 0; blue, 0 }  ,opacity=1 ]  {$|A|\ \leq \ |B|$};
\draw (372.5,167.9) node [anchor=north west][inner sep=0.75pt]  [font=\scriptsize,color={rgb, 255:red, 0; green, 118; blue, 255 }  ,opacity=1 ]  {$|A|\ \geq \ |B|$};
\draw (400.5,200.4) node [anchor=north west][inner sep=0.75pt]  [font=\scriptsize,color={rgb, 255:red, 255; green, 0; blue, 0 }  ,opacity=1 ]  {$0$};
\draw (140.5,201.9) node [anchor=north west][inner sep=0.75pt]  [font=\scriptsize,color={rgb, 255:red, 0; green, 120; blue, 255 }  ,opacity=1 ]  {$0$};
\draw (309,340.4) node [anchor=north west][inner sep=0.75pt]  [font=\tiny,color={rgb, 255:red, 155; green, 155; blue, 155 }  ,opacity=1 ]  {$N-1$};
\draw (110.5,293) node [anchor=north west][inner sep=0.75pt]  [color={rgb, 255:red, 252; green, 252; blue, 252 }  ,opacity=1 ] [align=left] {\tt cumulative\_sum → shift\_down};
\draw (455.5,293) node [anchor=north west][inner sep=0.75pt]  [color={rgb, 255:red, 82; green, 82; blue, 82 }  ,opacity=1 ] [align=left] {\tt remember};
\draw (307,445.9) node [anchor=north west][inner sep=0.75pt]  [font=\tiny,color={rgb, 255:red, 155; green, 155; blue, 155 }  ,opacity=1 ]  {$N-1$};
\draw (87,408.4) node [anchor=north west][inner sep=0.75pt]  [font=\scriptsize,color={rgb, 255:red, 255; green, 0; blue, 0 }  ,opacity=1 ]  {$1$};
\draw (388.5,404.4) node [anchor=north west][inner sep=0.75pt]  [font=\scriptsize,color={rgb, 255:red, 0; green, 120; blue, 255 }  ,opacity=1 ]  {$1$};
\draw (80.5,385.9) node [anchor=north west][inner sep=0.75pt]  [font=\scriptsize,color={rgb, 255:red, 255; green, 0; blue, 0 }  ,opacity=1 ]  {$|A|\ \leq \ |B|$};
\draw (373,384.9) node [anchor=north west][inner sep=0.75pt]  [font=\scriptsize,color={rgb, 255:red, 0; green, 118; blue, 255 }  ,opacity=1 ]  {$|A|\ \geq \ |B|$};
\draw (401,417.4) node [anchor=north west][inner sep=0.75pt]  [font=\scriptsize,color={rgb, 255:red, 255; green, 0; blue, 0 }  ,opacity=1 ]  {$0$};
\draw (96,419.9) node [anchor=north west][inner sep=0.75pt]  [font=\scriptsize,color={rgb, 255:red, 0; green, 120; blue, 255 }  ,opacity=1 ]  {$0$};
\draw (179,564.4) node [anchor=north west][inner sep=0.75pt]  [font=\tiny,color={rgb, 255:red, 155; green, 155; blue, 155 }  ,opacity=1 ]  {$\frac{N}{2} -1$};
\draw (308.5,566.4) node [anchor=north west][inner sep=0.75pt]  [font=\tiny,color={rgb, 255:red, 155; green, 155; blue, 155 }  ,opacity=1 ]  {$N-1$};
\draw (84.5,529.9) node [anchor=north west][inner sep=0.75pt]  [font=\scriptsize,color={rgb, 255:red, 255; green, 0; blue, 0 }  ,opacity=1 ]  {$1$};
\draw (242.5,530.4) node [anchor=north west][inner sep=0.75pt]  [font=\scriptsize,color={rgb, 255:red, 0; green, 120; blue, 255 }  ,opacity=1 ]  {$1$};
\draw (86,510.4) node [anchor=north west][inner sep=0.75pt]  [font=\scriptsize,color={rgb, 255:red, 255; green, 0; blue, 0 }  ,opacity=1 ]  {$|A|\ \leq \ |B|$};
\draw (225,505.9) node [anchor=north west][inner sep=0.75pt]  [font=\scriptsize,color={rgb, 255:red, 0; green, 118; blue, 255 }  ,opacity=1 ]  {$|A|\ \geq \ |B|$};
\draw (252,542.9) node [anchor=north west][inner sep=0.75pt]  [font=\scriptsize,color={rgb, 255:red, 255; green, 0; blue, 0 }  ,opacity=1 ]  {$0$};
\draw (96.5,544.9) node [anchor=north west][inner sep=0.75pt]  [font=\scriptsize,color={rgb, 255:red, 0; green, 120; blue, 255 }  ,opacity=1 ]  {$0$};

\end{tikzpicture}
        \caption{Schema for getting input-dependent shift kernels for the set disjointness (SD) problem.}
    \label{fig:sch-sd-shift}
\end{figure*}
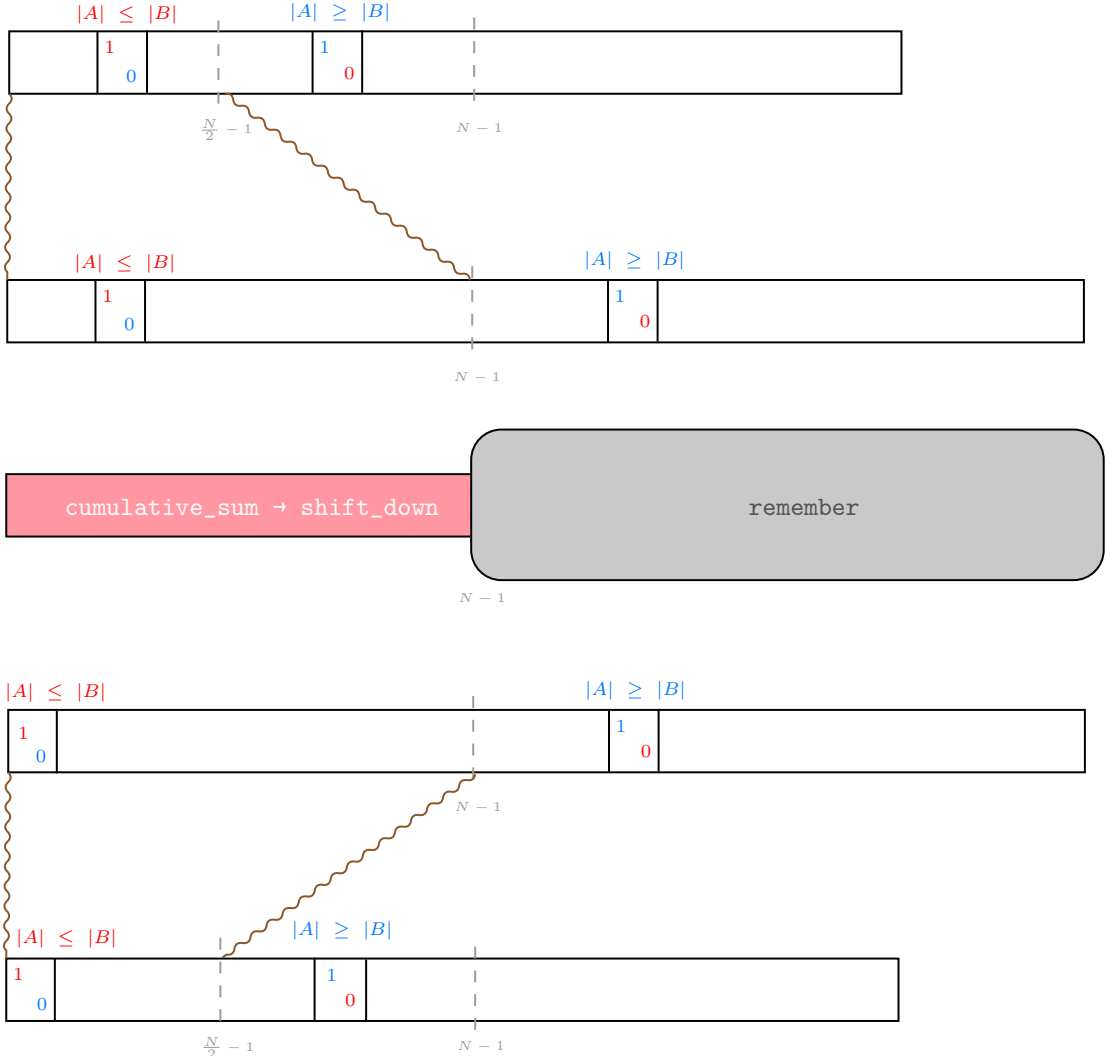

\end{proof}
We now use the kernels from Lemma \ref{lem:bc-ker} to do the appropriate shift.
\begin{proposition}
\label{prop: baseconv-sd-shift}
    Given a $\jrt$ prompt $\vu^{\jrt} \in \R^{2N \times (\sdDim + 1)}$ of the input $\bm{u} \in \R^{\inputLength \times (\sdDim+1)}$ for the set-disjointness (SD) problem $(A, B) \subseteq \{0,1\}^n$, there exist $O(1)$ {\em input-dependent} \BaseConv\ layers that can rearrange the input so that the smaller set out of $A$ and $B$ is placed at the start of the sequence.
\end{proposition}
\begin{proof}
The input $\bm{u} \in \{0,1\}^{\inputLength \times (\sdDim+1)}$ is formatted as in Remark \ref{remark:sd-input}. In the first case where $A$ is the smaller set, we do not need to change the input. Let $\vs:= [\bm{0}^n::1]$ be the separator, then we want:
\[
\vu^{\jrt} \equiv \begin{bmatrix}
\longleftarrow
    A
    \longrightarrow &
    \vs &
\longleftarrow
    B
    \longrightarrow  &
    \vline &
\longleftarrow
    A
    \longrightarrow &
    \vs &
\longleftarrow
    B
    \longrightarrow
\end{bmatrix}
\]
Otherwise, if $|B| \le |A|$, we want to shift the input so that $|B|$ comes at the start of the input sequence in the $\jrt$ prompt $\vu^{\jrt}$. To this end, 
we want to add a separator between after the first copy of the input ends. For this purpose, we can keep the first copy as is, and operate on the duplicate $A$ by shifting it down by $1$ and adding a separator at the start of the second input. We thus apply the {\tt remember}($\vu^{\jrt}_{{\tt shift\_up}},N,N+|A|, f$) primitive \cite[Definition F.15]{arora2024simple}  with $8$ layers of \BaseConv\, where $f$ is any function that maps $(A, \vs) \mapsto (\vs, A)$, so that we get
\[
\vu^{\jrt} \equiv \begin{bmatrix}
\longleftarrow
    A
    \longrightarrow &
    \vs &
\longleftarrow
    B
    \longrightarrow  &
    \vline &
    \vs &
\longleftarrow
    A
    \longrightarrow &
\longleftarrow
    B
    \longrightarrow
\end{bmatrix}
\]
Next, we shift up using the {\tt shift\_up}($\vu^{\jrt}, |A|+1$) primitive \cite[Proposition C.5]{arora2023zoology} for \BaseConv\ with $3$ layers by implementing the kernel $\vh_{\mathrm{shift}}$ from Lemma \ref{lem:bc-ker}. We then get
\[
\vu^{\jrt}_{{\tt shift\_up}} \equiv \begin{bmatrix}
\longleftarrow
    B
    \longrightarrow  &
    \vs &
\longleftarrow
    A
    \longrightarrow &
    \vline &
\longleftarrow
    B
    \longrightarrow & 
    \bm{0}^{|A|+1}
\end{bmatrix}
\]
That is, in both cases, the final output has the smaller set out of $A$ and $B$ at the start of the sequence. 

To complete the proof we note that we can do the above in one single model (that uses data dependent convolutions): (1) We add the extra separator after the second set in $\vu^{\jrt}$ and (2) we do the using the convolution operator in \BaseConv\ where we use the convolution kernel computed from \Cref{lem:bc-ker}.\footnote{We also need to only keep the first $N$ rows of the matrix, which we can obtain by zeroing out all the remaining rows using another \BaseConv\ layer.}
\end{proof}
Finally, we can combine Propositions \ref{prop: lin-att-sd} and \ref{prop: baseconv-sd-shift} to claim that \Based\ can solve $\sd$ with JRT-prompting in space $O(\min\{|A|, |B|\})$.
\begin{theorem}
    \label{thm: based-sd}
    Given a $\jrt$ prompt $\vu^{\jrt} \in \R^{2N \times (\sdDim + 1)}$ of the input $\bm{u} \in \R^{\inputLength \times (\sdDim+1)}$ for the set-disjointness (SD) problem $(A, B)$, there exists a (data dependent) \Based\ model (\BaseConv\ + MLP + \LinAtt\ + MLP)\footnote{This matches the architecture in our experiments.} that solves the SD problem with space $O(\min\{|A|, |B|\} \cdot n)$.
\end{theorem}
\begin{proof}
    First, we use the \BaseConv\ layers from \Cref{prop: baseconv-sd-shift} to get the smaller set of $A$ and $B$ in $\vu^{\jrt}$ to the start of the sequence in $\bm{z}^{\BaseConv}$. Next, we reduce  $\bm{z}^{\BaseConv}$ using an MLP layer to get $\bm{z}^{\BaseConv}[0:N-1, :]$ as the input to the \LinAtt\ (+MLP) layer in \Cref{prop: lin-att-sd} so that we solve the SD problem for the original input $\bm{u}$. Finally, for the $\LinAtt$ layer, we can use the data-dependent IP kernels from \eqref{eq: data-dep-kernels} to get $f = O(\min\{|A|, |B|\}))$, which yields the claimed space usage since we have $d=n$.
\end{proof}
\begin{remark}
    \label{rem: based-sd-2}
    We note that we can use the random kernels from \eqref{eq: rand-kernels} in Theorem \ref{thm: based-sd} to get space usage of $O\paren{(\min\{|A|, |B|\})^2\cdot n}$ {\em without} using data-dependent IP kernels.
\end{remark}
\subsection{$\nmqar$ and $\sd$}
In this section, we introduce the {\em general associative recall} $\nmqar$ problem. Recall that the query in the AR problem comes at the end, and thus, the query is compared with all the keys in the input. On the other hand, in MQAR, a query at position $i$ is only compared with keys at positions $j < i$. Moreover, the number of keys and queries in the input are the same for MQAR. Instead, we introduce the following alternate generalization of AR that has all the queries at the end with the number of queries different from the number of keys.
\begin{definition}[$\nmqar$]
\label{def: nmqar}
We are given an input sequence 
    \begin{equation}
    \label{eq: nmqar-input}
    \bm{u}[0 \cdots N-1] \triangleq (\vk_0,\vv_0),\ldots,(\vk_{n-1},\vv_{n-1});\vq_0,\ldots,\vq_{m-1},
    \end{equation}
where $K:= \{\vk_i\}_{i=0}^{n-1}, V:= \{\vv_i\}_{i=0}^{n-1},$ and $Q:= \{\vq_i\}_{i=0}^{m-1}$, with each $\vk_i, \vv_i, \vq_i \in C$ is a token drawn from a vocabulary of size $c = \abs{C}$, and we have $N = 2n+m$.

Our goal in the {\em general associative recall} ($\nmqar$) problem is to check, for each $\bm{q}_i \in Q$, whether there exists $\bm{k}_j \in K$ such that $\bm{q}_i \equiv \bm{k_j}$; if so, output the corresponding value $\vv_j$, and otherwise, output ${\tt Null}$.
\end{definition}

We will first show that SD reduces to $\nmqar$.
\begin{proposition}
\label{prop: nmqar-sd}
    Any algorithm $\calA$ solving $\nmqar$ can also solve $\sd$.
\end{proposition}
\begin{proof}
Given an input to the set-disjointness problem $(A, B)$ with $A:= \{A_0, \ldots, A_{|A|-1}\}, B:= \{B_0, \ldots, B_{|B|-1}\}$, we can construct the following input to the $\nmqar$ problem:
\[
    \bm{u} := (A_0, A_0), \ldots, (A_{|A|-1}, A_{|A|-1});B_0, \ldots, B_{|B|-1}.
\]
Now, we run algorithm $\calA$ on $\bm{u}$, and if for all $q \in Q$, we get ${\tt Null}$, then we know $A \cap B = \emptyset$, and otherwise, $A\cap B \neq \emptyset$. This solves the set disjointness (SD) problem.
\end{proof}

What we have shown is that $\nmqar$ is much more general compared to $\sd$. However, we can also show that we can solve $\nmqar$ under certain conditions if we had access to an algorithm solving $\sd$.
\begin{proposition}
\label{prop: sd-gar}
Let $\calA_{\sd}$ be an algorithm solving the set disjointness ($\sd$) problem. Then, for a vocabulary $\calC$ with $\abs{\calC} = c$ with values from $[c]$ and represented as $v_j\in \{0,1\}^d$ where $d=\ceil{\log_2(c+1)}$ with at most one match for each query, we can solve the $\nmqar$ problem (\cref{def: nmqar}) with $d$ calls to $\calA_{\sd}$.
\end{proposition}
\begin{proof}
    Given an input $(\vk_0,\vv_0),\ldots,(\vk_{n-1},\vv_{n-1});\vq_0,\ldots,\vq_{m-1}$ to $\nmqar$, for each call $\ell \in [d]$ to algorithm $\calA_{\sd}$, we construct the inputs to algorithm $\calA$ by taking $A:= Q, B:= K_{\ell}$ with $K_{\ell}$ defined as follows:
    \begin{equation}
    \label{eq: l-call}
    k_j\in K_\ell \iff v_j[\ell]=1.
    \end{equation}
    That is, we include $k_j \in K_\ell$ iff the $\ell$'th bit of $v_j$ is 1.

    We now claim that we can solve the MQAR problem given $Q\cap K_\ell$ for all $\ell\in [d]$. To see this, note that if a query $q\in Q$ is not in $K$,  then $q \notin Q\cap K_\ell$ for every $\ell \in [d]$. We thus output ${\tt Null}$ for these queries.
    
    Otherwise, if $q\in Q\cap K$, then there exists a non-empty set of calls $L \subseteq [d]$ such that $q\in Q\cap K_\ell$ for all $\ell\in L$. We can then extract the $\ell$'th bit of $v_j$, where $q=k_j$. That is, for $q = k_j$, we use \eqref{eq: l-call} to get
    \[
    v_j[\ell] = \begin{cases}
        1 & \text{if }\ell \in L\\
        0 & \text{otherwise.}
    \end{cases}
    \]
    
    This is exactly the value corresponding to the unique matching key $k_j$ for the query $q$. 
\end{proof}

\subsubsection{Lower Bound for $\nmqar$ via $\sd$}
\label{app: jrp-gar}
In this section, we present a lower bound for solving $\nmqar$. For this purpose, we require the following two-way randomized communication complexity\footnote{Here, in contrast to one-way randomized communication protocol in \cref{sec: lower-bound-AR-layers}, both Alice and Bob are allowed to send messages to each other.} lower bound for set-disjointness ($\sd$). 
\begin{theorem}[\citep{haastad2007randomized}\footnote{\citep{haastad2007randomized} provides a lower bound of $n$ for $|A| = |B|$. However, we can extend it to Theorem \ref{thm: sd-lb} by reducing the $\min\{|A|, |B|\}$ subset to the equal sized set by picking a hard distribution where both sets are of size $\min\{|A|, |B|\}$ and then adding "extra" elements to only one of them to get a larger set (i.e., one can increase the universe size by these extra elements to get the desired lower bound). 
}]
\label{thm: sd-lb}
    The two-way randomized communication complexity of the set disjointness problem with sets $A, B \subseteq [n]$ is $\Omega(\min\{|A|, |B|\})$ bits for $n \ge o(\min\{|A|, |B|\})$.
\end{theorem}

\begin{definition}[$\jrp$ Prompts]
\label{def:jrp}
For any model $\calM$ with input $\bm{u} \in \R^{\inputLength \times \modelDim}$, a {\em $\jrp$ prompt} for input $\bm{u}$ is the $p$-times repeated input $\bm{u}^{\jrp} \in {\R^{p\inputLength \times \modelDim}}$ given by
\[
\bm{u}^{\jrp}[i,:] := \bm{u}[i \mod{N}, :]
\]
\end{definition}

\begin{proposition}
\label{prop: jrp-sd}
    Given a {\em $\jrp$ prompt} $\vu^{\jrp} \in \{0,1\}^{p\inputLength \times \modelDim}$ for input $\vu \in \{0,1\}^{\inputLength \times \modelDim}$ to the $\nmqar$ problem, any recurrent model $\calM_{\nmqar}$ (\cref{def: reg-model}) solving $\nmqar$ requires $\max_i \abs{\mZ_{\calM_{\nmqar}}^i}$ to be at least $\Omega\paren{\frac{\min\{|A|, |B|\}}{p}}$-bits.
\end{proposition}
\begin{proof}
We first take the input $\bm{u} \in \{0,1\}^{N \times d}$ to the $\nmqar$ problem and design a two-way communication protocol for solving ${\nmqar}$ given access to the reccurrent model $\calM_{\nmqar}$. To this end, Alice with their access of key-value part generates her part of the input: 
     \begin{equation}
     \label{eq: alice-prompt}
     \vu_{\mathrm{Alice}}:= (k_0,v_0),\ldots,(k_{n-1},v_{n-1})
     \end{equation}
     of the input for $\nmqar$ (without the queries), and Bob with their access of the query part generates the following;
     \begin{equation}
     \label{eq: bob-prompt}
     \vu_{\mathrm{Bob}}:= q_0,\ldots,q_{m-1}
     \end{equation}
     of the input for $\nmqar$ (without the key-value pairs) as in \eqref{eq: nmqar-input}. That is, the concatenation $\vu_{\mathrm{Alice}}::\vu_{\mathrm{Bob}} \equiv \vu$ in \eqref{eq: nmqar-input}. We then have
     \begin{equation}
         \label{eq: jrp-protocol}\underbrace{\vu_{\mathrm{Alice}}::\vu_{\mathrm{Bob}}:: \cdots ::\vu_{\mathrm{Alice}}::\vu_{\mathrm{Bob}}}_{p-\mathrm{times}} \equiv \vu^{\jrp},
     \end{equation}
     the corresponding {\em $\jrp$} prompt for the input $\vu$ to the $\nmqar$ problem. We now claim that the following protocol (\cref{alg:nmqar-two-way}) is equivalent to running the recurrent model $\calM_{\nmqar}$ on the $\jrp$ prompt $\vu^{\jrp}$:
    \begin{algorithm*}[!h]
  \caption{\label{alg:nmqar-two-way} Communication Protocol for $\nmqar$}
  \small
  \begin{algorithmic}[1]
    \Require\ A recurrent model $\calM_{\nmqar}$ solving $\nmqar$ along with the inputs $\vu_{\mathrm{Alice}},\vu_{\mathrm{Bob}}$ from \ref{eq: alice-prompt} and \ref{eq: bob-prompt}.
    \Ensure\ $\calM_{\nmqar}(\vu^{\jrp})$.
    \For{$i \leftarrow$ $0$ to $p-1$}
        \For{$j\leftarrow$ $0$ to $2n-1$}
            \State$\mZ_{\calM_{\nmqar}}^{i\cdot N + j} \leftarrow f_{\calM}^{i\cdot N + j}(\mZ_{\calM_{\nmqar}}^{i\cdot N + j-1}, \vu_{\mathrm{Alice}}[j])$ \Comment{
    $\mZ_{\calM_{\nmqar}}^0 \leftarrow \vu_{\mathrm{Alice}}[0,:]$}
        \EndFor
        \State Alice sends $\mZ_{\calM_{\nmqar}}^{i\cdot N + 2n-1}$ to Bob
        \For{$j\leftarrow$ $0$ to $m-1$}
            \State$\mZ_{\calM_{\nmqar}}^{i\cdot N + 2n + j} \leftarrow f_{\calM}^{i\cdot N + j}(\mZ_{\calM_{\nmqar}}^{i\cdot N + 2n + j-1}, \vu_{\mathrm{Bob}}[j])$
        \EndFor
        \State Bob sends $\mZ_{\calM_{\nmqar}}^{i\cdot N + m-1}$ to Alice
    \EndFor
  \end{algorithmic}
\end{algorithm*}

The equivalency of this protocol with running the model $\calM_{\nmqar}$ follows from \eqref{eq: jrp-protocol}.

Next, consider an instance $\bm{u}^{\sd}:= (A, B)$ of the set-disjointness problem with $A, B \subseteq [\sdDim]$ and $|A| + |B| = N$, where $A:= \{A_0, \ldots, A_{|A|-1}\}, B:= \{B_0, \ldots, B_{|B|-1}\}$. Due to \cref{prop: nmqar-sd}, we know that we can generate an equivalent input $\bm{u}$ for $\nmqar$ given an input $\bm{u}^{\sd}$ to the $\sd$ problem, whence we can generate inputs for Alice and Bob as in \eqref{eq: alice-prompt} and \eqref{eq: bob-prompt}. Applying \cref{alg:nmqar-two-way} then solves the $\nmqar$ problem for $\bm{u}$, and consequently, the SD problem for $\bm{u}^{\sd}$. Here, the total number of bits that are communicated in this protocol is 
\[
T_{\mathrm{bits}}:= \sum_{i=0}^{p-1} \abs{\mZ_{\calM_{\nmqar}}^{i\cdot N + 2n-1}} + \abs{\mZ_{\calM_{\nmqar}}^{i\cdot N + m-1}}.
\] 

Now, if $T_{\mathrm{bits}}$ is $o(\min\{|A|, |B|\})$ bits, we have shown that a two-way communication protocol exists for solving the set-disjointness ($\sd$) that uses $o(\min\{|A|, |B|\})$ communication complexity. However, this contradicts \cref{thm: sd-lb}.
Thus, we have $T_{\mathrm{bits}} \ge \Omega\paren{\min\{|A|, |B|\}}.$

Finally, note that we have
\begin{align*}
    p \cdot 2\max_k\abs{\mZ_{\calM_{\nmqar}}^{k}} &= \sum_{i=0}^{p-1} 2\max_k\abs{\mZ_{\calM_{\nmqar}}^{k}}\\
    &\ge \sum_{i=0}^{p-1} \abs{\mZ_{\calM_{\nmqar}}^{i\cdot N + 2n-1}} + \abs{\mZ_{\calM_{\nmqar}}^{i\cdot N + m-1}} \\
    &\ge \Omega\paren{\min\{|A|, |B|\}}.\\
\implies \max_k\abs{\mZ_{\calM_{\nmqar}}^{k}} &\ge \Omega\paren{\frac{\min\{|A|, |B|\}}{2p}}.
\end{align*}
This concludes the proof.
\end{proof}

\clearpage
\begin{table*}[h!]
    \caption{{\arch} Training Settings. For hybridizing the three layer types -- gated convolutions, sliding window, and linear attention -- we use linear attention at layers $\{2, 7, 12, 17, 22, 27, 32\}$ and sliding window at layers $\{3, 8, 13, 18, 23, 28, 33\}$, with gated convolution layers elsewhere. We did not tune the layer orderings and proportions.}
    \centering
    \begin{tabular}{rccc}
    \toprule
    {} & 356M & 1.3B  \\
    \midrule
    Optimizer & \multicolumn{3}{c}{Adam} \\
    Optimizer momentum & \multicolumn{3}{c}{$\beta_1, \beta_2=0.9, 0.95$} \\
    Optimizer eps & \multicolumn{3}{c}{$1e-8$} \\
    Precision &  \multicolumn{3}{c}{BFloat16} \\
    \midrule
    Encoder region length  & \multicolumn{3}{c}{1024 }  \\
    Masked language modeling probability & \multicolumn{3}{c}{15\%} \\
    MLM loss scale & \multicolumn{3}{c}{0.25}  \\
    NTP loss scale & \multicolumn{3}{c}{1.00}  \\
    \midrule
    Warmup & \multicolumn{3}{c}{1\%} \\
    Learning rate decay & \multicolumn{3}{c}{Cosine} \\
    Learning rate (min, base) & \multicolumn{3}{c}{8e-5, 8e-4} \\
    Global batch size & \multicolumn{3}{c}{256} \\ 
    Weight decay & \multicolumn{3}{c}{0.1} \\
    \midrule
    Num Layers  & 26   & 36   \\
    Hidden Size & 1024 & 1792 \\
    MLP Activation   &\multicolumn{3}{c}{SwiGLU}   \\
    MLP Width   & \multicolumn{3}{c}{2}  \\
    \midrule
    Num. Linear Attn Layers  & 5    & 7   \\
    Num. Linear Attn Heads   & \multicolumn{2}{c}{16}  \\
    Taylor Feature Dimension & \multicolumn{2}{c}{16} \\
    Linear Attn Positional Encodings & \multicolumn{2}{c}{None} \\
    \midrule
    Num. Sliding Window Layers & 5    & 7  \\
    Sliding Window Size & 64 & 16 \\
    Sliding Window Heads & \multicolumn{2}{c}{16} \\
    Sliding Window Positional Encodings & \multicolumn{2}{c}{Rotary} \\
    \midrule
    Num. \BaseConv\ Layers & 17  & 22 \\
    \BaseConv\ Projection Expansion Factor & \multicolumn{2}{c}{4} \\
    \BaseConv\ Filter Size & \multicolumn{2}{c}{3} \\
     \BaseConv\ Activation & \multicolumn{2}{c}{SiLU} \\
    \bottomrule 
    \end{tabular}
    \label{tab:jrt-training-details}
\end{table*}
\begin{table*}[h!]
    \caption{Based Training Settings. For hybridizing the three layer types -- gated convolutions, sliding window, and linear attention -- we use linear attention at layers $\{2, 7, 12, 17, 22, 27, 32\}$ and sliding window at layers $\{3, 8, 13, 18, 23, 28, 33\}$, with gated convolution layers elsewhere. We did not tune the layer orderings and proportions.}
    \centering
    \begin{tabular}{rccc}
    \toprule
    {} & 363M & 1.4B  \\
    \midrule
    Optimizer & \multicolumn{3}{c}{Adam} \\
    Optimizer momentum & \multicolumn{3}{c}{$\beta_1, \beta_2=0.9, 0.95$} \\
    Optimizer eps & \multicolumn{3}{c}{$1e-8$} \\
    Precision &  \multicolumn{3}{c}{BFloat16} \\
    \midrule
    Warmup & \multicolumn{3}{c}{1\%} \\
    Learning rate decay & \multicolumn{3}{c}{Cosine} \\
    Learning rate (min, base) & \multicolumn{3}{c}{8e-5, 8e-4} \\
    Global batch size & \multicolumn{3}{c}{256} \\ 
    Weight decay & \multicolumn{3}{c}{0.1} \\
    \midrule
    Num Layers  & 27   & 36   \\
    Hidden Size & 1024 & 1792 \\
    MLP Activation   &\multicolumn{3}{c}{SwiGLU}   \\
    MLP Width   & \multicolumn{3}{c}{2}  \\
    \midrule
    Num. Linear Attn Layers  & 5    & 7   \\
    Num. Linear Attn Heads   & \multicolumn{2}{c}{16}  \\
    Taylor Feature Dimension & \multicolumn{2}{c}{16} \\
    Linear Attn Positional Encodings & \multicolumn{2}{c}{None} \\
    \midrule
    Num. Sliding Window Layers & 5    & 7  \\
    Sliding Window Size & \multicolumn{2}{c}{128}  \\
    Sliding Window Heads & \multicolumn{2}{c}{16} \\
    Sliding Window Positional Encodings & \multicolumn{2}{c}{Rotary} \\
    \midrule
    Num. \BaseConv\ Layers & 17  & 22 \\
    \BaseConv\ Projection Expansion Factor & \multicolumn{2}{c}{4} \\
    \BaseConv\ Filter Size & \multicolumn{2}{c}{3} \\
     \BaseConv\ Activation & \multicolumn{2}{c}{SiLU} \\
    \bottomrule 
    \end{tabular}
    \label{tab:based-training-details}
\end{table*}

\begin{table*}[h!]
    \caption{Mamba Training Settings}
    \centering
    \begin{tabular}{rccc}
    \toprule
    {} & 358M & 1.3B  \\
    \midrule
    Optimizer & \multicolumn{3}{c}{Adam} \\
    Optimizer momentum & \multicolumn{3}{c}{$\beta_1, \beta_2=0.9, 0.95$} \\
    Optimizer eps & \multicolumn{3}{c}{$1e-8$} \\
    Precision &  \multicolumn{3}{c}{BFloat16} \\
    \midrule
    Warmup & \multicolumn{3}{c}{1\%} \\
    Learning rate decay & \multicolumn{3}{c}{Cosine} \\
    Learning rate (min, base) & \multicolumn{3}{c}{8e-5, 8e-4} \\
    Global batch size & \multicolumn{3}{c}{256} \\ 
    Weight decay & \multicolumn{3}{c}{0.1} \\
    \midrule
    Num Layers   & \multicolumn{3}{c}{46}    \\
    Hidden Size  & 1024    & 2048   \\
    RMSNorm      & \multicolumn{3}{c}{True}  \\
    Norm Epsilon & \multicolumn{3}{c}{$1e-5$}  \\
    Dt State     & \multicolumn{3}{c}{$16$} \\ 
    Dt (Min, Max) & \multicolumn{3}{c}{$(0.001, 0.1)$} \\
    Dt Init. Strategy & \multicolumn{3}{c}{Random} \\
    Dt Init. Floor & \multicolumn{3}{c}{$1e-4$} \\
    Dt Scale & \multicolumn{3}{c}{$1.0$} \\
    Dt Softplus & \multicolumn{3}{c}{True} \\
    Projection Expansion Factor & \multicolumn{3}{c}{2} \\
    Short Conv Filter Size & \multicolumn{3}{c}{4} \\
    \bottomrule 
    \end{tabular}
    \label{tab:mamba-training-details}
\end{table*}
\begin{table*}[h!]
    \caption{Attention Training Settings}
    \centering
    \begin{tabular}{rccc}
    \toprule
    {} & 360M & 1.3B  \\
    \midrule
    Optimizer & \multicolumn{3}{c}{Adam} \\
    Optimizer momentum & \multicolumn{3}{c}{$\beta_1, \beta_2=0.9, 0.95$} \\
    Optimizer eps & \multicolumn{3}{c}{$1e-8$} \\
    Precision &  \multicolumn{3}{c}{BFloat16} \\
    \midrule
    Warmup & \multicolumn{3}{c}{1\%} \\
    Learning rate decay & \multicolumn{3}{c}{Cosine} \\
    Learning rate (min, base) & \multicolumn{3}{c}{8e-5, 8e-4} \\
    Global batch size & \multicolumn{3}{c}{256} \\ 
    Weight decay & \multicolumn{3}{c}{0.1} \\
    \midrule
    Num Layers  & 24    & 36 \\
    Hidden Size & 1024  & 1680 \\
    Num Heads   & 16 & 24  \\
    RMSNorm     & \multicolumn{3}{c}{True} \\
    MLP Bias    & \multicolumn{3}{c}{False}  \\
    Flash Attn  & \multicolumn{3}{c}{True}  \\
    Rotary Emb. Fraction & \multicolumn{3}{c}{0.5}  \\
    MLP Activation & \multicolumn{3}{c}{SwiGLU} \\
    MLP Width & \multicolumn{3}{c}{4}  \\
    \bottomrule 
    \end{tabular}
    \label{tab:attn-training-details}
\end{table*}



\end{document}